\documentclass{article}

\usepackage{microtype}
\usepackage{graphicx}
\usepackage{subfigure}
\usepackage{booktabs} 

\usepackage{hyperref}

\usepackage[accepted]{icml2026}

\usepackage{amsmath}
\usepackage{amssymb}
\usepackage{mathtools}
\usepackage{amsthm}

\usepackage{multirow}
\usepackage{threeparttable}
\usepackage{makecell}
\usepackage{array}

\usepackage[capitalize,noabbrev]{cleveref}

\usepackage{enumitem}   
\usepackage{algorithm}
\usepackage{algorithmic}

\theoremstyle{plain}
\newtheorem{theorem}{Theorem}[section]
\newtheorem{proposition}[theorem]{Proposition}
\newtheorem{lemma}[theorem]{Lemma}
\newtheorem{corollary}[theorem]{Corollary}
\theoremstyle{definition}
\newtheorem{definition}[theorem]{Definition}
\newtheorem{assumption}[theorem]{Assumption}
\theoremstyle{remark}

\theoremstyle{plain}
\newtheorem{subtheorem}{Theorem}[subsection]
\newtheorem{subproposition}[subtheorem]{Proposition}

\usepackage[textsize=tiny]{todonotes}

\icmltitlerunning{General Quantification of Covariate and Concept Shifts}

\begin{document}

\twocolumn[
  \icmltitle{General Quantification of Covariate and Concept Shifts}



  \icmlsetsymbol{equal}{*}

  \begin{icmlauthorlist}
    \icmlauthor{Hongbo Chen}{yyy}
    \icmlauthor{Li Charlie Xia}{yyy}
  \end{icmlauthorlist}

  \icmlaffiliation{yyy}{Department of Statistics and Financial Mathematics, School of Mathematics, South China University of Technology, Guangzhou, China. First author: Hongbo Chen \textless{}hongboc616@gmail.com\textgreater{}} 
  \icmlcorrespondingauthor{Li Charlie Xia}{lcxia@scut.edu.cn}

  \icmlkeywords{Machine Learning, ICML}

  \vskip 0.3in
]



\printAffiliationsAndNotice{}  

\begin{abstract}
Generalization under distribution shift remains a core challenge in modern machine learning, yet existing learning bound theory is limited to narrow, idealized settings and is non-estimable from samples. In this paper, we bridge the gap between theory and practical applications. We first show that existing definition of concept shift breaks when the source and target supports mismatch. Leveraging entropic optimal transport, we propose a key notion: $\gamma^{*}\!$-concept shifts, and derive a general error bound unifying covariate and $\gamma^{*}\!$-concept shifts, which applies to broad loss functions, label spaces, and stochastic labeling. We further develop estimators for these shifts with concentration guarantees, and the DataShifts algorithm, which can quantify distribution shifts and estimate the error bound in most applications - a rigorous and general tool for analyzing learning error under distribution shift.
\end{abstract}

\section{Introduction}

With the growth of data and computing power, supervised learning has achieved remarkable success. Nonetheless, traditional supervised learning assumes that the training data (source domain) and the test or deployment data (target domain) share the same distribution. However, in many real-world applications, the test data distribution can differ substantially from the training distribution, and this discrepancy can significantly impact model performance in the target domain. To analyze such challenges, researchers theorized that the distributions between the source and target domains are shifted and developed methods to assess how learners trained in the source domain could perform on the target domain. Depending on whether the target domain data is accessible, the problems are further categorized into domain adaptation and domain generalization.

Theoretical results on distribution shift usually bound a model's target domain error by its source domain error plus a measure of the distribution shift. The shift is further dissected into X (covariate) and Y$\mid$X (concept) shifts \citep{moreno2012unifying,liu2021towards,zhang2023nico}. Early studies proposed using $\mathcal{H}$-divergence to measure  X shift and derived an error bound for binary classification \citep{ben2006analysis,ben2010theory}. Later works improved X shift using the maximum mean discrepancy \citep{long2015learning} and the Wasserstein distance \citep{shen2018wasserstein,courty2017joint}. Further studies proposed more complex metrics for the X shift to obtain bounds for multiclass classification \citep{zhang2020unsupervised, zhang2019bridging}. These results focused only on X shift and additionally relied on a joint-error term between the source and target domains. Lately, \citet{zhao2019learning} improved this loose joint-error term and proposed a bound that explicitly considers both X and Y$\mid$X shifts for binary classification, and \citet{zhang2023nico} extended the theory to multiclass classification.

In this paper, we focus on two remaining key problems in the existing theoretical frameworks, which significantly hinder researchers from analyzing X and Y$\mid$X shifts in applications: 

\begin{itemize}[leftmargin=1.6em, labelsep=0.5em]  
  \item \textbf{Generalizability.} Existing theories rely on restrictive assumptions:
        they require deterministic labeling, omitting label noise and latent
        confounders commonly seen in practice; moreover, they only apply to classification with absolute error, excluding broader tasks such as regression and loss families.
  \item \textbf{Estimability.} Although existing theories provide a preliminary
        definition of Y$\mid$X shift, this definition and the accompanying error
        bounds—are \emph{not} estimable. As a result, one cannot rigorously quantify
        Y$\mid$X shift on real data, nor assess its impact on model performance.
\end{itemize}

We aim to provide a general theoretical framework unifying X and Y$\mid$X shifts that widely applies to stochastic labeling, most supervised tasks, and broader loss functions. Moreover, these two shifts can be accurately estimated from samples, thereby offering a rigorous, plug-and-play tool for quantifying and analyzing distribution shifts in real applications.

Specifically, our key observation is that if X shift occurs, the supports of the source and target covariate distributions may not overlap. Such a support mismatch renders the existing theories' Y$\mid$X shift ill-defined, fundamentally causing the error bounds non-estimable and loose. Our key innovation is to employ the entropic optimal transport to give a general definition for X and Y$\mid$X shift. The $\gamma^{*}$-Y$\mid$X shift we propose, which depends on the entropic optimal transport coupling of X shift, stays well-defined even when supports mismatch, and applies to the stochastic labeling and general label space. Based on that, we derived a new error bound that considers both the X and the $\gamma^{*}$-Y$\mid$X shifts. Our new bound relies only on the Lipschitz continuity of the hypothesis $h$ and loss $\ell$, and is agnostic to specific hypothesis space, label space, or loss function, which naturally generalizes to binary and multiclass classification, regression and other tasks.

Since our $\gamma^{*}$-Y$\mid$X shift is well-defined regardless of support mismatch, X and Y$\mid$X shifts, and our new error bound becomes estimable from samples. Nonetheless, for X shift, the traditional plug-in estimator for entropic optimal transport tends to overestimate due to the curse of dimensionality. We thus further developed a debiased estimator that remains accurate in high dimensions. For our $\gamma^{*}$-Y$\mid$X shift, we also proposed an estimator. We proved both estimators' concentration inequalities to their true values. Leveraging these two estimators, we presented the DataShifts algorithm, enabling quantification of X and Y$\mid$X shifts on general labeled data. Finally, we apply our theoretical framework and DataShifts to three distinct tasks: Novozymes enzyme prediction, ColoredMNIST, and PACS, clearly validating the general effectiveness of our theoretical results.

\textbf{Contributions.} In summary, our major contributions are:
\begin{itemize}[leftmargin=1.6em, labelsep=0.5em]
\item We show that support mismatch makes the existing Y$\mid$X shift ill-defined—hence loose and non-estimable. We introduce the $\gamma^{*}$-Y$\mid$X shift as a well-defined concept shift, which possesses many desirable properties.
\item We derive a general error bound unifying covariate and $\gamma^{*}\!$-concept shifts. Our bound covers broader learning scenarios, far beyond traditional binary classification.
\item For the X and $\gamma^{*}$-Y$\mid$X shifts in our theory, we propose two estimators and prove their concentration inequalities, ensuring that these shifts can be rigorously estimated from finite samples.
\item We integrate our theoretical results into the DataShifts algorithm, which can estimate X and $\gamma^{*}$-Y$\mid$X shifts from real data, supplying a rigorous and general tool for quantifying and analyzing distribution shifts.
\end{itemize}

The paper is organized as follows: Section~\ref{sec:preliminaries} covers the preliminaries, Section~\ref{sec:theory} presents the population-level theoretical results, Section~\ref{sec:estimation} focuses on the statistical results, and Section~\ref{sec:experiments} presents experiments validating our theory.


\section{Preliminary}
\label{sec:preliminaries}

\subsection{Problem Setup}

Our problem is to bound the error of models under distribution shift. Let $\mathcal{X}$ and $\mathcal{Y}$ be the covariate space and the label space, respectively. Let $\mathcal D_{XY}^{S}$ and $\mathcal D_{XY}^{T}$ be the joint distributions of covariates and labels on $\mathcal X\times\mathcal Y$ for the source and target domains, respectively. $\mathcal D_{X}^{S}$, $\mathcal D_{X}^{T}$ are their covariate marginals on $\mathcal X$. For any $x\in\mathcal X$, we let $\mathcal D_{Y|X=x}^{S}$ and $\mathcal D_{Y|X=x}^{T}$ be the conditional label distributions at $x$ in the source and target domain. Let $\mathcal Y^{\prime}$ be the output space of the learner, $\ell:\mathcal Y\times\mathcal Y^{\prime}\to\mathbb R$ the loss, and $\mathcal H\subseteq\{g:\mathcal X\to\mathcal Y^{\prime}\}$ the hypothesis space. For a hypothesis $h\in\mathcal H$, the learning errors for source and target domains are: 
\[
\epsilon_S(h) = \mathbb E_{(x_S,y_S)\sim\mathcal D_{XY}^{S}}\bigl[\ell(y_S,h(x_S))\bigr] \\
\]
\[
\epsilon_T(h) = \mathbb E_{(x_T,y_T)\sim\mathcal D_{XY}^{T}}\bigl[\ell(y_T,h(x_T))\bigr]
\]
We hope to bound $\epsilon_T(h)$ by $\epsilon_S(h)$ and a measure of distribution shift, consistent with existing theoretical results \citep{ben2006analysis,zhao2019learning}. Notably, the space $\mathcal X$ can be the raw input space or a representation space output by an upstream learner \citep{ben2006analysis}. Our theory treats them in the same way, so it applies to both raw data and learned representations.

\paragraph{Stochastic and deterministic labeling.} Above, we assume that the label $y$ follows the conditional distribution at point $x$: $y\sim\mathcal D_{Y|X=x}$, namely the stochastic labeling setting \citep{zhao2019learning}. It enables our theory to accommodate latent confounders and label noise, which are common in practice. In contrast, existing theories oversimplify by using deterministic labeling: a labeling function $f:\mathcal X\to\mathcal Y$ with $y=f(x)$, a special case of stochastic labeling in which the conditional distribution collapses to a Dirac mass at $f(x)$, i.e., $\mathcal D_{Y|X=x}=\delta_{f(x)}$.

\subsection{Existing Theory and Ill-Defined Y$\mid$X Shift}

We first demonstrate that support mismatch leads to an ill-defined Y$\mid$X shift, a key flaw in existing theory.

\begin{definition}[Support]
\label{def:support}
Let $\mu$ be a probability measure on the topological space $(\mathcal X,\tau)$. Its support is defined as:
\[
\operatorname{supp}(\mu) \;=\; \{\,x\in\mathcal X \mid \forall\,U\in\tau,\;x\in U,\;\mu(U)>0\}
\]
\end{definition}
$\operatorname{supp}(\mu)$ is the closure of every region where $\mu$ has positive measure, or equivalently, the complement of the union of all $\mu$-null open sets:
\[
\operatorname{supp}(\mu) \;=\; \mathcal X \setminus \Bigl\{\,\bigcup \{\,U\in\tau \mid \mu(U)=0\}\Bigr\}
\]

\begin{lemma}[Ill-Defined Expectation of Conditional Probability]
\label{lem:ill-defined-expectation}
For probability measure $P_X$, a $P_X$-measurable function $P(A\mid X=x):\mathcal X\to[0,1]$ is the conditional probability of event $A$ given $X=x$. For probability measure $P_X^{\prime}$ satisfying $\operatorname{supp}(P_X^{\prime}) \setminus \operatorname{supp}(P_X) \neq \varnothing$, the expectation $\mathbb E_{P_X^{\prime}}[P(A\mid X=x)]$ is arbitrary.
\end{lemma}

\paragraph{Remark} This problem arises because the conditional probability $P(A\mid X=x)$ is unique almost everywhere with respect to $P_X$ (unique $P_X$-a.e.), rather than to $P_X^{\prime}$. When $\operatorname{supp}(P_X^{\prime}) \setminus \operatorname{supp}(P_X) \neq \varnothing$, $P_X^{\prime}$ assigns positive mass to some $P_X$-null sets. In this case, $\mathbb E_{P_X^{\prime}}[P(A\mid X=x)]$ depends on the values of  $P(A\mid X=x)$ on sets where it is not uniquely determined. Hence, support mismatch leads to the ill-defined expectation of conditional probability. This issue directly impacts the definition and computation of Y$\mid$X shift in existing theories, since the Y$\mid$X shift is typically formulated as an expectation over conditional labeling distributions or its collapsed labeling functions.

With deterministic labeling, \citet{zhao2019learning} used $\mathcal H$-divergence \citep{ben2006analysis} to derive an error bound for soft-label binary classification under X and Y$\mid$X shifts:

\begin{theorem}[Existing Learning Bound on Distribution Shift]
\label{thm:existing-bound}
Let $\mathcal D_{X}^{S}, \mathcal D_{X}^{T}$ be the covariate distributions and $f_S,f_T:\mathcal X\to[0,1]$ be the labeling functions for the source and target domain. Using absolute error loss $|\cdot|$, for any hypothesis space $\mathcal H\subseteq[0,1]^{\mathcal X}$ define:
\[
\tilde{\mathcal H}
= \bigl\{\operatorname{sgn}(|h(x)-h^{\prime}(x)| - t)
\;\bigm|\; h,h^{\prime}\in\mathcal H,\; t\in[0,1]\bigr\},
\]
then for any $h\in\mathcal H$,
\begin{align}
\epsilon_T(h)
&\;\le\;
\epsilon_S(h)
+ d_{\tilde{\mathcal H}}\!\bigl(\mathcal D_{X}^{S},\mathcal D_{X}^{T}\bigr)
+ \min\Bigl\{
    \mathbb{E}_{x \sim \mathcal D_{X}^{S}}
      \bigl[\,|f_S(x) \notag \\[2pt]   
&\qquad
    -f_T(x)|\,\bigr],\mathbb{E}_{x \sim \mathcal D_{X}^{T}}
      \bigl[\,|f_S(x)-f_T(x)|\,\bigr]
  \Bigr\} \label{eq:existing-bound}                     
\end{align}
\end{theorem}
Here, $d_{\tilde{\mathcal H}}$ measures X shift, and the two expectations $\mathbb E_{x\sim\mathcal D_{X}^{S}}[|f_S(x)-f_T(x)|]$, $\mathbb E_{x\sim\mathcal D_{X}^{T}}[|f_S(x)-f_T(x)|]$ are both Y$\mid$X shift but measured with the source or target covariate distribution, with the smaller one taken in the bound.

\paragraph{Remark* (Limitations of Theorem~\ref{thm:existing-bound})} \label{remark:weakness}
First, it is highly specialized, applying only to deterministic labeling, binary classification, and absolute loss. Moreover, when X shift causes support mismatch : $\operatorname{supp}(\mathcal D_{X}^{T})\setminus\operatorname{supp}(\mathcal D_{X}^{S})\neq\varnothing$, the conditional distribution $\mathcal D_{Y|X=x}^{S}$ (i.e., $f_S(x)$) is only unique $\mathcal D_{X}^{S}$-a.e. and arbitrary on the mismatched region, so the Y$\mid$X shift $\mathbb E_{x\sim\mathcal D_{X}^{T}}[|f_S(x)-f_T(x)|]$ is ill-defined. Similarly, if $\operatorname{supp}(\mathcal D_{X}^{S})\setminus\operatorname{supp}(\mathcal D_{X}^{T})\neq\varnothing$, the other Y$\mid$X shift term is also ill-defined. This additionally causes two problems:

\begin{itemize}[leftmargin=1.6em, labelsep=0.5em]  
  \item \textbf{Loose bound.} The ill-defined $Y|X$ shifts can take arbitrary values,
        making the bound in Eq.~\eqref{eq:existing-bound} loose.
  \item \textbf{Non-estimable.} When the real concept $\mathcal D_{Y|X=x}^{S}$ or $f_S(x)$ is unknown, we cannot sample it outside 
        $\operatorname{supp}(\mathcal D_{X}^{S})$;
        hence, under support mismatch the expectations
        $\mathbb E_{x\sim\mathcal D_X^{T}}\!\bigl[|f_S(x)-f_T(x)|\bigr]$
        is non-estimable, as well as $\mathbb E_{x\sim\mathcal D_X^{S}}\!\bigl[|f_S(x)-f_T(x)|\bigr]$.
\end{itemize}

\subsection{Entropic Optimal Transport}

We introduce entropic optimal transport, which we will employ to give the general definitions of X and Y$\mid$X shifts properly. It measures the distance between two probability distributions, augmenting conventional optimal transport with a relative-entropy regularizer. Given two probability distributions $\mathbb P$ and $\mathbb Q$ on the metric space $(\Omega,\rho)$ and a parameter $\beta\ge0$, the order-1 entropic optimal transport is:
\begin{align}
W_{\beta}(\mathbb P,\mathbb Q)
&=
\inf_{\gamma\in\Gamma(\mathbb P,\mathbb Q)}
\biggl\{
    \int \rho\,d\gamma +\beta\,H\bigl(\gamma \bigm| \mathbb P\otimes\mathbb Q\bigr)
\biggr\} \label{eq:entropic-OT}
\end{align}

where $
H\!\bigl(\gamma \bigm| \mathbb P\otimes\mathbb Q\bigr)\!=\!
\int \log\Bigl(\tfrac{d\gamma(x_1,x_2)}{d\mathbb P(x_1)\,d\mathbb Q(x_2)}\Bigr)\,d\gamma(x_1,x_2)$.
Here, $\rho(x_1,x_2)$ is the cost of transporting mass between points, $\Gamma(\mathbb P,\mathbb Q)$ is the set of joint distributions with marginals $\mathbb P$ and $\mathbb Q$, and $\gamma(x_1,x_2)$ is a transport coupling. Hence, $W_\beta(\mathbb P,\mathbb Q)$ is the minimum total transport cost between $\mathbb P$ and $\mathbb Q$ under an entropy regularizer. When $\beta=0$, this reduces to the Wasserstein-1 distance, denoted by $W_{1}(\mathbb P,\mathbb Q)$. Compared with other distribution distances, (entropic) optimal transport has a well-established geometric meaning \citep{gangbo1996geometry}, and we will consistently use it to measure distribution shifts.

\section{Theoretical Results}
\label{sec:theory}

In this section, we focus on population-level theoretical results. We assume that covariate $X$ and label $Y$ are distributed on the metric spaces $(\mathcal X,\rho_{\mathcal X})$ and $(\mathcal Y,\rho_{\mathcal Y})$, respectively, and give the rigorous and general definitions of their distribution shifts and essential theorems below.

\subsection{General X and Y$\mid$X Shifts}

\begin{definition}[X Shift]
\label{def:x-shift}
Using entropic optimal transport, the X (covariate) shift is defined as:
\begin{align}
S_{Cov}
&=
W_{\beta}\!\bigl(\mathcal D_{X}^{S},\mathcal D_{X}^{T}\bigr)
=
\inf_{\gamma\in\Gamma(\mathcal D_{X}^{S},\mathcal D_{X}^{T})}
\Bigl\{
  \int \rho_{\mathcal X}(x_S,x_T)\, \notag \\[4pt]
&\qquad
  d\gamma(x_S,x_T)
  + \beta\,H\bigl(\gamma \bigm|
     \mathcal D_{X}^{S}\otimes\mathcal D_{X}^{T}\bigr)
\Bigr\} \label{eq:cov-shift}
\end{align}
where the optimal coupling is denoted as $\gamma^{*}$, a joint distribution of $\mathcal D_{X}^{S}$, $\mathcal D_{X}^{T}$ that gives the minimum transport cost.
\end{definition}

With realistic stochastic labeling, the label follows a conditional distribution given the covariate; we give a general and rigorous definition of Y$\mid$X shift below:

\begin{definition}[$\gamma^{*}$-Y$\mid$X Shift]
\label{def:total-pair-yx-shift}
Let $\gamma^{*}$ be the optimal transport coupling as in Definition~\ref{def:x-shift}, $S_{pair}(x_S, x_T)
= W_{1}\bigl(\mathcal D_{Y|X=x_S}^{S}, \mathcal D_{Y|X=x_T}^{T}\bigr)$, then the $\gamma^{*}$-Y$\mid$X shift is defined as the expectation of $S_{pair}(x_S, x_T)$ under $\gamma^{*}$:
\begin{align}
S_{Cpt}^{\gamma^{*}}
&=
\mathbb E_{(x_S, x_T)\sim \gamma^{*}}\bigl[S_{pair}(x_S, x_T)\bigr] \notag \\[3pt]
&=
\int W_{1}\bigl(\mathcal D_{Y|X=x_S}^{S}, \mathcal D_{Y|X=x_T}^{T}\bigr)\,
   d\gamma^{*}(x_S, x_T)
\label{eq:cpt-gamma}
\end{align}
\end{definition}

This definition is based on the optimal transport coupling $\gamma^{*}$ for X shift. $S_{pair}(x_S, x_T)$ denotes the paired conditional distribution shift between the source domain at $x_S$ and the target domain at $x_T$. Intuitively, since optimal transport coupling $\gamma^{*}$ places most of its mass on nearby pairs $(x_S, x_T)$, $S_{Cpt}^{\gamma^{*}}$ is evaluated mainly on such neighboring points.

With deterministic labeling, this definition reduces to:

\begin{corollary}[Consistency under Deterministic Labeling]
\label{cor:consistency-under-deterministic-labeling}
Assume deterministic labeling $\mathcal D_{Y|X=x}^{S} = \delta_{f_S(x)}$, $\mathcal D_{Y|X=x}^{T} = \delta_{f_T(x)}$, then:
\begin{align}
S_{Cpt}^{\gamma^{*}}=
\mathbb E_{\gamma^{*}}\bigl[S_{pair}(x_S, x_T)\bigr]
=\mathbb E_{\gamma^{*}}\bigl[\rho_{\mathcal Y}(f_S(x_S), f_T(x_T))\bigr] \notag
\end{align}
\end{corollary}

The following three lemmas guarantee the good properties of $\gamma^{*}$-Y$\mid$X Shift.

\begin{lemma}[Support of $\gamma^{*}$]
\label{lem:support-of-gamma}
For any $\gamma\in\Gamma(\mathcal D_X^{S},\mathcal D_X^{T})$, $\operatorname{supp}(\gamma) \subseteq \operatorname{supp}(\mathcal D_{X}^{S})\times\operatorname{supp}(\mathcal D_{X}^{T})$.
\end{lemma}

\begin{lemma}[Uniqueness of $\gamma^{*}$]
\label{lem:uniqueness-of-gamma}
If $\beta>0$, then the entropic optimal transport coupling $\gamma^{*}$ in Definition~\ref{def:x-shift} is unique.
\end{lemma}

\begin{lemma}[Collapse of $\gamma^{*}$]
\label{lem:collapse-of-gamma}
If $\beta=0$ and $\mathcal D_{X}^{S} = \mathcal D_{X}^{T}$, $\gamma^{*}$ collapses to diagonal (identity)
coupling $\gamma^{*}=(\mathrm{Id},\mathrm{Id})_{\#}\mathcal D_{X}^{S}$, equivalently,
for every measurable $A\subseteq\mathcal X\times\mathcal X$,
$\gamma^{*}(A)=\int \mathbf{1}_{\{(x,x)\in A\}}\,d\mathcal D_{X}^{S}(x)$, where $\mathbf{1}_{\{\cdot\}}$ is the indicator function.
\end{lemma}

Combining Lemmas~\ref{lem:support-of-gamma} and ~\ref{lem:uniqueness-of-gamma}, the rigor of $\gamma^{*}$-Y$\mid$X Shift is ensured:

\begin{theorem}[Well-Definedness of $\gamma^{*}$-Y$\mid$X Shift]
\label{thm:uniqueness-total-pair}
For $\beta>0$, the $\gamma^{*}$-Y$\mid$X shift $S_{Cpt}^{\gamma^{*}}$ in Definition~\ref{def:total-pair-yx-shift} is unique even when $\operatorname{supp}(\mathcal D_{X}^{S}) \neq \operatorname{supp}(\mathcal D_{X}^{T})$.
\end{theorem}

\paragraph{Remark} Such $\gamma^{*}$-Y$\mid$X shift not only avoids the ill-definition encountered by existing theory, but also makes Y$\mid$X shift tight and estimable. On the other hand, combining Corollary~\ref{cor:consistency-under-deterministic-labeling} and Lemma~\ref{lem:collapse-of-gamma}, the relationship between $\gamma^{*}$-Y$\mid$X shift and existing Y$\mid$X shift is shown as follows:

\begin{proposition}[Relationship to existing Y$\mid$X shift]
\label{pro:relationship-to-existing-yx-shift}
Assume deterministic labeling $\mathcal D_{Y|X=x}^{S} = \delta_{f_S(x)}$, $\mathcal D_{Y|X=x}^{T} = \delta_{f_T(x)}$ and $\mathcal Y\subset\mathbb R$ with $\rho_{\mathcal Y} = |\cdot|$, if $\beta=0$ and $\mathcal D_{X}^{S} = \mathcal D_{X}^{T}$, then:
\begin{align}
S_{Cpt}^{\gamma^{*}}
&=
\mathbb E_{x\sim\mathcal D_X^{S}}\!\bigl[|f_S(x)-f_T(x)|\bigr] \notag \\[3pt]
&=
\mathbb E_{x\sim\mathcal D_X^{T}}\!\bigl[|f_S(x)-f_T(x)|\bigr] \notag
\label{eq:cpt-equality}
\end{align}
\end{proposition}

\paragraph{Remark} The above $\gamma^{*}$-Y$\mid$X shift, defined via the entropic optimal transport coupling $\gamma^{*}$, applies to general settings including stochastic labeling. When $\mathcal D_{X}^{S}=\mathcal D_{X}^{T}$, it recovers the existing Y$\mid$X shift; and when the supports of $\mathcal D_{X}^{S}$ and $\mathcal D_{X}^{T}$ are mismatched, it still remains rigorous.

\subsection{General Learning Bound}

We now give a new cross-domain learning error bound based on X and $\gamma^{*}$-Y$\mid$X shift. To be general, the output space of the learner is a metric space $(\mathcal Y^{\prime}, \rho_{\mathcal Y}^{\prime})$, possibly different from the true label space $(\mathcal Y, \rho_{\mathcal Y})$. For the loss function $\ell:\mathcal Y\times\mathcal Y^{\prime}\to\mathbb R$, we require the following basic assumption from them:

\begin{assumption}[Separately Lipschitz Continuity]
\label{ass:separately-lipschitz}
For metric spaces $(\mathcal Y, \rho_{\mathcal Y})$ and $(\mathcal Y^{\prime}, \rho_{\mathcal Y}^{\prime})$, a function $\ell:\mathcal Y\times\mathcal Y^{\prime}\to\mathbb R$ satisfies separately $(L_{\ell}, L_{\ell}^{\prime})$-Lipschitz if there exist $L_{\ell}, L_{\ell}^{\prime}\ge0$ such that for any $y_{1}, y_{2}\in\mathcal Y$ and $y_{1}^{\prime}, y_{2}^{\prime}\in\mathcal Y^{\prime}$, there is:
\begin{equation}
\bigl|\ell(y_{1}, y_{1}^{\prime}) - \ell(y_{2}, y_{2}^{\prime})\bigr|
\!\le\!L_{\ell}\,\rho_{\mathcal Y}(y_{1}, y_{2})
+ L_{\ell}^{\prime}\,\rho_{\mathcal Y}^{\prime}(y_{1}^{\prime}, y_{2}^{\prime})
\end{equation}
\end{assumption}

\paragraph{Remark} This is a mild assumption: most loss functions are differentiable, which already implies continuity. And their stable optimization usually needs bounded gradients in a region, further ensuring Lipschitz continuity \citep{boyd2004convex}. Note that this assumption also covers asymmetric losses. For instance, for binary classification with label space $[0,1]$, output space is typically $[a,1-a]$ ($a\in(0,0.5)$) by Sigmoid function and $\rho_{\mathcal Y}=\rho_{\mathcal Y}^{\prime}=|\cdot|$, the cross-entropy loss: $
\ell_{CE}(y,\hat y) = -\,y\log\hat y \;-\; (1-y)\log(1-\hat y)
$ satisfies separately $(L_{\ell},L_{\ell}^{\prime})$-Lipschitz with $L_{\ell} = \log\bigl(\tfrac{1-a}{a}\bigr)$ and $L_{\ell}^{\prime} = \tfrac1a$. 

\begin{corollary}[Composition Preserves Separate Lipschitzness]
\label{cor:composition-preserves-separate-lipschitzness}
Let $h:\mathcal X\to\mathcal Y'$ be $L_h$-Lipschitz and let $\ell:\mathcal Y\times\mathcal Y'\to\mathbb R$ be separately $(L_\ell,L'_\ell)$-Lipschitz. Then composite function $\ell \bigl(y,h(x)\bigr):\mathcal Y\times\mathcal X\to\mathbb R$ is separately $(L_\ell,\,L_hL'_\ell)$-Lipschitz.
\end{corollary}

The following two lemmas characterize the transport coupling of joint distributions $\mathcal D_{XY}^{S}$ and $\mathcal D_{XY}^{T}$:

\begin{lemma}[Separately Weak Duality]
\label{lem:separately-weak-duality}
Let $\mathcal{G} \bigl(y,x\bigr):\mathcal Y\times\mathcal X\to\mathbb R$ is separately $(L_\mathcal Y,\,L_\mathcal X )$-Lipschitz, for any coupling of joint distributions $\gamma_{_{XY}}\in\Gamma(\mathcal D_{XY}^{S},\mathcal D_{XY}^{T})$, there exists:
\begin{equation}
\bigl| \mathbb E_{\mathcal D_{XY}^{S}}\bigl[ \mathcal{G}\bigr]\bigr] -\mathbb E_{\mathcal D_{XY}^{T}}\bigl[ \mathcal{G}\bigr]\bigr|
\le
L_\mathcal X E_{\gamma_{_{XY}}}\bigl[ \rho_{\mathcal X} \bigr] + L_\mathcal Y E_{\gamma_{_{XY}}}\bigl[ \rho_{\mathcal Y} \bigr]
\end{equation}
\end{lemma}

\begin{lemma}[Gluing Construction for Joint Coupling]
\label{lem:gluing-construction-for-joint-coupling}
Let $\gamma^{*}$ be the optimal transport coupling of X shift in Definition~\ref{def:x-shift}, $\gamma^{*}_{Y|(x_S,x_T)}$ be the optimal transport coupling of $S_{pair}(x_S, x_T)$ in Definition~\ref{def:total-pair-yx-shift}, construct the joint distribution of couplings: $\gamma_{_{XY}}(dx_S,dx_T,dy_S,dy_T)=\gamma^{*}(dx_S,dx_T)\gamma^{*}_{Y|(x_S,x_T)}(dy_S,dy_T)$, then $\gamma_{_{XY}}$ is the coupling of joint distributions: $\gamma_{_{XY}}\in\Gamma(\mathcal D_{XY}^{S},\mathcal D_{XY}^{T})$.
\end{lemma}

Combining Corollary~\ref{cor:composition-preserves-separate-lipschitzness}, Lemmas~\ref{lem:separately-weak-duality} and ~\ref{lem:gluing-construction-for-joint-coupling}, we obtain the general cross-domain error bound as follows:

\begin{theorem}[General Learning Bound]
\label{thm:learning-bound}
Given the covariate space $(\mathcal X, \rho_{\mathcal X})$, the label space $(\mathcal Y, \rho_{\mathcal Y})$, and the output space $(\mathcal Y^{\prime}, \rho_{\mathcal Y}^{\prime})$, the source and target distributions are $(\mathcal D_{X}^{S}, \mathcal D_{Y|X=x}^{S})$ and $(\mathcal D_{X}^{T}, \mathcal D_{Y|X=x}^{T})$, respectively. If the loss $\ell:\mathcal Y\times\mathcal Y^{\prime}\to\mathbb R$ satisfies separately $(L_{\ell},L_{\ell}^{\prime})$-Lipschitz, then for any hypothesis $h:\mathcal X\to\mathcal Y^{\prime}$ that satisfies $L_{h}$-Lipschitz, the following bound holds:
\begin{equation}
\epsilon_{T}(h)
\;\le\;
\epsilon_{S}(h)
\;+\;
L_{h}\,L_{\ell}^{\prime}\,S_{Cov}
\;+\;
L_{\ell}\,S_{Cpt}^{\gamma^{*}}
\end{equation}
where $S_{Cov}$ is the X shift in Definition~\ref{def:x-shift} and $S_{Cpt}^{\gamma^{*}}$ is the $\gamma^{*}$-Y$\mid$X shift in Definition~\ref{def:total-pair-yx-shift}.
\end{theorem}

\paragraph{Remark} This elegant bound unifies the covariate shift $S_{Cov}$ and the concept shift ($\gamma^{*}$-Y$\mid$X shift) $S_{Cpt}^{\gamma^{*}}$'s effect on target error via the Lipschitz factors $L_{h}\,L_{\ell}^{\prime}$ and $L_{\ell}$. Notably, it depends only on the Lipschitz continuity of the hypothesis $h$ and the loss $\ell$, without any other specific restriction on the label space $\mathcal Y$ or loss $\ell$. It naturally covers binary classification or regression tasks when $\mathcal Y$ is one-dimensional, and multiclass classification or multi-label tasks when $\mathcal Y$ is multi-dimensional. Besides, since the bound holds under stochastic labeling, it applies to a wide range of supervised learning scenarios in practice. On the other hand, by using the well-defined $\gamma^{*}$-Y$\mid$X shift, our bound can be tighter than the existing bound, and more crucially, both the covariate and concept shifts in our theory can be robustly estimated from the finite samples. Additionally, since the entropic optimal transport $S_{Cov}$ increases with the hyperparameter $\beta$, we recommend choosing $\beta$ as a small non-zero value to balance the tightness and the rigor of the theory.

\section{Statistical Results}
\label{sec:estimation}

In practice, true domain distributions are all unknown. We wish to estimate the shifts from domain samples and analyze how these shifts will influence model performance. In this section, we focus on sample-level theoretical results. Suppose we have i.i.d.\ samples $\{(X_{i}^{(S)}, Y_{i}^{(S)})\}$ and $\{(X_{j}^{(T)}, Y_{j}^{(T)})\}$ drawn from $\mathcal D_{XY}^{S}$ and $\mathcal D_{XY}^{T}$ with sizes $N_S$ and $N_T$, respectively. This section involves three parts: the estimation of X shift, the estimation of Y$\mid$X shift, and the DataShifts algorithm to estimate the overall bound.

\subsection{Estimation of X Shift}

By Definition~\ref{def:x-shift}, the X shift is the entropic optimal transport between $\mathcal D_{X}^{S}$ and $\mathcal D_{X}^{T}$. The traditional estimation method uses the entropic optimal transport of empirical distributions—known as the plug-in estimator.

\paragraph{Traditional Plug-in Estimator}
\label{def:plugin-estimator}
Given i.i.d.\ samples $\{X_{i}^{(S)}\}\sim\mathcal D_{X}^{S}$, $\{X_{j}^{(T)}\}\sim\mathcal D_{X}^{T}$ with sample sizes $N_S, N_T$ respectively, set the empirical measures: $\widehat{\mathcal D_{X}^{S}} = \frac{1}{N_S}\sum_{i=1}^{N_S} \delta_{X_{i}^{(S)}}$, $\widehat{\mathcal D_{X}^{T}} = \frac{1}{N_T}\sum_{j=1}^{N_T} \delta_{X_{j}^{(T)}}$, their entropic optimal transport is:
\begin{align}
&W_{\beta}\bigl(\widehat{\mathcal D_{X}^{S}}, \widehat{\mathcal D_{X}^{T}}\bigr)
=
\min_{\hat\gamma}\,
  \langle C, \hat\gamma\rangle
  + \beta \sum_{i=1}^{N_S}\sum_{j=1}^{N_T}
    \hat\gamma_{ij}\,\log \hat\gamma_{ij} \notag \\[3pt]
\quad
&+ \beta\,\log\bigl(N_S N_T\bigr) \quad
\text{s.t.}\quad
\hat\gamma \mathbf 1\!=\!\tfrac{1}{N_S}\mathbf 1,\;
\hat\gamma^{\top}\mathbf 1\!=\!\tfrac{1}{N_T}\mathbf 1
\label{eq:empirical-entropic-ot}
\end{align}

where $C\in\mathbb R_{+}^{N_S\times N_T}$ is the cost matrix with $c_{ij} = \rho_{\mathcal X}\bigl(X_{i}^{(S)}, X_{j}^{(T)}\bigr)$, and $\hat\gamma\in \mathbb R_{+}^{N_S\times N_T}$ represents any discretized transport coupling satisfying the given linear constraints. Such an optimization problem can be solved efficiently at a large scale by the \emph{Sinkhorn} algorithm \citep{cuturi2013sinkhorn,genevay2016stochastic}.

\begin{figure*}[!htbp]
  \centering
  \subfigure[Empirical distance vs. dimension ($d$)]{%
    \includegraphics[height=0.26\textwidth]{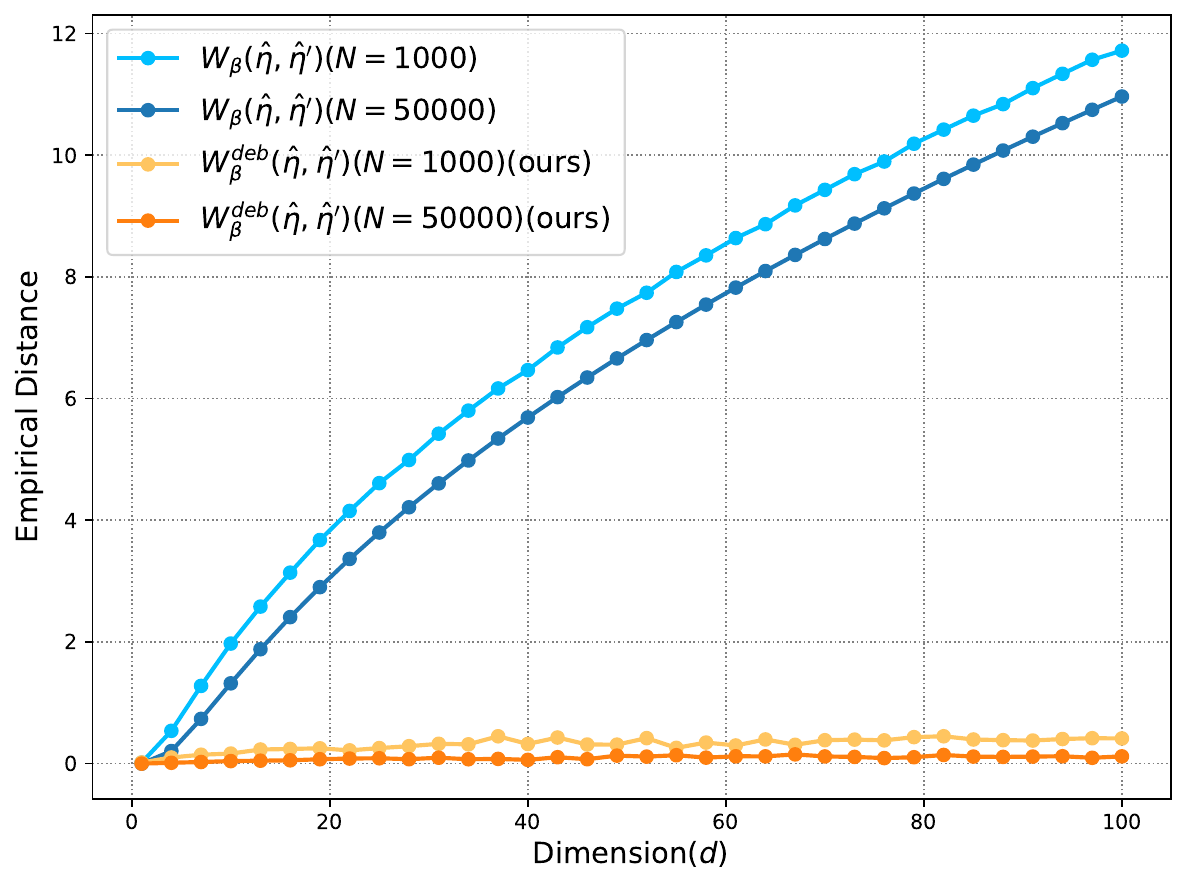}
    \label{fig:1a}%
  }\hfill
  \subfigure[Empirical distance vs. sample size ($N$)]{%
    \includegraphics[height=0.26\textwidth]{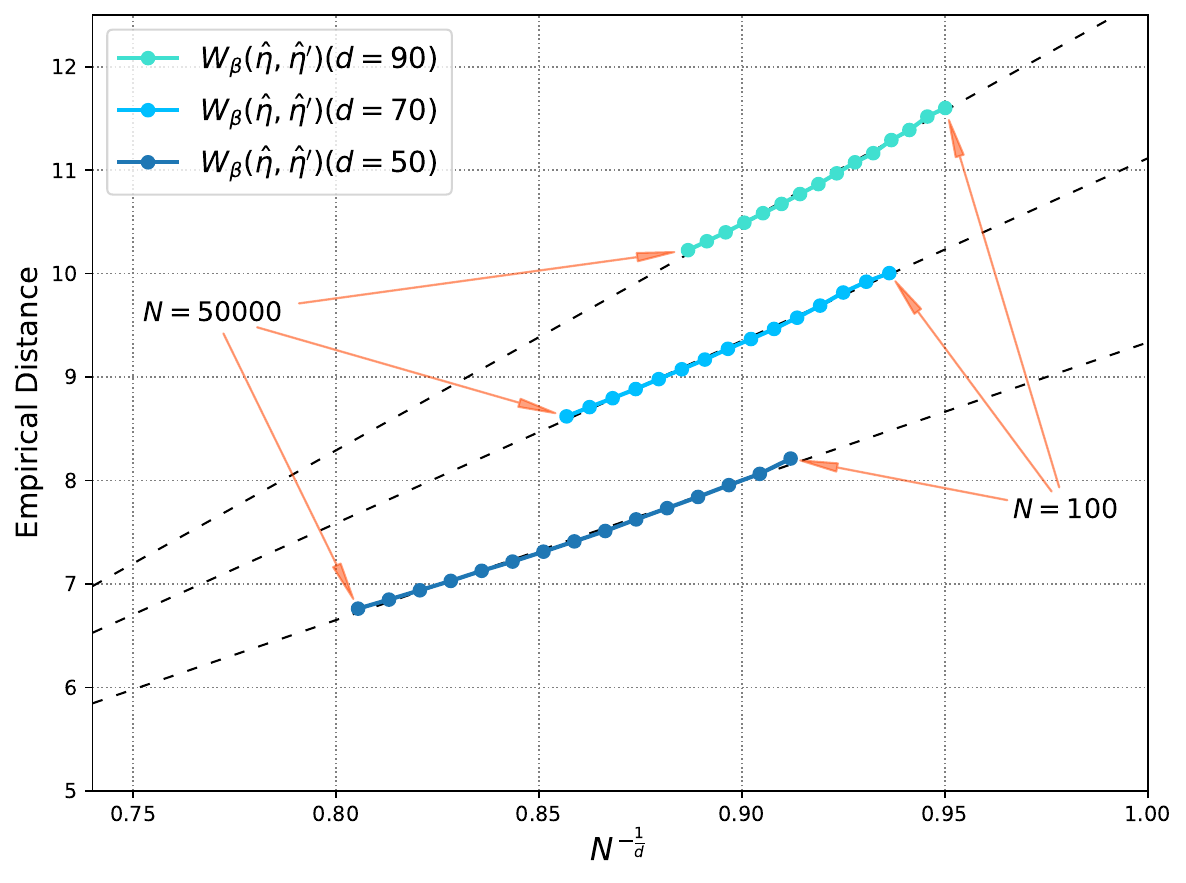}
    \label{fig:1b}%
  }\hfill
  \subfigure[Empirical distance vs. true $W_1$]{%
    \includegraphics[height=0.26\textwidth]{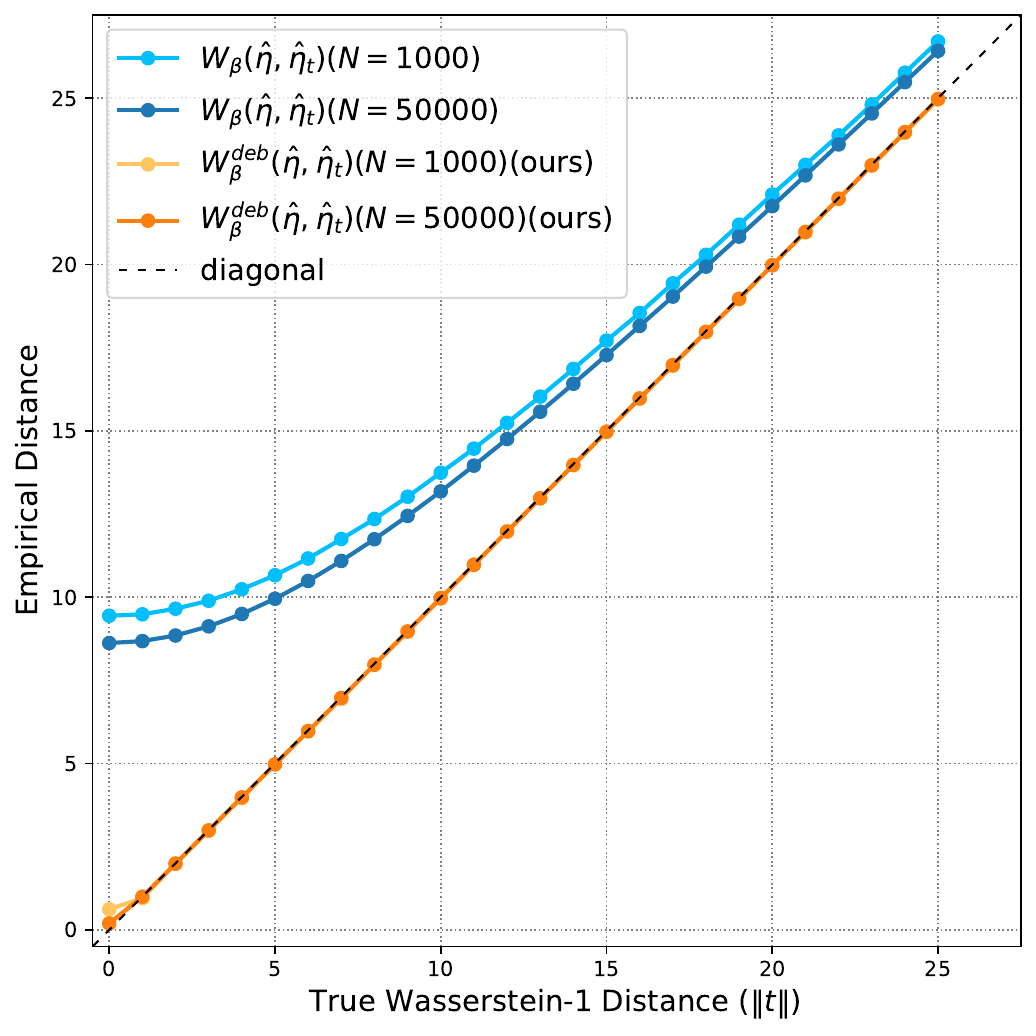}
    \label{fig:1c}%
  }
  \caption{(a)–(c) show the empirical distance from entropic optimal transport $(\beta = 0.001)$ versus dimension $d$, sample size $N$, and true Wasserstein-1 distance. In (a)–(b), $\hat\eta$ and $\hat\eta'$ are independent empirical measures from high-dimensional standard normals, thus the true distance is zero. The traditional estimator $W_\beta(\hat\eta,\hat\eta')$ greatly overestimates as $d$ increases, as shown in (a), and even a much larger $N$ only brings a small improvement, as shown in (b). In (c), with $d = 70$, $\hat\eta_t$ is the empirical measure of shifted standard normal $\mathcal{N}(t,I)$, whose true Wasserstein-1 distance to the standard normal is $\|t\|$. Our debiased estimator remains accurate in every case.}
  \label{fig:combined}
\end{figure*}

\paragraph{Curse of Dimensionality.}
However, in modern applications, the covariate space $\mathcal{X}$ is often high-dimensional. Even when two distributions are very close, their samples' distance can be large. Such curse of dimensionality makes the plug-in estimator greatly overestimated \citep{verleysen2005curse,panaretos2019statistical} (Fig.\ref{fig:1a}). When $\beta$ is small, entropic optimal transport behaves similarly to Wasserstein distance \citep{carlier2017convergence,carlier2023convergence}, and the upward bias decays only in $O(N^{-1/d})$ order \citep{fournier2015rate}, which implies that increasing $N$ has a very limited debiasing effect when $d$ is high (Fig.\ref{fig:1b}).

To address the overestimation problem of the traditional plug-in estimator, we propose the following debiased estimator.

\begin{definition}[Debiased Estimator]
\label{def:debiased-estimator}
Given i.i.d.\ samples $\{X_{i}^{(S)}\}\sim\mathcal D_{X}^{S}$, $\{X_{j}^{(T)}\}\sim\mathcal D_{X}^{T}$ with sample sizes $N_S, N_T$ respectively, split the samples in half to obtain four independent empirical measures:
\(
\widehat{\mathcal D_{X}^{S}}' = \frac{2}{N_S} \sum_{i=1}^{N_S/2} \delta_{X_{i}^{(S)}}\),
\(
\widehat{\mathcal D_{X}^{S}}'' = \frac{2}{N_S} \sum_{i=N_S/2+1}^{N_S} \delta_{X_{i}^{(S)}}\),
\(
\widehat{\mathcal D_{X}^{T}}' = \frac{2}{N_T} \sum_{j=1}^{N_T/2} \delta_{X_{j}^{(T)}}\),
\(
\widehat{\mathcal D_{X}^{T}}'' = \frac{2}{N_T} \sum_{j=N_T/2+1}^{N_T} \delta_{X_{j}^{(T)}}
\). The debiased estimator is:
\begin{align}
W_{\beta}^{deb}\!&\bigl(
\widehat{\mathcal D_{X}^{S}},
\widehat{\mathcal D_{X}^{T}}
\bigr)
\!=\!
\Biggl\lvert
\tfrac12 W_{\beta}\!\bigl(
\widehat{\mathcal D_{X}^{S}}'\!,
\widehat{\mathcal D_{X}^{T}}'
\bigr)^2
\!+\!
\tfrac12 W_{\beta}\!\bigl(
\widehat{\mathcal D_{X}^{S}}''\!,
\widehat{\mathcal D_{X}^{T}}''
\bigr)^2
\notag\\[-1pt]
-&
\tfrac12 W_{\beta}\!\bigl(
\widehat{\mathcal D_{X}^{S}}'\!,
\widehat{\mathcal D_{X}^{S}}''
\bigr)^2
-
\tfrac12 W_{\beta}\!\bigl(
\widehat{\mathcal D_{X}^{T}}'\!,
\widehat{\mathcal D_{X}^{T}}''
\bigr)^2
\Biggr\rvert^{1/2}
\label{eq:debias}
\end{align}
\end{definition}

\paragraph{Remark}
This debiased estimator uses four plug-in estimators. The first two terms, \(W_{\beta}\bigl(\widehat{\mathcal D_{X}^{S}}', \widehat{\mathcal D_{X}^{T}}'\bigr)\) and \(W_{\beta}\bigl(\widehat{\mathcal D_{X}^{S}}'', \widehat{\mathcal D_{X}^{T}}''\bigr)\), estimate the distance between $\mathcal D_{X}^{S}$ and $\mathcal D_{X}^{T}$, including the sample bias. And the last two terms,
\(W_{\beta}\bigl(\widehat{\mathcal D_{X}^{S}}', \widehat{\mathcal D_{X}^{S}}''\bigr)\) and \(W_{\beta}\bigl(\widehat{\mathcal D_{X}^{T}}', \widehat{\mathcal D_{X}^{T}}''\bigr)\)
estimate the distance arising from sample bias only. Subtracting the two parts reduces the sample bias, thus giving a better estimate of $W_{\beta}(\mathcal D_{X}^{S}, \mathcal D_{X}^{T})$. Compared with the traditional plug-in estimator, our debiased estimator $W_{\beta}^{deb}$ reduces the overestimation and remains accurate regardless of the true distribution distance (see Fig.\ref{fig:1a} and \ref{fig:1c}).

Moreover, when the covariate space is \emph{Euclidean}, we derived a concentration inequality guaranteeing that the debiased estimator converges to the true distance:

\begin{theorem}[Concentration Inequality of Debiased Estimator]
\label{thm:concentration-debiased}
Let $\mathcal D_{X}^{S}, \mathcal D_{X}^{T}$ be two distributions on $(\mathbb R^{d}, \|\cdot\|)$ with finite squared-exponential moments. For i.i.d.\ samples $\{X_{i}^{(S)}\}\sim\mathcal D_{X}^{S}$, $\{X_{j}^{(T)}\}\sim\mathcal D_{X}^{T}$ with sample sizes $N_S, N_T$ respectively, when $\beta=0$, and for any $\varepsilon>0$, there exists $N$ such that if $N_S, N_T > N$, then:
\begin{align}
\mathbb P\Bigl(
\bigl|\,W_{\beta}^{deb}(\widehat{\mathcal D_{X}^{S}}, \widehat{\mathcal D_{X}^{T}}) 
- W_{\beta}(\mathcal D_{X}^{S}, \mathcal D_{X}^{T})\bigr| 
> \varepsilon
\Bigr)
&\le \notag \\[3pt]
\quad 2\exp\Bigl(-\tfrac{\lambda_S\,N_S\,V_{\varepsilon}\,\varepsilon^{2}}{32}\Bigr)
+ 2\exp\Bigl(-\tfrac{\lambda_T\,N_T\,V_{\varepsilon}\,\varepsilon^{2}}{32}\Bigr) &
\label{eq:deb-bd}
\end{align}

where $\lambda_S, \lambda_T > 0$ depend only on squared-exponential moments of $\mathcal D_{X}^{S}, \mathcal D_{X}^{T}$, respectively, and $V_{\varepsilon} \in [2-\sqrt3, 2)$ depends only on $W_{\beta}(\mathcal D_{X}^{S}, \mathcal D_{X}^{T})/\varepsilon$.
\end{theorem}

\paragraph{Remark} This theorem implies that the deviation probability of the debiased estimator decays exponentially with sample sizes, so with high probability, our debiased estimator can well approximate the true entropic optimal transport distance $W_{\beta}(\mathcal D_{X}^{S}, \mathcal D_{X}^{T})$. Notably, it also shows the enlightening fact that our estimator's concentration depends not only on each distribution's scale characterized by $\lambda_S, \lambda_T$ but also on the true distance between two distributions.

\subsection{Estimation of Y$\mid$X Shift}

As above, we defined the $\gamma^{*}$-Y$\mid$X shift $S_{Cpt}^{\gamma^{*}}$ in Definition~\ref{def:total-pair-yx-shift}; we now estimate it from samples.

\begin{definition}[Estimator for $\gamma^{*}$-Y$\mid$X Shift]
\label{def:estimator-total-pair-yx}
Given i.i.d.\ samples $\{(X_{i}^{(S)}, Y_{i}^{(S)})\} \sim \mathcal D_{XY}^{S}$, $\{(X_{j}^{(T)}, Y_{j}^{(T)})\} \sim \mathcal D_{XY}^{T}$ with sample sizes $N_S, N_T$ respectively, the estimator of $\gamma^{*}$-Y$\mid$X shift is defined as:
\begin{equation}
\hat S_{Cpt}
= \sum_{i=1}^{N_S}\sum_{j=1}^{N_T} 
  \rho_{\mathcal Y}\bigl(Y_{i}^{(S)}, Y_{j}^{(T)}\bigr)\,\hat\gamma_{ij}^{*},
\end{equation}
where $\hat\gamma^{*}\in\mathbb R_{+}^{N_S\times N_T}$ represents the discrete optimal transport coupling for $W_{\beta}(\widehat{\mathcal D_{X}^{S}}, \widehat{\mathcal D_{X}^{T}})$.
\end{definition}

\paragraph{Remark} In Section 4.1, we use the debiased estimator $W_{\beta}^{deb}\!\bigl(\widehat{\mathcal D_{X}^{S}},\widehat{\mathcal D_{X}^{T}}\bigr)$ for X shift, whereas Definition~\ref{def:estimator-total-pair-yx} still estimates $\gamma^{*}$-Y$\mid$X shift via the optimal transport coupling from the plug-in estimator $W_{\beta}(\widehat{\mathcal D_{X}^{S}}, \widehat{\mathcal D_{X}^{T}})$. This is because the curse of dimensionality mainly affects the transport cost (i.e., inter-sample distances), rather than the transport coupling. Leveraging the stability of entropic optimal transport \cite{eckstein2022quantitative}, the following lemma establishes the convergence of the coupling for plug-in estimator:

\begin{lemma}[Stability of Entropic Optimal Transport Coupling]
\label{lem:stability-of-entropic-optimal-transport-coupling}
Let $\mathcal D_{X}^{S}, \mathcal D_{X}^{T}$ be two distributions on $(\mathcal X,\rho_{\mathcal X})$ with finite squared-exponential moments. $\gamma^{*}$ and $\hat\gamma^{*}$ are the optimal transport couplings of $W_{\beta}(\mathcal D_{X}^{S}, \mathcal D_{X}^{T})$ and $W_{\beta}(\widehat{\mathcal D_{X}^{S}}, \widehat{\mathcal D_{X}^{T}})$ respectively, then:
\begin{align}
&W_1\left( \gamma^{*},\hat\gamma^{*} \right) \le \varLambda +2\sqrt{\frac{2}{\beta \lambda _{\gamma^{*}}}}\sqrt{\varLambda} \notag \\[3pt]
&\varLambda =W_1\left( \mathcal D_{X}^{S} ,\widehat{\mathcal D_{X}^{S}} \right) +W_1\left( \mathcal D_{X}^{T} ,\widehat{\mathcal D_{X}^{T}} \right) 
\end{align}
where $W_1\left( \gamma^{*},\hat\gamma^{*} \right)$ is computed on $\mathcal X\times\mathcal X$ with metric
$\rho\big((x_1,x_2),(x_1',x_2')\big):=\rho_{\mathcal X}(x_1,x_1')+\rho_{\mathcal X}(x_2,x_2')$, $\lambda _{\gamma^{*}} > 0$ depend only on squared-exponential moments of $\mathcal D_{X}^{S}, \mathcal D_{X}^{T}$.
\end{lemma}

\paragraph{Remark} This lemma shows that the Wasserstein-$1$ distance between the plug-in estimator coupling and the population optimal transport coupling, is bounded by the sum of the marginal Wasserstein-$1$ distance between each population and its empirical distribution, with the bound scaled by the parameter $\beta$. Such a uniform bound is guaranteed only when $\beta>0$, which further highlights the necessity of using entropic optimal transport.

By Lemma~\ref{lem:stability-of-entropic-optimal-transport-coupling}, when covariate space $\mathcal X$ is \emph{Euclidean} and label space $\mathcal Y$ is a bounded set in an Euclidean space, we derive a concentration inequality guaranteeing estimator in Definition~\ref{def:estimator-total-pair-yx} approximates true value:

\begin{theorem}[Concentration Inequality for Definition~\ref{def:estimator-total-pair-yx}]
\label{thm:concentration-total-pair}
Let $\mathcal D_{X}^{S}, \mathcal D_{X}^{T}$ be two distributions on $(\mathbb R^{d}, \|\cdot\|)$ with finite squared-exponential moments. Let the label space $\mathcal Y \subset \mathbb R^{d’}$ be bounded by $M = \sup_{y,y'\in\mathcal Y} \|\,y - y'\|$, on which conditional distributions $\mathcal D_{Y|X=x_S}^{S}, \mathcal D_{Y|X=x_T}^{T}$ satisfy $L_{Y|X}$-Lipschitz respectively: $
d_{\mathrm{TV}}\bigl(\mathcal D_{Y|X=x}, \mathcal D_{Y|X=x'}\bigr)
\;\le\; L_{Y|X}\,\|\,x - x'\|$. For i.i.d.\ samples $\{(X_{i}^{(S)}, Y_{i}^{(S)})\} \sim \mathcal D_{XY}^{S}$, $\{(X_{j}^{(T)}, Y_{j}^{(T)})\} \sim \mathcal D_{XY}^{T}$ with sample sizes $N_S, N_T$, when $\beta>0$, and for any $\varepsilon>0$, there exists $N$ such that if $N_S, N_T > N$, then:
\begin{align}
\mathbb P\bigl(&|\hat S_{Cpt} - S_{Cpt}^{\gamma^{*}} - \varDelta|\!>\! \varepsilon\bigr)
\le 2\exp\Bigl(-\frac{N_S\,N_T\,\varPhi\,\varepsilon^{2}}{(N_S + N_T)M^{2}}\Bigr) \notag \\[4pt]
&+\!\exp\Bigl(-\frac{\lambda_S^{1/2}N_S\varPhi\,\varepsilon^{2}}{4\,\lambda_T^{1/2}M^{2}}\Bigr)
\!+\exp\Bigl(-\frac{\lambda_T^{1/2}N_T\,\varPhi\,\varepsilon^{2}}{4\,\lambda_S^{1/2}M^{2}}\Bigr)
\label{eq:cpt-bound}
\end{align}

where $d_{\mathrm{TV}}$ is total-variation distance, $\lambda_S, \lambda_T > 0$ depend only on the squared-exponential moments of $\mathcal D_{X}^{S}, \mathcal D_{X}^{T}$, $\varPhi > 0$ depends on $L_{Y|X}, \lambda_S, \lambda_T, \beta$, the bias $\varDelta$ is a constant.
\end{theorem}

\paragraph{Remark} This theorem implies that the deviation probability of the $\gamma^{*}$-Y$\mid$X Shift estimator $\hat S_{Cpt}$ decays exponentially with sample sizes, so with high probability, the estimator can well approximate the $S_{Cpt}^{\gamma^{*}}\!+\!\varDelta$. The bias term $\varDelta$ arises because Definition~\ref{def:estimator-total-pair-yx} uses $\rho_{\mathcal Y}\!\bigl(Y_{i}^{(S)}, Y_{j}^{(T)}\bigr)$ as a single-point estimate of $S_{pair}\bigl(X_{i}^{(S)}, X_{j}^{(T)}\bigr)$ in Definition~\ref{def:total-pair-yx-shift}. The following two propositions show that the bias $\varDelta$ is controlled.

\begin{proposition}[$\varDelta$ in Deterministic Labeling]
\label{pro:varDelta-in-deterministic-labeling}
    Assume deterministic labeling $\mathcal D_{Y|X=x}^{S} = \delta_{f_S(x)}$, $\mathcal D_{Y|X=x}^{T} = \delta_{f_T(x)}$, then the $\varDelta$ in Theorem~\ref{thm:concentration-total-pair} satisfies:
\begin{equation}
    \varDelta=0
\end{equation}
\end{proposition}

\paragraph{Remark} This proposition shows that under deterministic labeling, the estimator $\hat S_{Cpt}$ accurately estimates the Y$\mid$X shift $S_{Cpt}^{\gamma^{*}}$ itself. For stochastic labeling, we show that $\varDelta$ can be bounded by the irreducible error \citep{james2013introduction} (also known as the Bayes risk \citep{berger2013statistical}) in traditional statistical learning.

\begin{proposition}[$\varDelta$ in Stochastic Labeling]
\label{pro:varDelta-in-stochastic-labeling}
When $(\mathcal Y,\rho_{\mathcal Y}) = (\mathbb R^{d'}, \|\cdot\|)$, the irreducible error of joint distribution $\mathcal D_{XY}$ under squared loss $\|\cdot\|^2$ is defined as: $\mathrm I(\mathcal D_{XY})
= \inf_{g:\mathcal X\to\mathcal Y}
  \mathbb E_{(x,y)\sim\mathcal D_{XY}}\bigl[\|y-g(x)\|^{2}\bigr]$, then the $\varDelta$ in Theorem~\ref{thm:concentration-total-pair} satisfies:
\begin{equation}
0 \;\le\; \varDelta 
\;\le\; \sqrt{\mathrm I(\mathcal D_{XY}^{S})} \;+\;\sqrt{\mathrm I(\mathcal D_{XY}^{T})}
\end{equation}
\end{proposition}

\paragraph{Remark} The irreducible error is the fundamental error inherent to stochastic labeling that no model can overcome. When the problem is learnable, covariate and label are often well correlated, and the irreducible error is small relative to the overall label variability. In this case, Proposition~\ref{pro:varDelta-in-stochastic-labeling} guarantees that the estimator $\hat S_{Cpt}$ does not substantially overestimate the Y$\mid$X shift $S_{Cpt}^{\gamma^{*}}$.

\subsection{DataShifts Algorithm}
\paragraph{On the Lipschitz Constant of Learners.} The Lipschitz constant of the learner have been well studied, such as logistic regression \citep{roux2012stochastic}, multi-layer perceptron (MLP) \citep{fazlyab2019efficient}, convolutional neural networks (CNN) \citep{virmaux2018lipschitz,zou2019lipschitz} and attention mechanism \citep{kim2021lipschitz,castin2023smooth}. Although the Lipschitz constants of modern large-scale neural networks such as ResNet-50 or Transformers remain difficult to analyze, researchers are more concerned with whether these models learn distribution-robust representations \citep{hendrycks2020pretrained}, rather than distribution shift in the raw input space. In this setting, the covariate space $\mathcal X$ is the representation space, and the downstream model (the hypothesis $h$ in our theory) is often a simple learner—such as a linear classifier—whose Lipschitz constant is still easy to handle. We summarize the Lipschitz constants of various learners in Appendix~\ref{app:lipschitz}.

\begin{algorithm}
   \caption{DataShifts}
   \label{alg:datashifts}
\begin{algorithmic}
   \STATE {\bfseries Input:} hyperparameter $\beta$ (default $0.01$),
   \STATE \quad \quad \quad samples $\{{(X_{i}^{(S)},Y_{i}^{(S)})}\}$, $\{{(X_{j}^{(T)},Y_{j}^{(T)})}\}$,
   \STATE \quad \quad \quad Lipschitz constants $L_{\ell},L_{\ell}^{\prime},L_{h}$ (optional),
   \STATE \quad \quad \quad source domain empirical error $\hat \epsilon_{S}$ (optional)
   \STATE {\bfseries Do:}
   \STATE  Estimate X shift by \ref{def:debiased-estimator} as $\hat S_{Cov}$
   \STATE Estimate $\gamma^{*}$-Y$\mid$X shift by \ref{def:estimator-total-pair-yx} as $\hat S_{Cpt}$
   \IF{$L_{\ell},L_{\ell}^{\prime},L_{h}$, and $\hat \epsilon_{S}$ are provided}
   \STATE  Estimate bound: $B=\hat \epsilon_{S}+L_{h}L_{\ell}^{\prime}\hat S_{Cov}+L_{\ell}\hat S_{Cpt}$
   \ENDIF
   \STATE {\bfseries Return:} $\hat S_{Cov}$, $\hat S_{Cpt}$ and $B$ (optional)
\end{algorithmic}
\end{algorithm}

\begin{figure*}[!t]
  \centering
  \subfigure[Novozymes]{%
    \includegraphics[height=0.323\textwidth]{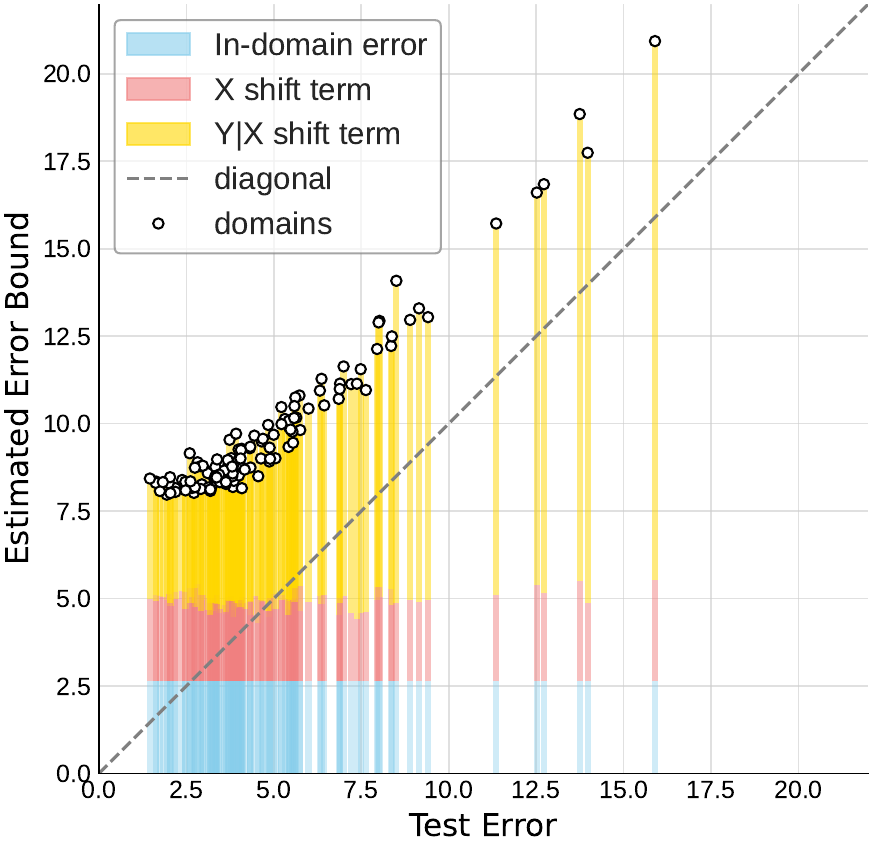}
    \label{fig:novo}%
  }\hfill
  \subfigure[ColoredMNIST]{%
    \includegraphics[height=0.323\textwidth]{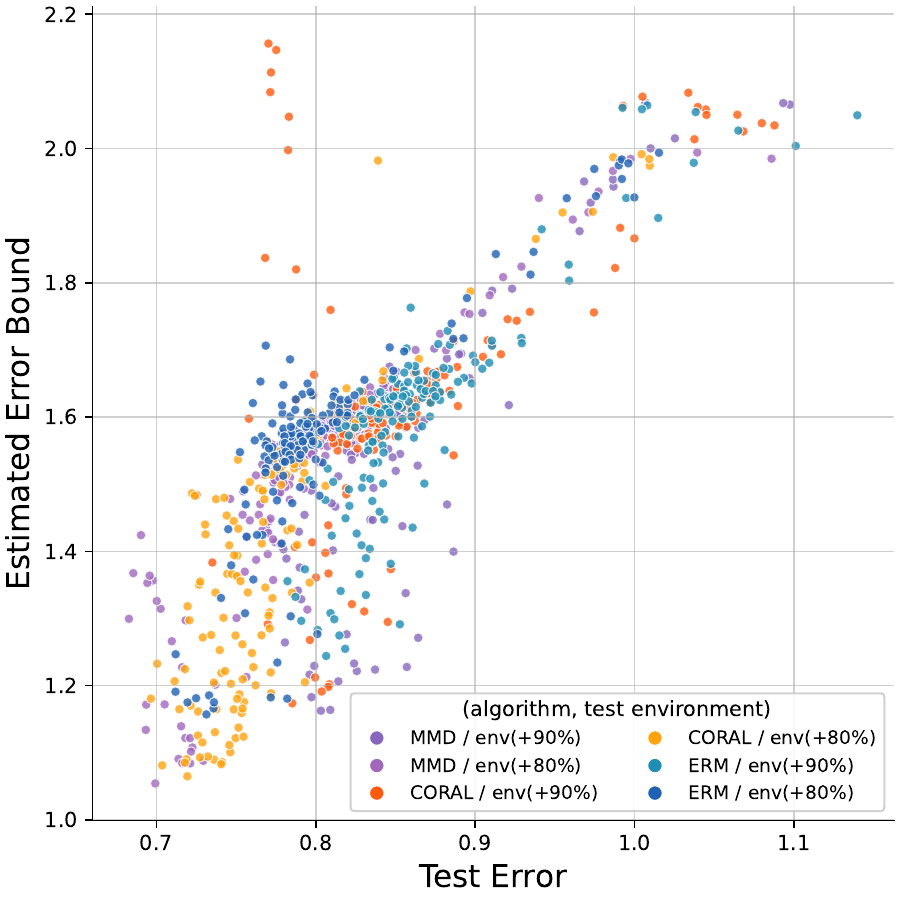}
    \label{fig:coloredmnist}%
  }\hfill
  \subfigure[PACS]{%
    \includegraphics[height=0.323\textwidth]{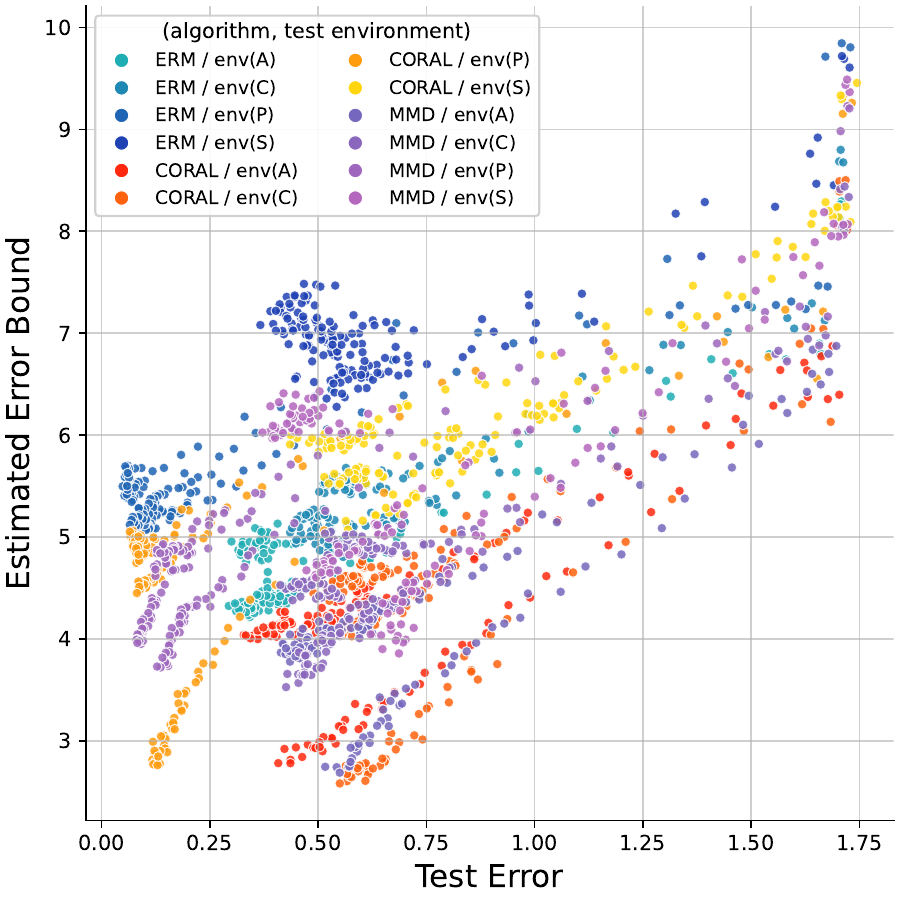}
    \label{fig:pacs}%
  }
  \caption{ (a)–(c) show that the estimated error bounds track the test error well across three distinct tasks, corroborating the general effectiveness of our learning bound and estimators.}
  \label{fig:combined2}
\end{figure*}

By leveraging the theoretical results above, we give the plug-and-play Algorithm~\ref{alg:datashifts} (DataShifts) for quantifying X and Y$\mid$X shifts from finite samples and estimating error bounds.

\section{Experiments}
\label{sec:experiments}
In this section, we apply our estimable general theory to three distinct practical tasks: tabular regression, image binary classification, and image multi-class classification -- to validate the general effectiveness of our bound and estimators. We also conduct experiments on synthetic tasks where existing theory apply, demonstrating our bound is tighter.

\subsection{Novozymes Enzyme Stability Prediction}
The Novozymes Enzyme Prediction Competition \citep{novozymes-enzyme-stability-prediction} is a large‐scale Kaggle contest. It is a tabular regression task, with 9,000 point-mutation samples spanning 180 enzyme families, aiming to predict transition temperatures for unseen families. Each enzyme family is treated as a separate domain; due to distribution shifts between them, thousands of participants found it difficult to develop any effective solution. 

We select the 60 enzyme families with the smallest pretraining error as the source domain, and treat each of the remaining 120 enzyme families as a target domain. In this task, we focus on the distribution shift between the raw data of different enzyme families, taking the 20-dimensional input feature space as covariate space $\mathcal{X}$. We train a 3-layer MLP on the source domain; an analysis of its Lipschitz constant is provided in the Appendix~\ref{app:lipschitz-mlp}. We use the absolute loss (separately $(1,1)$-Lipschitz) to evaluate the source domain error and the test error on each target domain, and apply our DataShifts algorithm($\beta=0.2$) to estimate an error bound for each target domain.

We plot the test error and the estimated error bound on each target domain in Fig.~\ref{fig:novo}. In this figure, the overall trend of the test error and the error bound lies just above the diagonal, indicating that our bound is tight and effectively captures the test error under distribution shift. Meanwhile, it directly shows the contributions of X and Y$\mid$X shifts on the error bound. The large Y$\mid$X shift across enzyme families is what drives the generalization failure in this contest.

\begin{figure*}[!t]
  \centering
  \subfigure[Error vs. X shift]{%
    \includegraphics[height=0.245\textwidth]{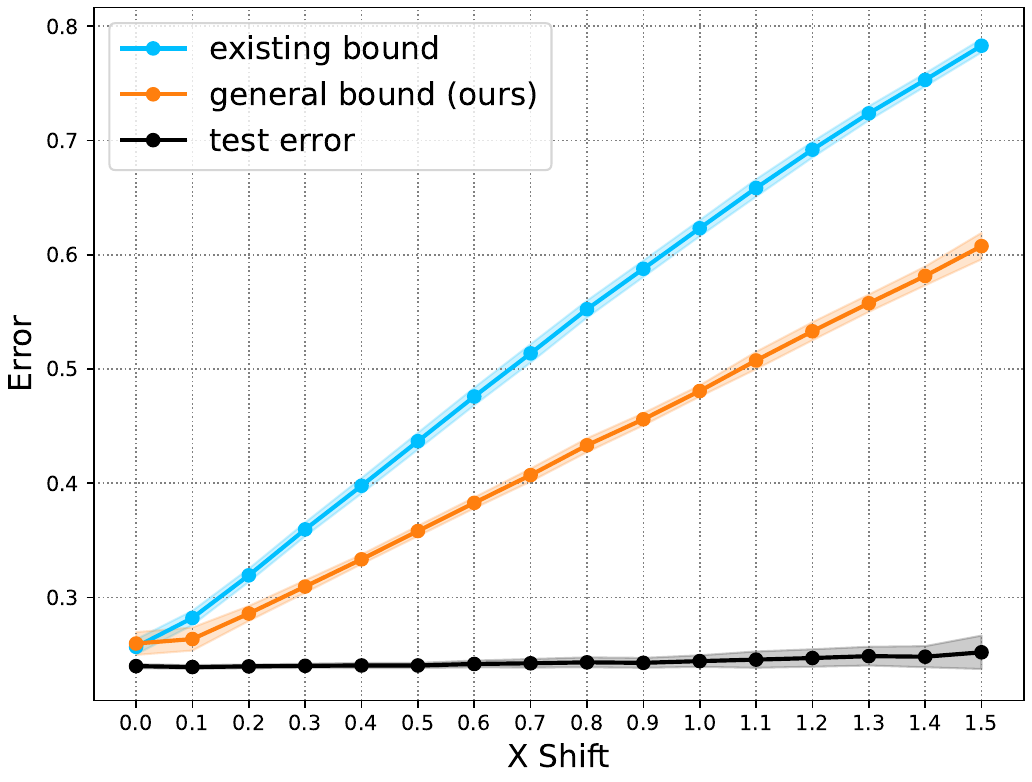}
    \label{fig:3a}%
  }\hfill
  \subfigure[Error vs. Y$\mid$X shift ($\theta$)]{%
    \includegraphics[height=0.245\textwidth]{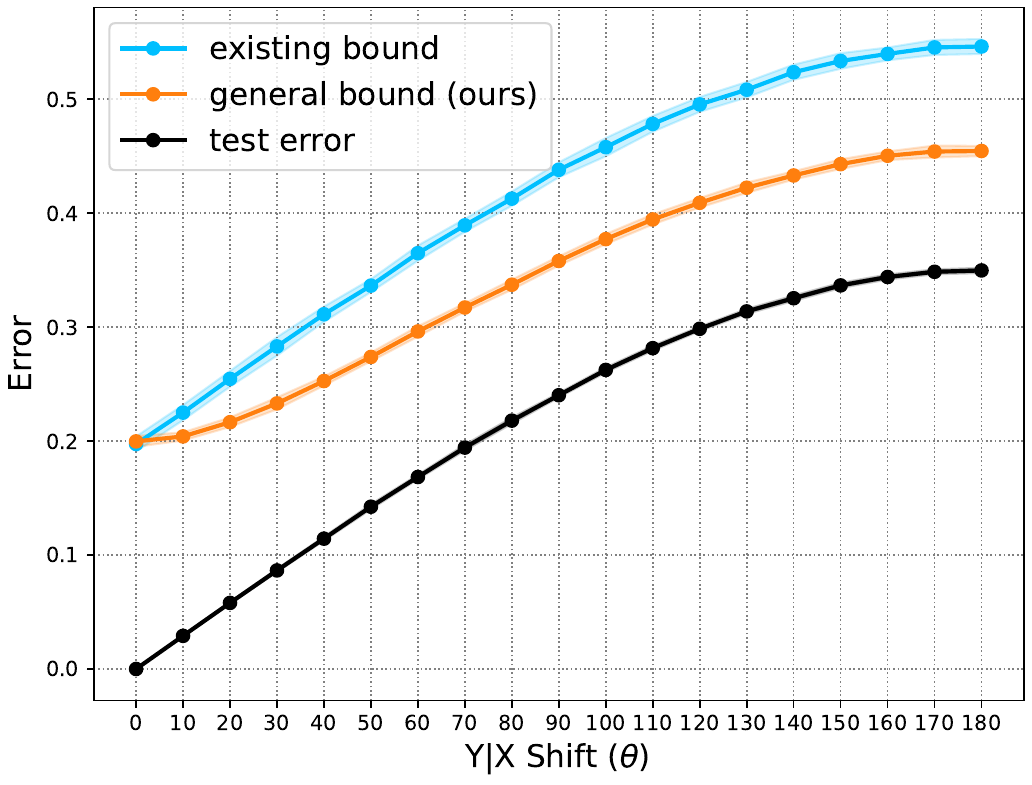}
    \label{fig:3b}%
  }\hfill
  \subfigure[Error vs. Y$\mid$X shift ($b$)]{%
    \includegraphics[height=0.245\textwidth]{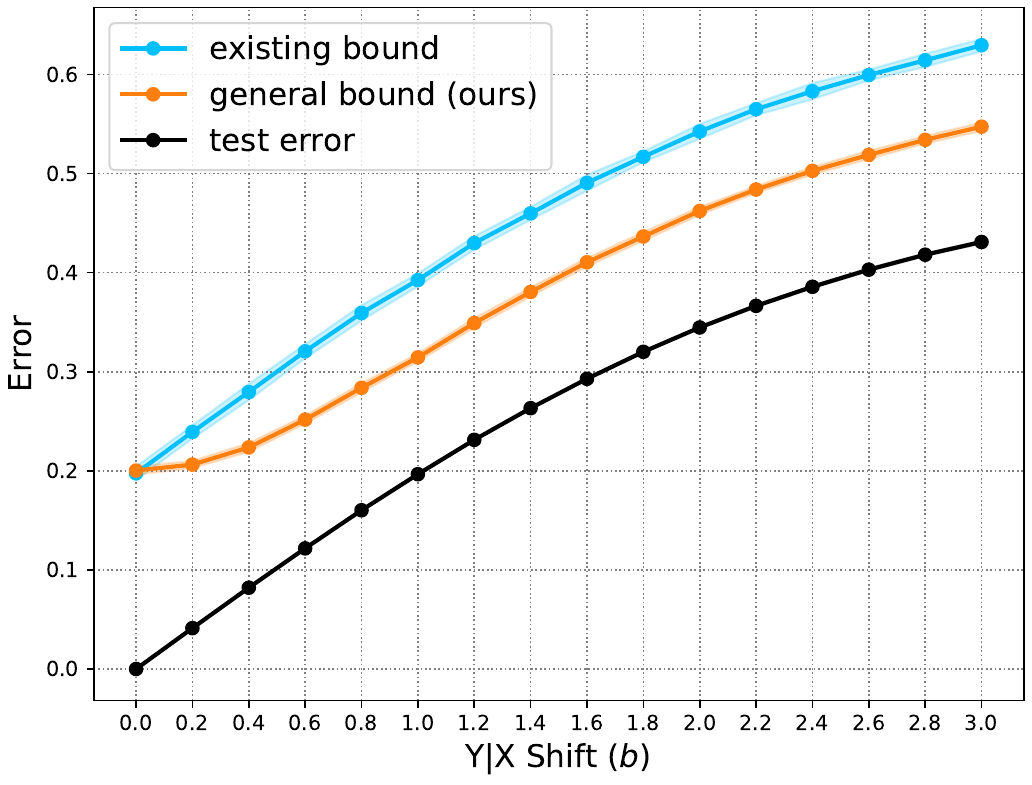}
    \label{fig:3c}%
  }
  \caption{ (a)–(c) respectively show the relationship between the estimated learning bounds and the X or Y$\mid$X shift in synthetic binary classification tasks. As these shifts increase, our bound becomes significantly tighter than the existing bound.}
  \label{fig:comparison}
\end{figure*}

\subsection{ColoredMNIST and PACS}
ColoredMNIST and PACS are two standard tasks in the DomainBed benchmark\citep{gulrajani2020search}. ColoredMNIST is a binary classification task with 70,000 handwritten digit images exhibiting color shift, while PACS is a multi-class object recognition task with 9,991 images exhibiting style shift. For both tasks, we treat the model’s representation space as the covariate space $\mathcal{X}$. Following the DomainBed setup, we use the simple CNN with a 128-dimensional representation for ColoredMNIST, and ResNet-50 with a 2048-dimensional representation for PACS. In this setting, the hypothesis $h$ in our theory corresponds to the model’s final layer (a linear classifier), whose Lipschitz constant is analyzed in the Appendix~\ref{app:lipschitz-linear-classifier}. 

We evaluate three methods: ERM, CORAL, and MMD. Adhering to DomainBed, we select the best hyperparameters via training-domain validation over 20 random hyperparameter trials for each method, domain, and trial. For ColoredMNIST, we train each best-hyperparameter run for 5,000 steps and save checkpoints every 100 steps. For PACS, since the model converges earlier, we train for 1,000 steps and save checkpoints every 20 steps. At each checkpoint, we regard the mixture distribution over training domains as the source domain and the test domain as the target domain. We still use the absolute loss to measure source and target errors, and run our DataShifts algorithm ($\beta=0.2$) using each domain’s representation-label pairs to estimate the test domain error bound at every checkpoint.

We plot the test error and the estimated error bound for both tasks in Fig.~\ref{fig:coloredmnist} and~\ref{fig:pacs}. In Fig.~\ref{fig:coloredmnist}, the bound tracks the test error well across checkpoints. The MMD (purple) and CORAL (yellow) points lie closer to the lower-left region, indicating that these X shift-reducing methods indeed yield smaller bounds, and consequently lower error on ColoredMNIST. In Fig.~\ref{fig:pacs}, although the bound becomes looser in magnitude when applied to PACS, a more complex image classification task, it still exhibits a consistent trend with the test error in each run.

\subsection{Synthetic Binary Classification}
We compare our theory with the existing bound in \citet{zhao2019learning}, which is shown to be tighter than previous learning bounds. As discussed in Remark\hyperref[remark:weakness]{*}, the existing learning bound only applies to soft-label binary classification with deterministic labeling and absolute loss. Since its estimation requires sampling the true concepts $f_S$ and $f_T$ outside the supports of the covariate distributions $\mathcal D_{X}^{S}$, $\mathcal D_{X}^{T}$ (oracle), it can only be estimated on synthetic tasks where the true concepts are known. 

We construct a synthetic binary classification task using logistic regression. Let the covariate space be 10-dimensional, and the source inputs are sampled from standard normal distribution: $\mathcal D^S_X=\mathcal{N}(\mathbf{0},I)$, with labels generated by $f_S(x)=\sigma(w_S^\top x+b_S)$, where $w_S=\mathbf 1/\sqrt{10}$, $b_S=0$ and $\sigma$ is the sigmoid function. The target inputs are generated by shifting $\mathcal D^S_X$ along a random direction, thereby controlling the X shift. And the target labels are generated by another logistic regression: $f_T(x)=\sigma(w_T^\top x+b_T)$, where $w_T$ is obtained by rotating $w_S$ by angle $\theta$ toward a random direction. $\theta$ and $b_T$ further control the Y$\mid$X shift. By varying the X shift, $\theta$, and $b_T$, we obtain a series of target domains. We also use logistic regression as the learner and train it on the source domain. Its Lipschitz constant is analyzed in Appendix~\ref{app:lipschitz-logistic-regression}. For each target domain, we estimate the test error, the existing and our bounds under the absolute loss.

We plot the learning bounds with respect to the X shift and the Y$\mid$X shifts ($\theta$ and $b_T$) in Fig.~\ref{fig:comparison}. In Fig.~\ref{fig:3a}, since both the source concept and the learner are logistic regression models, the learner can fit the source concept well, and the X shift in the target domain has only a minor effect on the test error. As the X shift increases, our bound becomes tighter than the existing bound. In Figs.~\ref{fig:3b} and~\ref{fig:3c}, as the Y$\mid$X shift increases, the existing bound becomes looser than ours. This verifies our point in Remark\hyperref[remark:weakness]{*}: the ill-defined Y$\mid$X shift in the existing bound is loose, whereas our $\gamma^{*}$-Y$\mid$X shift and the accompanying learning bound can be tighter.

In addition to the above results, sensitivity experiments of the proposed estimators with respect to the parameter $\beta$ are provided in Appendix~\ref{app:sensitivity}. Experiments on the bias of the $\gamma^{*}$-Y$\mid$X shift estimator under stochastic labeling are presented in Appendix~\ref{app:overest}.

\section{Conclusion}
\label{sec:conclusion}

In this paper, we focus on a general and estimable theoretical framework for learning under distribution shift. We first introduce a key notion, $\gamma^{*}$-Y$\mid$X shift, via entropic optimal transport, which addresses the ill-definedness in existing theory. Then we derive a general learning bound unifying X shift and $\gamma^{*}$-Y$\mid$X shift. We further develop concentration-guaranteed estimators for both shifts, and integrate our theory into the plug-and-play DataShifts algorithm, enabling researchers to quantify and analyze distribution shift in broad settings. Experiments on practical and synthetic tasks validate the general effectiveness and tightness of our theory. We believe our theoretical framework takes an important step toward learning under distribution shift and will spur further algorithmic advances.

\section*{Acknowledgements}
We also thank Dr. Jie Ren for suggestions and feedback on this work. This study was funded by the National Natural Science Foundation of China (12571529) and Guangdong Basic and Applied Basic Research Foundation (2024A1515-010699) to LCX.

\section*{Impact Statement}

This paper presents work whose goal is to advance the field of Machine
Learning. There are many potential societal consequences of our work, none
which we feel must be specifically highlighted here.


\bibliography{example_paper}
\bibliographystyle{icml2026}

\newpage
\appendix
\onecolumn


\section{Related Work}
\label{sec:related}
Generalization theory under distribution shift is an important and long-standing problem. Researchers usually seek to measure distribution shift using some distribution divergence and derive corresponding learning bounds. Early theoretical works mainly focused on the X (covariate) shift. \citet{ben2006analysis,ben2010theory} used the $\mathcal H$-divergence to measure X shift between the source and target domains for binary classification, thereby deriving learning bounds for domain adaptation. Subsequent studies characterized the X shift via maximum mean discrepancy (MMD), leading to new learning bounds and improved domain adaptation algorithms \citep{long2015learning}. \citet{redko2017theoretical,courty2017joint,shen2018wasserstein} adopted optimal transport (OT) to characterize X shift and obtained similar bounds. Further studies proposed more complex metrics for the X shift to extend the theory to multiclass classification \citep{zhang2019bridging,zhang2020unsupervised,el2025theoretical}.

Most of these theories focus only on the X shift and follow a similar proof strategy, resulting in a joint error term between the source and target domains, $\lambda=\min_{h}\epsilon_S(h)+\epsilon_T(h)$, which is widely recognized as loose and non-estimable. Our baseline, \citet{zhao2019learning}, showed that the Y$\mid$X (concept) shift is implicitly contained in this joint error term $\lambda$, and derived a tighter learning bound by explicitly introducing the Y$\mid$X shift instead. \citet{zhang2023nico} extended this theory to multiclass classification. We follow this path of explicitly characterizing the Y$\mid$X shift, and point out that the existing definition of the Y$\mid$X shift becomes ill-defined when the supports of the covariate distributions are mismatched, which still makes the Y$\mid$X shift loose and non-estimable. Through entropic optimal transport, we introduce a well-defined $\gamma^{*}$-Y$\mid$X shift and the corresponding learning bound, which can be rigorously estimated from samples. Our theory can be tighter than existing bounds and further generalizes to a broader tasks, losses, and stochastic labeling.

\newcolumntype{C}[1]{>{\centering\arraybackslash}m{#1}}

\begin{table*}[h]
\centering
\caption{Comparison of learning bounds for distribution shift.}
\label{tab:bound-comparison}
\small
\setlength{\tabcolsep}{4pt}
\renewcommand{\arraystretch}{1.18}
\begin{threeparttable}
\begin{tabular}{@{}C{3.5cm}C{2.0cm}C{1.10cm}C{2.40cm}C{3.00cm}C{0.95cm}C{1.15cm}C{0.95cm}@{}}
\toprule
\multirow{2}{*}{\makecell{Bound}}
& \multirow{2}{*}{\makecell{Divergence}}
& \multicolumn{3}{c}{Terms}
& \multicolumn{3}{c}{Properties} \\
\cmidrule(lr){3-5} \cmidrule(l){6-8}
& & X shift & Y$\mid$X shift & Remaining Term
& Tight & Estimable & General \\
\midrule
\makecell{\citet{ben2006analysis}}
& $\mathcal H$-divergence
& \checkmark
& 
& joint-error $\lambda$
& 
& 
& \\

\makecell{\citet{ben2010theory}}
& $\mathcal H$-divergence
& \checkmark
& 
& joint-error $\lambda$
& 
& 
& \\

\makecell{\citet{long2015learning}}
& MMD
& \checkmark
& 
& joint-error $\lambda$
& 
& 
& \\

\makecell{\citet{redko2017theoretical}}
& OT
& \checkmark
& 
& joint-error $\lambda$
& 
& 
& \\

\makecell{\citet{courty2017joint}}
& OT
& \multicolumn{2}{>{\centering\arraybackslash}m{4.0cm}}{estimated XY shift}
& \makecell{joint-error $\lambda$, $kM\Phi$}
& 
& 
& \checkmark \\

\makecell{\citet{shen2018wasserstein}}
& OT
& \checkmark
& 
& joint-error $\lambda$
& 
& 
& \\

\makecell{\citet{zhao2019learning}}
& $\mathcal H$-divergence
& \checkmark
& \makecell{\checkmark (ill-defined)}
& 
& \checkmark
& 
& \\

\makecell{\citet{zhang2023nico}}
& $\mathcal H$-divergence
& \checkmark
& \makecell{\checkmark (ill-defined)}
& 
& \checkmark
& 
& \checkmark \\

\makecell{\citet{el2025theoretical}}
& hierarchical OT
& \checkmark
& 
& joint-error $\lambda$
& 
& 
& \\

\textbf{Ours}
& entropic OT
& \checkmark
& \checkmark
& 
& \checkmark\checkmark
& \checkmark
& \checkmark\checkmark \\
\bottomrule
\end{tabular}
\begin{tablenotes}
\footnotesize
\item \checkmark\checkmark indicates a stronger degree than \checkmark.
\end{tablenotes}
\end{threeparttable}
\end{table*}

Distribution shift theory often serves as a theoretical framework for domain adaptation and domain generalization, where many algorithms have been developed \citep{sugiyama2007covariate,glorot2011domain,sun2016deep,ganin2016domain,pei2018multi,arjovsky2019invariant,krueger2021out,rame2022fishr}. However, the DomainBed benchmark \citep{gulrajani2020search} shows that many algorithms do not outperform standard empirical risk minimization. This calls for tighter and more general theoretical frameworks to explain the generalization failures of these algorithms. In addition, the proposed $\gamma^{*}$-Y$\mid$X shift takes a nested OT form, whose other mathematical properties have also been studied recently, such as the gradient flow of the Wasserstein-over-Wasserstein distance \citep{bonet2025flowing}.

\section{Technical Tools}
In this section, we introduce several existing mathematical tools to prove the propositions in Appendix~\ref{app:full-proofs}.

\subsection{Conditional Probability}
\begin{definition}[Conditional Probability]
\label{def:conditional-probability}
Let $(\Omega,\mathcal F,P)$ be a probability space and $X:\Omega\to\mathcal X$ a random variable with distribution $P_X$. For event $A\in\mathcal F$, a $P_X$-measurable function $P(A\mid X=x):\mathcal X\to[0,1]$ is the conditional probability of $A$ given $X$ if:
\[
\forall B\in\mathcal B_{\mathcal X},\,
P(A\cap\{X\in B\}) = \int_B P(A\mid X=x)\,dP_X(x)
\]
where $\mathcal B_{\mathcal X}$ denotes the Borel $\sigma$-algebra on $\mathcal X$.
\end{definition}
The conditional probability is unique almost everywhere with respect to $P_X$ (unique $P_X$-a.e.) \citep{kallenberg1997foundations}. That is, for any $P_X$-null set $Z$ with $P_X(Z)=0$, $\int_{Z} P(A\mid X=x)\,dP_X(x) = 0$, 
implying that $P(A\mid X=x)$ can be assigned arbitrarily on $Z$ without affecting its overall properties. Intuitively, it means discussing conditional probabilities on a null set is meaningless \citep{hajek2003conditional}.

\subsection{McDiarmid's Inequality}
\begin{lemma}[McDiarmid's Inequality]
\label{lem:mcdiarmid}
Let $Z_1,\ldots,Z_n$ be independent random variables taking values in sets
$\mathcal Z_1,\ldots,\mathcal Z_n$, and let
$F:\mathcal Z_1\times\cdots\times\mathcal Z_n\to\mathbb R$.
Assume $F$ satisfies the bounded-differences property: there exist constants
$c_1,\ldots,c_n\ge 0$ such that for every $i\in[n]$ and any two inputs
$(z_1,\ldots,z_n)$ and $(z_1,\ldots,z_i',\ldots,z_n)$ differing only at the $i$-th coordinate,
\[
\bigl|F(z_1,\ldots,z_i,\ldots,z_n)-F(z_1,\ldots,z_i',\ldots,z_n)\bigr|
\;\le\; c_i .
\]
Then for any $\varepsilon>0$,
\[
\mathbb P\Bigl(
F(Z_1,\ldots,Z_n)-\mathbb E\bigl[F(Z_1,\ldots,Z_n)\bigr]\ge \varepsilon
\Bigr)
\;\le\;
\exp\!\left(
-\frac{2\varepsilon^2}{\sum_{i=1}^n c_i^2}
\right),
\]
and
\[
\mathbb P\Bigl(
\bigl|F(Z_1,\ldots,Z_n)-\mathbb E[F(Z_1,\ldots,Z_n)]\bigr|\ge \varepsilon
\Bigr)
\;\le\;
2\exp\!\left(
-\frac{2\varepsilon^2}{\sum_{i=1}^n c_i^2}
\right).
\]
\end{lemma}

\subsection{Concentration for Empirical $W_1$}

\begin{lemma}[Square-Exponential Moment Implies $T_1$ \citep{bolley2007quantitative}]
\label{lem:sqexp-implies-T1}
Let $(\mathcal X,\rho)$ be a Polish metric space and $\mu$ a Borel probability measure on $\mathcal X$.
Assume $\mu$ admits a square-exponential moment: there exist $a>0$ and $x_0\in\mathcal X$ such that
\[
\int_{\mathcal X}\exp\!\bigl(a\,\rho(x,x_0)^2\bigr)\,d\mu(x) < \infty.
\]
Then $\mu$ satisfies a $T_1(\lambda)$ transport-entropy inequality for some $\lambda>0$:
for all probability measures $\nu$ on $\mathcal X$,
\[
W_1(\mu,\nu)\le \sqrt{\frac{2}{\lambda}\,H(\nu\mid\mu)}.
\]
\end{lemma}

\begin{lemma}[Bolley--Guillin Concentration for Empirical $W_1$ \citep{bolley2007quantitative}]
\label{lem:bolley}
Let $\mu$ be a probability measure on $(\mathbb R^{d},\|\cdot\|)$ that satisfies the transport
inequality $T_1(\lambda)$ for some $\lambda>0$, namely for all probability measures $\nu$,
\[
W_1(\mu,\nu)\;\le\;\sqrt{\frac{2}{\lambda}\,H(\nu\mid\mu)},
\]
where $H(\nu\mid\mu)$ denotes the relative entropy.
Let $\hat\mu=\frac1N\sum_{i=1}^N\delta_{X_i}$ be the empirical measure of i.i.d.\ samples
$X_1,\ldots,X_N\sim\mu$.
Then for any $d'>d$ and any $\lambda_\mu\in(0,\lambda)$, there exists a constant
$N_0$ depending only on $\lambda_\mu$, $d'$, and the squared-exponential moments of $\mu$
such that for any $\varepsilon>0$ and any
\[
N\;\ge\; N_0\max\bigl(\varepsilon^{-(d'+2)},\,1\bigr),
\]
we have
\[
\mathbb P\bigl(W_1(\mu,\hat\mu)>\varepsilon\bigr)
\;\le\;
\exp\!\left(-\frac{\lambda_\mu}{2}\,N\,\varepsilon^2\right).
\]
\end{lemma}

\section{Full proofs}
\label{app:full-proofs}

\subsection{Proof of Lemma~\ref{lem:ill-defined-expectation}}

\begin{proof}
Since $\operatorname{supp}(P_X^{\prime}) \setminus \operatorname{supp}(P_X)\neq\varnothing$,
pick $x_0\in \operatorname{supp}(P_X^{\prime}) \setminus \operatorname{supp}(P_X)$.
By Definition~\ref{def:support}, $x_0\notin\operatorname{supp}(P_X)$ implies that there
exists an open neighborhood $U\subseteq\mathcal X$ with $x_0\in U$ such that
\[
P_X(U)=0.
\]
On the other hand, $x_0\in\operatorname{supp}(P_X^{\prime})$ implies $P_X^{\prime}(V)>0$
for every open neighborhood $V\ni x_0$, hence in particular $P_X^{\prime}(U)>0$.
Let $Z:=U$. Then $Z\in\mathcal B_{\mathcal X}$ and
\[
P_X(Z)=0,
\qquad
P_X^{\prime}(Z)>0.
\]

Let $P(A\mid X=x):\mathcal X\to[0,1]$ be a conditional probability in Definition~\ref{def:conditional-probability}, i.e., for all $B\in\mathcal B_{\mathcal X}$,
\[
P\bigl(A\cap\{X\in B\}\bigr)=\int_B P(A\mid X=x)\,dP_X(x).
\]
For any constant $c\in[0,1]$, define a modified function
\[
\widetilde P(A\mid X=x):=
\begin{cases}
P(A\mid X=x), & x\notin Z,\\
c, & x\in Z.
\end{cases}
\]
For any $B\in\mathcal B_{\mathcal X}$,
\begin{align*}
\int_B \widetilde P(A\mid X=x)\,dP_X(x)
&=\int_{B\setminus Z} P(A\mid X=x)\,dP_X(x)
  +\int_{B\cap Z} c\,dP_X(x) \\
&=\int_{B\setminus Z} P(A\mid X=x)\,dP_X(x) + c\,P_X(B\cap Z) \\
&=\int_{B\setminus Z} P(A\mid X=x)\,dP_X(x)
\qquad (\text{since } P_X(Z)=0) \\
&=\int_B P(A\mid X=x)\,dP_X(x)
= P\bigl(A\cap\{X\in B\}\bigr),
\end{align*}
so $\widetilde P(A\mid X=x)$ is also a valid conditional probability of $A$
given $X$ by Definition~\ref{def:conditional-probability}, as another version of $P(A\mid X=x)$.

Now take expectation under $P_X^{\prime}$:
\begin{align*}
\mathbb E_{P_X^{\prime}}[\,\widetilde P(A\mid X=x)\,]
&=\int_{\mathcal X} \widetilde P(A\mid X=x)\,dP_X^{\prime}(x) \\
&=\int_{\mathcal X\setminus Z} P(A\mid X=x)\,dP_X^{\prime}(x)
  +\int_Z c\,dP_X^{\prime}(x) \\
&=\int_{\mathcal X\setminus Z} P(A\mid X=x)\,dP_X^{\prime}(x)
  + c\,P_X^{\prime}(Z).
\end{align*}
Since $P_X^{\prime}(Z)>0$ and $c\in[0,1]$ is arbitrary, the value of
$\mathbb E_{P_X^{\prime}}[P(A\mid X=x)]$ can be changed by choosing different $c$ while
preserving the defining property of conditional probability under $P_X$. Hence
$\mathbb E_{P_X^{\prime}}[P(A\mid X=x)]$ is arbitrary.
\end{proof}

\paragraph{Remark} The above proof shows that support mismatch leads to arbitrariness. Broadly speaking, support mismatch commonly breaks the rigor of theoretical analyses. Some studies explicitly exclude it from analysis by extra assumptions, such as done in the absolute continuity in measure theory \citep{duncan1970absolute}, the positivity assumption in causal inference \citep{cole2009consistency}, or support overlap in importance sampling \citep{gelman1998simulating}.

\subsection{Proof of Corollary~\ref{cor:consistency-under-deterministic-labeling}}
\begin{proof}
By Definition~\ref{def:total-pair-yx-shift},
\[
S_{Cpt}^{\gamma^*}
=\mathbb E_{(x_S,x_T)\sim\gamma^*}\!\left[
W_1\!\bigl(\mathcal D_{Y|X=x_S}^{S},\mathcal D_{Y|X=x_T}^{T}\bigr)\right].
\]
Under deterministic labeling, $\mathcal D_{Y|X=x_S}^{S}=\delta_{f_S(x_S)}$ and
$\mathcal D_{Y|X=x_T}^{T}=\delta_{f_T(x_T)}$, hence
\[
S_{pair}(x_S,x_T)=W_1\!\bigl(\delta_{f_S(x_S)},\delta_{f_T(x_T)}\bigr).
\]
For any $a,b\in\mathcal Y$, any coupling $\pi\in\Gamma(\delta_a,\delta_b)$ must satisfy
$\pi(\{a\}\times\mathcal Y)=1$ and $\pi(\mathcal Y\times\{b\})=1$, hence
$\pi=\delta_{(a,b)}$ is the unique coupling. Therefore,
\[
W_1(\delta_a,\delta_b)
=\inf_{\pi\in\Gamma(\delta_a,\delta_b)}\int\rho_{\mathcal Y}(y_1,y_2)\,d\pi
=\int\rho_{\mathcal Y}(y_1,y_2)\,d\delta_{(a,b)}(y_1,y_2)
=\rho_{\mathcal Y}(a,b).
\]
Applying this with $a=f_S(x_S)$ and $b=f_T(x_T)$ yields
$S_{pair}(x_S,x_T)=\rho_{\mathcal Y}(f_S(x_S),f_T(x_T))$, and thus
\[
S_{Cpt}^{\gamma^*}
=\mathbb E_{(x_S,x_T)\sim\gamma^*}\bigl[\rho_{\mathcal Y}(f_S(x_S),f_T(x_T))\bigr],
\]
as claimed.
\end{proof}

\subsection{Proof of Lemma~\ref{lem:support-of-gamma}}
\begin{proof}
Any $\gamma\in\Gamma(\mathcal D_X^{S},\mathcal D_X^{T})$ is a coupling of
$\mathcal D_X^{S}$ and $\mathcal D_X^{T}$, hence its first marginal is
$\mathcal D_X^{S}$ and its second marginal is $\mathcal D_X^{T}$.

Take any $(x_S,x_T)\notin \operatorname{supp}(\mathcal D_X^{S})\times
\operatorname{supp}(\mathcal D_X^{T})$. Then either
$x_S\notin\operatorname{supp}(\mathcal D_X^{S})$ or
$x_T\notin\operatorname{supp}(\mathcal D_X^{T})$.

If $x_S\notin\operatorname{supp}(\mathcal D_X^{S})$, by Definition~\ref{def:support}
there exists an open neighborhood $U\subseteq\mathcal X$ of $x_S$ such that
$\mathcal D_X^{S}(U)=0$. Since the first marginal of $\gamma$ is $\mathcal D_X^{S}$,
\[
\gamma(U\times\mathcal X)=\mathcal D_X^{S}(U)=0.
\]
In particular, for any open neighborhood $V$ of $x_T$ we have
$\gamma(U\times V)\le \gamma(U\times\mathcal X)=0$, so $\gamma(U\times V)=0$.
Thus, taking the product open set $U\times V$ containing $(x_S,x_T)$, we conclude that
$(x_S,x_T)\notin\operatorname{supp}(\gamma)$ by Definition~\ref{def:support} applied
on the product topology of $\mathcal X\times\mathcal X$.

The case $x_T\notin\operatorname{supp}(\mathcal D_X^{T})$ is symmetric: there exists an
open $V\ni x_T$ with $\mathcal D_X^{T}(V)=0$, and using the second marginal of $\gamma$
gives $\gamma(\mathcal X\times V)=0$, hence $\gamma(U\times V)=0$ for any open
$U\ni x_S$, implying $(x_S,x_T)\notin\operatorname{supp}(\gamma)$.

Therefore every point outside
$\operatorname{supp}(\mathcal D_X^{S})\times\operatorname{supp}(\mathcal D_X^{T})$
is outside $\operatorname{supp}(\gamma)$, i.e.,
\[
\operatorname{supp}(\gamma) \subseteq 
\operatorname{supp}(\mathcal D_{X}^{S})\times\operatorname{supp}(\mathcal D_{X}^{T}).
\]
\end{proof}

\subsection{Proof of Lemma~\ref{lem:uniqueness-of-gamma}}
\begin{proof}
Let $\nu:=\mathcal D_X^{S}\otimes \mathcal D_X^{T}$ and consider the entropic optimal transport
objective
\[
J(\gamma)
:=\int \rho_{\mathcal X}(x_S,x_T)\,d\gamma(x_S,x_T)
+\beta\,H(\gamma\mid \nu),
\qquad \gamma\in\Gamma(\mathcal D_X^{S},\mathcal D_X^{T}).
\]

We claim that $J$ is \emph{strictly convex} over the convex set
$\Gamma(\mathcal D_X^{S},\mathcal D_X^{T})$. Indeed, the transport cost term
$\gamma\mapsto \int \rho_{\mathcal X}\,d\gamma$ is linear in $\gamma$. For the
entropy term, write $r=\frac{d\gamma}{d\nu}$; then
\[
H(\gamma\mid \nu)=\int r\log r\,d\nu.
\]
For any
$\gamma_1,\gamma_2 \in\Gamma(\mathcal D_X^{S},\mathcal D_X^{T})$ with $\gamma_1\neq \gamma_2$, by Lemma~\ref{lem:support-of-gamma},
\[
\operatorname{supp}(\gamma_1) \subseteq 
\operatorname{supp}(\mathcal D_{X}^{S})\times\operatorname{supp}(\mathcal D_{X}^{T})=\operatorname{supp}(\nu).
\]
\[
\operatorname{supp}(\gamma_2) \subseteq 
\operatorname{supp}(\mathcal D_{X}^{S})\times\operatorname{supp}(\mathcal D_{X}^{T})=\operatorname{supp}(\nu).
\]
Because $\varphi(t)=t\log t$ is strictly convex on $[0,\infty)$, for and any $\lambda\in(0,1)$, 
\[
H\bigl(\lambda\gamma_1+(1-\lambda)\gamma_2 \mid \nu\bigr)
<\lambda H(\gamma_1\mid \nu)+(1-\lambda)H(\gamma_2\mid \nu),
\]
hence (multiplying by $\beta>0$) the whole objective satisfies
\[
J\bigl(\lambda\gamma_1+(1-\lambda)\gamma_2\bigr)
<\lambda J(\gamma_1)+(1-\lambda)J(\gamma_2).
\]

Now suppose, toward a contradiction, that there are two distinct optimal
couplings $\gamma_1\neq\gamma_2$ minimizing $J$ over
$\Gamma(\mathcal D_X^{S},\mathcal D_X^{T})$. By convexity of
$\Gamma(\mathcal D_X^{S},\mathcal D_X^{T})$, their mixture
$\gamma_\lambda:=\lambda\gamma_1+(1-\lambda)\gamma_2$ is feasible. Strict
convexity then yields
\[
J(\gamma_\lambda)<\lambda J(\gamma_1)+(1-\lambda)J(\gamma_2)=\inf_{\gamma\in\Gamma}J(\gamma),
\]
a contradiction. Therefore the optimizer $\gamma^*$ is unique.
\end{proof}

\subsection{Proof of Lemma~\ref{lem:collapse-of-gamma}}
\begin{proof}
When $\beta=0$, Definition~\ref{def:x-shift} reduces to the Wasserstein-1 problem
\[
\inf_{\gamma\in\Gamma(\mathcal D_X^{S},\mathcal D_X^{T})}\int \rho_{\mathcal X}(x_S,x_T)\,d\gamma(x_S,x_T),
\]
where $\rho_{\mathcal X}$ is a metric, hence $\rho_{\mathcal X}\ge 0$ and
$\rho_{\mathcal X}(x_S,x_T)=0$ if $x_S=x_T$.

Consider the diagonal (identity) coupling
$\bar\gamma:=(\mathrm{Id},\mathrm{Id})_{\#}\mathcal D_{X}^{S}$, i.e.,
\[
\bar\gamma(A)=\int \mathbf{1}_{\{(x,x)\in A\}}\,d\mathcal D_{X}^{S}(x).
\]
It is immediate that $\bar\gamma\in\Gamma(\mathcal D_{X}^{S},\mathcal D_{X}^{S})$, and its transport cost is
\[
\int \rho_{\mathcal X}(x_S,x_T)\,d\bar\gamma(x_S,x_T)
=\int \rho_{\mathcal X}(x,x)\,d\mathcal D_{X}^{S}(x)=0.
\]
Therefore the optimal value is at most $0$, hence equals $0$.

Let $\gamma^{*}\in\Gamma(\mathcal D_{X}^{S},\mathcal D_{X}^{S})$ be any optimal coupling. Then
\[
0=\int \rho_{\mathcal X}(x_S,x_T)\,d\gamma^{*}(x_S,x_T).
\]
Since the integrand is nonnegative, this implies
$\rho_{\mathcal X}(x_S,x_T)=0$ holds $\gamma^{*}$-almost surely, i.e.,
$(x_S,x_T)\in\zeta:=\{(x,x):x\in\mathcal X\}$ $\gamma^{*}$-a.s. Hence
$\operatorname{supp}(\gamma^{*})\subseteq \zeta$.

Finally, because $\gamma^*$ is supported on $\zeta$ and has first marginal $\mathcal D_{X}^{S}$,
for every measurable $A\subseteq\mathcal X\times\mathcal X$ we have
\[
\gamma^{*}(A)
=\gamma^{*}(A\cap \zeta)
=\int \mathbf{1}_{\{(x,x)\in A\}}\,d\mathcal D_{X}^{S}(x)
=(\mathrm{Id},\mathrm{Id})_{\#}\mathcal D_{X}^{S}(A),
\]
which proves $\gamma^{*}=(\mathrm{Id},\mathrm{Id})_{\#}\mathcal D_{X}^{S}$ and the stated equivalent form.
\end{proof}

\subsection{Proof of Theorem~\ref{thm:uniqueness-total-pair}}
\begin{proof}
For $\beta>0$, Lemma~\ref{lem:uniqueness-of-gamma} ensures that the entropic optimal
transport coupling $\gamma^*$ in Definition~\ref{def:x-shift} is unique. Hence any
potential ambiguity of $S_{Cpt}^{\gamma^*}$ can only come from the choice of versions
of the conditional distributions.

Given any two versions of the source conditionals
$\{\mathcal D_{Y|X=x}^{S}\}_{x\in\mathcal X}$ and
$\{\widetilde{\mathcal D}_{Y|X=x}^{S}\}_{x\in\mathcal X}$, and any two versions of the
target conditionals
$\{\mathcal D_{Y|X=x}^{T}\}_{x\in\mathcal X}$ and
$\{\widetilde{\mathcal D}_{Y|X=x}^{T}\}_{x\in\mathcal X}$,
such that
$\mathcal D_{Y|X=x}^{S}=\widetilde{\mathcal D}_{Y|X=x}^{S}$ holds $\mathcal D_X^{S}$-a.e.,
and
$\mathcal D_{Y|X=x}^{T}=\widetilde{\mathcal D}_{Y|X=x}^{T}$ holds $\mathcal D_X^{T}$-a.e.
Therefore, there exist measurable sets $E_S,E_T\subseteq\mathcal X$ with
$\mathcal D_X^{S}(E_S)=0$ and $\mathcal D_X^{T}(E_T)=0$ such that the equalities hold
for all $x\notin E_S$ and all $x\notin E_T$, respectively.

Define the $\gamma^*$-Y$|$X shift under the two choices as
\[
S_{Cpt}^{\gamma^*}
:=\int W_1\!\bigl(\mathcal D_{Y|X=x_S}^{S},\mathcal D_{Y|X=x_T}^{T}\bigr)\,d\gamma^*(x_S,x_T),
\]
\[
\widetilde S_{Cpt}^{\gamma^*}
:=\int W_1\!\bigl(\widetilde{\mathcal D}_{Y|X=x_S}^{S},\widetilde{\mathcal D}_{Y|X=x_T}^{T}\bigr)\,d\gamma^*(x_S,x_T).
\]
Let $E:=(E_S\times\mathcal X)\,\cup\,(\mathcal X\times E_T)$.
For any $(x_S,x_T)\in E^{c}$ we have $x_S\notin E_S$ and $x_T\notin E_T$, hence
\[
\mathcal D_{Y|X=x_S}^{S}=\widetilde{\mathcal D}_{Y|X=x_S}^{S},
\qquad
\mathcal D_{Y|X=x_T}^{T}=\widetilde{\mathcal D}_{Y|X=x_T}^{T},
\]
which implies
\[
W_1\!\bigl(\mathcal D_{Y|X=x_S}^{S},\mathcal D_{Y|X=x_T}^{T}\bigr)
=
W_1\!\bigl(\widetilde{\mathcal D}_{Y|X=x_S}^{S},\widetilde{\mathcal D}_{Y|X=x_T}^{T}\bigr)
\quad\text{for all }(x_S,x_T)\in E^c.
\]
Consequently, the difference between the two versions satisfies
\[
S_{Cpt}^{\gamma^*}-\widetilde S_{Cpt}^{\gamma^*}
=\int_E
\Bigl[
W_1\!\bigl(\mathcal D_{Y|X=x_S}^{S},\mathcal D_{Y|X=x_T}^{T}\bigr)
-
W_1\!\bigl(\widetilde{\mathcal D}_{Y|X=x_S}^{S},\widetilde{\mathcal D}_{Y|X=x_T}^{T}\bigr)
\Bigr]\,d\gamma^*(x_S,x_T),
\]
By Lemma~\ref{lem:support-of-gamma}, $\operatorname{supp}(\gamma^*)\subseteq
\operatorname{supp}(\mathcal D_X^S)\times\operatorname{supp}(\mathcal D_X^T)$, so $\gamma^*$
never places mass outside the support product region. Furthermore, since
$\gamma^*\in\Gamma(\mathcal D_X^S,\mathcal D_X^T)$, hence
\[
\gamma^*(E_S\times\mathcal X)=\mathcal D_X^S(E_S)=0,
\qquad
\gamma^*(\mathcal X\times E_T)=\mathcal D_X^T(E_T)=0,
\]
and therefore $\gamma^*(E)=0$, and thus
\[
S_{Cpt}^{\gamma^*}-\widetilde S_{Cpt}^{\gamma^*}=0.
\]
Therefore $S_{Cpt}^{\gamma^*}$ does not depend on the choice of versions of the
conditional distributions. This proves that $S_{Cpt}^{\gamma^*}$ is unique, regardless of whether $\operatorname{supp}(\mathcal D_X^S)\neq\operatorname{supp}(\mathcal D_X^T)$.
\end{proof}

\subsection{Proof of Proposition~\ref{pro:relationship-to-existing-yx-shift}}
\begin{proof}
Since $\beta=0$ and $\mathcal D_X^{S}=\mathcal D_X^{T}$,
Lemma~\ref{lem:collapse-of-gamma} gives that the optimal coupling collapses to the diagonal:
\[
\gamma^{*}=(\mathrm{Id},\mathrm{Id})_{\#}D_{X}^{S} .
\]
Under deterministic labeling and $\rho_{\mathcal Y}=|\cdot|$, Corollary~\ref{cor:consistency-under-deterministic-labeling}
yields
\[
S_{Cpt}^{\gamma^{*}}
=\mathbb E_{(x_S,x_T)\sim\gamma^{*}}\bigl[|f_S(x_S)-f_T(x_T)|\bigr].
\]
Substituting $\gamma^{*}=(\mathrm{Id},\mathrm{Id})_{\#}D_{X}^{S}$ implies $(x_S,x_T)=(x,x)$ with $x\sim D_{X}^{S}$, hence
\[
S_{Cpt}^{\gamma^{*}}
=\mathbb E_{x\sim D_{X}^{S}}\bigl[|f_S(x)-f_T(x)|\bigr].
\]
Finally, since $\mathcal D_X^{S}=\mathcal D_X^{T}$, we obtain
\[
S_{Cpt}^{\gamma^{*}}
=\mathbb E_{x\sim\mathcal D_X^{S}}\bigl[|f_S(x)-f_T(x)|\bigr]
=\mathbb E_{x\sim\mathcal D_X^{T}}\bigl[|f_S(x)-f_T(x)|\bigr],
\]
as claimed.
\end{proof}

\subsection{Proof of Corollary~\ref{cor:composition-preserves-separate-lipschitzness}}
\begin{proof}
Take any $y_1,y_2\in\mathcal Y$ and $x_1,x_2\in\mathcal X$. By the separate
$(L_\ell,L'_\ell)$-Lipschitzness of $\ell$,
\[
\bigl|\ell(y_1,h(x_1))-\ell(y_2,h(x_2))\bigr|
\le L_\ell\,\rho_{\mathcal Y}(y_1,y_2)
+L'_\ell\,\rho_{\mathcal Y'}\!\bigl(h(x_1),h(x_2)\bigr).
\]
By the $L_h$-Lipschitzness of $h$, we further have
$\rho_{\mathcal Y'}\!\bigl(h(x_1),h(x_2)\bigr)\le L_h\,\rho_{\mathcal X}(x_1,x_2)$.
Substituting this bound yields
\[
\bigl|\ell(y_1,h(x_1))-\ell(y_2,h(x_2))\bigr|
\le L_\ell\,\rho_{\mathcal Y}(y_1,y_2)
+(L_hL'_\ell)\,\rho_{\mathcal X}(x_1,x_2),
\]
which is exactly the separate $(L_\ell,\,L_hL'_\ell)$-Lipschitz condition for the
composite function $\ell\bigl(y,h(x)\bigr)$.
\end{proof}

\subsection{Proof of Lemma~\ref{lem:separately-weak-duality}}
\begin{proof}
Write the two expectations explicitly:
\[
\mathbb E_{\mathcal D_{XY}^{S}}[\mathcal G]
=\int_{\mathcal X\times\mathcal Y}\mathcal G(y_S,x_S)\,d\mathcal D_{XY}^{S}(x_S,y_S),
\qquad
\mathbb E_{\mathcal D_{XY}^{T}}[\mathcal G]
=\int_{\mathcal X\times\mathcal Y}\mathcal G(y_T,x_T)\,d\mathcal D_{XY}^{T}(x_T,y_T).
\]
Let $\gamma_{XY}\in\Gamma(\mathcal D_{XY}^{S},\mathcal D_{XY}^{T})$ be any coupling on
$(\mathcal X\times\mathcal Y)\times(\mathcal X\times\mathcal Y)$, i.e., its first marginal
is $\mathcal D_{XY}^{S}$ and second marginal is $\mathcal D_{XY}^{T}$. Hence, for any
integrable functions $\varphi,\psi$,
\[
\int \varphi(x_S,y_S)\,d\gamma_{XY}(x_S,y_S,x_T,y_T)
=\int \varphi(x_S,y_S)\,d\mathcal D_{XY}^{S}(x_S,y_S),
\]
\[
\int \psi(x_T,y_T)\,d\gamma_{XY}(x_S,y_S,x_T,y_T)
=\int \psi(x_T,y_T)\,d\mathcal D_{XY}^{T}(x_T,y_T).
\]
Applying these identities to $\varphi(x_S,y_S)=\mathcal G(y_S,x_S)$ and
$\psi(x_T,y_T)=\mathcal G(y_T,x_T)$ gives
\begin{align*}
\mathbb E_{\mathcal D_{XY}^{S}}[\mathcal G]-\mathbb E_{\mathcal D_{XY}^{T}}[\mathcal G]
&=
\int \mathcal G(y_S,x_S)\,d\mathcal D_{XY}^{S}(x_S,y_S)
-\int \mathcal G(y_T,x_T)\,d\mathcal D_{XY}^{T}(x_T,y_T) \\
&=
\int \mathcal G(y_S,x_S)\,d\gamma_{XY}(x_S,y_S,x_T,y_T)
-\int \mathcal G(y_T,x_T)\,d\gamma_{XY}(x_S,y_S,x_T,y_T) \\
&=
\int \Bigl(\mathcal G(y_S,x_S)-\mathcal G(y_T,x_T)\Bigr)\,
d\gamma_{XY}(x_S,y_S,x_T,y_T).
\end{align*}
Taking absolute values and using $|\int f\,d\mu|\le \int |f|\,d\mu$ yields
\begin{align*}
\bigl|\mathbb E_{\mathcal D_{XY}^{S}}[\mathcal G]-\mathbb E_{\mathcal D_{XY}^{T}}[\mathcal G]\bigr|
&\le
\int \bigl|\mathcal G(y_S,x_S)-\mathcal G(y_T,x_T)\bigr|\,
d\gamma_{XY}(x_S,y_S,x_T,y_T).
\end{align*}
By the separate $(L_{\mathcal Y},L_{\mathcal X})$-Lipschitz property of $\mathcal G$, we have
\[
\bigl|\mathcal G(y_S,x_S)-\mathcal G(y_T,x_T)\bigr|
\le
L_{\mathcal Y}\rho_{\mathcal Y}(y_S,y_T)+L_{\mathcal X}\rho_{\mathcal X}(x_S,x_T).
\]
Substituting and splitting the integral,
\begin{align*}
\bigl|\mathbb E_{\mathcal D_{XY}^{S}}[\mathcal G]-\mathbb E_{\mathcal D_{XY}^{T}}[\mathcal G]\bigr|
&\le
\int \Bigl(
L_{\mathcal Y}\rho_{\mathcal Y}(y_S,y_T)+L_{\mathcal X}\rho_{\mathcal X}(x_S,x_T)
\Bigr)\,
d\gamma_{XY}(x_S,y_S,x_T,y_T) \\
&=
L_{\mathcal Y}\int \rho_{\mathcal Y}(y_S,y_T)\,d\gamma_{XY}(x_S,y_S,x_T,y_T)
+L_{\mathcal X}\int \rho_{\mathcal X}(x_S,x_T)\,d\gamma_{XY}(x_S,y_S,x_T,y_T) \\
&=
L_{\mathcal Y}\,\mathbb E_{\gamma_{XY}}\!\bigl[\rho_{\mathcal Y}(y_S,y_T)\bigr]
+
L_{\mathcal X}\,\mathbb E_{\gamma_{XY}}\!\bigl[\rho_{\mathcal X}(x_S,x_T)\bigr],
\end{align*}
as claimed.
\end{proof}

\subsection{Proof of Lemma~\ref{lem:gluing-construction-for-joint-coupling}}
\begin{proof}
We show that the first marginal of $\gamma_{XY}$ equals $\mathcal D_{XY}^{S}$ and the
second marginal equals $\mathcal D_{XY}^{T}$.

\paragraph{First marginal.}
Take any measurable $A\subseteq\mathcal X$ and $B\subseteq\mathcal Y$. By the definition
of $\gamma_{XY}$,
\begin{align*}
\gamma_{XY}\bigl((A\times B)\times(\mathcal X\times\mathcal Y)\bigr)
&=
\int_{(x_S,x_T)\in\mathcal X\times\mathcal X}
\gamma^{*}_{Y|(x_S,x_T)}\bigl((B\times\mathcal Y)\bigr)\,
\mathbf 1_{\{x_S\in A\}}\,
d\gamma^{*}(x_S,x_T).
\end{align*}
For each fixed $(x_S,x_T)$, the coupling
$\gamma^{*}_{Y|(x_S,x_T)}\in\Gamma(\mathcal D_{Y|X=x_S}^{S},\mathcal D_{Y|X=x_T}^{T})$
has first marginal $\mathcal D_{Y|X=x_S}^{S}$, hence
\[
\gamma^{*}_{Y|(x_S,x_T)}(B\times\mathcal Y)=\mathcal D_{Y|X=x_S}^{S}(B).
\]
Substituting gives
\begin{align*}
\gamma_{XY}\bigl((A\times B)\times(\mathcal X\times\mathcal Y)\bigr)
&=
\int \mathbf 1_{\{x_S\in A\}}\,\mathcal D_{Y|X=x_S}^{S}(B)\,d\gamma^{*}(x_S,x_T).
\end{align*}
Finally, since the first marginal of $\gamma^{*}$ is $\mathcal D_X^{S}$, integrating out
$x_T$ yields
\[
\gamma_{XY}\bigl((A\times B)\times(\mathcal X\times\mathcal Y)\bigr)
=
\int_{x_S\in A}\mathcal D_{Y|X=x_S}^{S}(B)\,d\mathcal D_X^{S}(x_S)
=
\mathcal D_{XY}^{S}(A\times B),
\]

\paragraph{Second marginal.}
Similarly, for measurable $A\subseteq\mathcal X$ and $B\subseteq\mathcal Y$,
\begin{align*}
\gamma_{XY}\bigl((\mathcal X\times\mathcal Y)\times(A\times B)\bigr)
&=
\int \gamma^{*}_{Y|(x_S,x_T)}(\mathcal Y\times B)\,\mathbf 1_{\{x_T\in A\}}\,
d\gamma^{*}(x_S,x_T) \\
&=
\int \mathbf 1_{\{x_T\in A\}}\,\mathcal D_{Y|X=x_T}^{T}(B)\,d\gamma^{*}(x_S,x_T) \\
&=
\int_{x_T\in A}\mathcal D_{Y|X=x_T}^{T}(B)\,d\mathcal D_X^{T}(x_T)
=
\mathcal D_{XY}^{T}(A\times B).
\end{align*}

Therefore $\gamma_{XY}\in\Gamma(\mathcal D_{XY}^{S},\mathcal D_{XY}^{T})$.
\end{proof}

\subsection{Proof of Theorem~\ref{thm:learning-bound}}
\begin{proof}
Define the composite function
\[
\mathcal G(y,x):=\ell\bigl(y,h(x)\bigr),\qquad (y,x)\in\mathcal Y\times\mathcal X.
\]
By Corollary~\ref{cor:composition-preserves-separate-lipschitzness},
$\mathcal G$ is separately $(L_\ell,\;L_hL_\ell')$-Lipschitz, i.e.,
\[
\bigl|\mathcal G(y_1,x_1)-\mathcal G(y_2,x_2)\bigr|
\le L_\ell\,\rho_{\mathcal Y}(y_1,y_2)+ (L_hL_\ell')\,\rho_{\mathcal X}(x_1,x_2).
\]
Hence
\begin{align*}
\epsilon_T(h)
&=\epsilon_S(h)+\bigl(\epsilon_T(h)-\epsilon_S(h)\bigr) \\
&\le \epsilon_S(h)+\bigl|\epsilon_T(h)-\epsilon_S(h)\bigr| \\
&= \epsilon_S(h)+\bigl|\mathbb E_{\mathcal D_{XY}^{T}}[\mathcal G]-\mathbb E_{\mathcal D_{XY}^{S}}[\mathcal G]\bigr|.
\end{align*}

\paragraph{Step 1: construct a specific joint coupling.}
Let $\gamma^*\in\Gamma(\mathcal D_X^{S},\mathcal D_X^{T})$ be the optimal coupling in
Definition~\ref{def:x-shift}. For each $(x_S,x_T)$, let
$\gamma^*_{Y|(x_S,x_T)}\in\Gamma(\mathcal D_{Y|X=x_S}^{S},\mathcal D_{Y|X=x_T}^{T})$
be an optimal coupling of $S_{pair}(x_S, x_T)$ in Definition~\ref{def:total-pair-yx-shift}, define the joint distribution of couplings
\[
\gamma_{XY}(dx_S,dx_T,dy_S,dy_T)
:=\gamma^*(dx_S,dx_T)\,\gamma^*_{Y|(x_S,x_T)}(dy_S,dy_T).
\]
By Lemma~\ref{lem:gluing-construction-for-joint-coupling},
$\gamma_{XY}\in\Gamma(\mathcal D_{XY}^{S},\mathcal D_{XY}^{T})$.

\paragraph{Step 2: apply Lemma~\ref{lem:separately-weak-duality}.}
Apply Lemma~\ref{lem:separately-weak-duality} to $\mathcal G$ with the coupling
$\gamma_{XY}$, we obtain:
\begin{align*}
\bigl|\mathbb E_{\mathcal D_{XY}^{S}}[\mathcal G]-\mathbb E_{\mathcal D_{XY}^{T}}[\mathcal G]\bigr|
&\le (L_hL_\ell')\,\mathbb E_{\gamma_{XY}}\!\bigl[\rho_{\mathcal X}(x_S,x_T)\bigr]
+ L_\ell\,\mathbb E_{\gamma_{XY}}\!\bigl[\rho_{\mathcal Y}(y_S,y_T)\bigr].
\end{align*}
It remains to bound the two expectations on the right-hand side.

\paragraph{Step 3: bound the $\mathcal X$-term by $S_{Cov}$.}
By the construction of $\gamma_{XY}$,
\begin{align*}
\mathbb E_{\gamma_{XY}}\!\bigl[\rho_{\mathcal X}(x_S,x_T)\bigr]
&=\int \rho_{\mathcal X}(x_S,x_T)\,
\gamma^*(dx_S,dx_T)\,\gamma^*_{Y|(x_S,x_T)}(dy_S,dy_T) \\
&=\int \rho_{\mathcal X}(x_S,x_T)\,\gamma^*(dx_S,dx_T)
=\mathbb E_{\gamma^*}\!\bigl[\rho_{\mathcal X}(x_S,x_T)\bigr].
\end{align*}
Since $S_{Cov}=W_\beta(\mathcal D_X^{S},\mathcal D_X^{T})$ and $\gamma^*$ is optimal,
\[
S_{Cov}
=\int \rho_{\mathcal X}(x_S,x_T)\,d\gamma^*(x_S,x_T)
+\beta\,H\!\bigl(\gamma^*\mid \mathcal D_X^{S}\otimes\mathcal D_X^{T}\bigr)
\;\ge\;
\int \rho_{\mathcal X}(x_S,x_T)\,d\gamma^*(x_S,x_T),
\]
because $\beta\geq0$ and relative entropy is nonnegative. Therefore,
\[
\mathbb E_{\gamma_{XY}}\!\bigl[\rho_{\mathcal X}(x_S,x_T)\bigr]
=\mathbb E_{\gamma^*}\!\bigl[\rho_{\mathcal X}(x_S,x_T)\bigr]
\le S_{Cov}.
\]

\paragraph{Step 4: identify the $\mathcal Y$-term with $S_{Cpt}^{\gamma^*}$.}
Again by the construction of $\gamma_{XY}$,
\begin{align*}
\mathbb E_{\gamma_{XY}}\!\bigl[\rho_{\mathcal Y}(y_S,y_T)\bigr]
&=\int \rho_{\mathcal Y}(y_S,y_T)\,
\gamma^*(dx_S,dx_T)\,\gamma^*_{Y|(x_S,x_T)}(dy_S,dy_T) \\
&=\int_{\mathcal X\times\mathcal X}
\left[
\int_{\mathcal Y\times\mathcal Y}\rho_{\mathcal Y}(y_S,y_T)\,
d\gamma^*_{Y|(x_S,x_T)}(y_S,y_T)
\right]\,
d\gamma^*(x_S,x_T) \\
&=\int_{\mathcal X\times\mathcal X}
W_1\!\bigl(\mathcal D_{Y|X=x_S}^{S},\mathcal D_{Y|X=x_T}^{T}\bigr)\,
d\gamma^*(x_S,x_T) \\
&= \mathbb E_{(x_S,x_T)\sim\gamma^*}\bigl[S_{pair}(x_S,x_T)\bigr]
= S_{Cpt}^{\gamma^*},
\end{align*}

\paragraph{Step 5: conclude the bound.}
Combining Steps 2--4 yields
\[
\bigl|\mathbb E_{\mathcal D_{XY}^{S}}[\mathcal G]-\mathbb E_{\mathcal D_{XY}^{T}}[\mathcal G]\bigr|
\le (L_hL_\ell')\,S_{Cov}+L_\ell\,S_{Cpt}^{\gamma^*}.
\]
Substituting this into the earlier inequality for $\epsilon_T(h)$ gives
\[
\epsilon_T(h)
\le \epsilon_S(h)+(L_hL_\ell')\,S_{Cov}+L_\ell\,S_{Cpt}^{\gamma^*},
\]
which proves the theorem.
\end{proof}

\subsection{Proof of Theorem~\ref{thm:concentration-debiased}}
\begin{proof}
We work with $\beta=0$, hence $W_{\beta}=W_1$.  For brevity, denote
\[
\mu := \mathcal D_X^{S},\qquad \nu := \mathcal D_X^{T},\qquad c := W_1(\mu,\nu).
\]
Let $\hat\mu',\hat\mu''$ (resp.\ $\hat\nu',\hat\nu''$) be the two half-sample empirical
measures constructed in Definition~\ref{def:debiased-estimator}.  They are independent
within each domain because they are built from disjoint halves of i.i.d.\ samples.

\paragraph{A useful inequality.}
We will use the elementary fact: for any $a,b\ge 0$,
\[
\bigl|\sqrt{a}-\sqrt{b}\bigr|\le \sqrt{|a-b|}.
\]
Indeed, if $a\ge b$ then $\sqrt{a}-\sqrt{b}=\frac{a-b}{\sqrt{a}+\sqrt{b}}
\le \frac{a-b}{\sqrt{a-b}}=\sqrt{a-b}$, and the case $b\ge a$ is symmetric.

\paragraph{Step 1: rewrite the debiased estimator and introduce an auxiliary statistic.}
Define
\[
S \;:=\;
\tfrac12\,W_1(\hat\mu',\hat\nu')^{2}
+\tfrac12\,W_1(\hat\mu'',\hat\nu'')^{2}
-\tfrac12\,W_1(\hat\mu',\hat\mu'')^{2}
-\tfrac12\,W_1(\hat\nu',\hat\nu'')^{2}.
\]
Then Definition~\ref{def:debiased-estimator} (with $\beta=0$) gives
\[
W_1^{deb}\!\bigl(\widehat{\mathcal D_X^{S}},\widehat{\mathcal D_X^{T}}\bigr)=\sqrt{|S|}.
\]
Introduce the auxiliary random variable
\[
T \;:=\; \bigl|\,S-c^2\,\bigr|.
\]

\paragraph{Step 2: reduce $\{|W_1^{deb}-c|>\varepsilon\}$ to an event on $T$.}
First note the deterministic inequality
\[
\bigl||S|-c^2\bigr|\le |S-c^2|=T,
\]
which holds because $c^2\ge 0$. Hence, using the inequality in the first paragraph,
\[
|W_1^{deb}-c|
= \bigl|\sqrt{|S|}-\sqrt{c^2}\bigr|
\le \sqrt{\bigl||S|-c^2\bigr|}
\le \sqrt{T}.
\]
Therefore, for any $\varepsilon>0$,
\[
\mathbb P\bigl(|W_1^{deb}-c|>\varepsilon\bigr)
\le \mathbb P\bigl(\sqrt{T}>\varepsilon\bigr)
= \mathbb P\bigl(T>\varepsilon^2\bigr).
\]
This bound is always valid and is particularly useful when $c$ is small.

When $c\ge \varepsilon$, we can obtain a complementary reduction by
squaring the deviation event. Indeed, $|W_1^{deb}-c|>\varepsilon$ implies either
$W_1^{deb}>c+\varepsilon$ or $W_1^{deb}<c-\varepsilon$. In both cases,
\[
\bigl| (W_1^{deb})^2 - c^2 \bigr|
> (c+\varepsilon)^2-c^2
= 2c\varepsilon+\varepsilon^2
\quad\text{or}\quad
c^2-(c-\varepsilon)^2
=2c\varepsilon-\varepsilon^2,
\]
hence in particular
\[
\bigl| (W_1^{deb})^2 - c^2 \bigr|>2c\varepsilon-\varepsilon^2.
\]
Since $(W_1^{deb})^2=|S|$, we have
\[
\{|W_1^{deb}-c|>\varepsilon\}
\subseteq \bigl\{\bigl||S|-c^2\bigr|>2c\varepsilon-\varepsilon^2\bigr\}
\subseteq \{T>2c\varepsilon-\varepsilon^2\}.
\]
Combining both regimes, define the threshold
\[
t(c,\varepsilon):=
\begin{cases}
\varepsilon^2, & c<\varepsilon,\\
2c\varepsilon-\varepsilon^2, & c\ge \varepsilon,
\end{cases}
\]
so that for all $c\ge 0$,
\[
\mathbb P\bigl(|W_1^{deb}-c|>\varepsilon\bigr)\le \mathbb P\bigl(T>t(c,\varepsilon)\bigr).
\]

\paragraph{Step 3: bound $T$ by marginal empirical Wasserstein errors.}
Start from the definition of $T$ and use the triangle inequality for $|\cdot|$:
\begin{align*}
T
&=
\Bigl|
\tfrac12\,W_1(\hat\mu',\hat\nu')^{2}
+\tfrac12\,W_1(\hat\mu'',\hat\nu'')^{2}
-\tfrac12\,W_1(\hat\mu',\hat\mu'')^{2}
-\tfrac12\,W_1(\hat\nu',\hat\nu'')^{2}
- c^2
\Bigr|
\\
&\le
\tfrac12\Bigl|W_1(\hat\mu',\hat\nu')^{2}-c^2\Bigr|
+\tfrac12\Bigl|W_1(\hat\mu'',\hat\nu'')^{2}-c^2\Bigr|
+\tfrac12\,W_1(\hat\mu',\hat\mu'')^{2}
+\tfrac12\,W_1(\hat\nu',\hat\nu'')^{2}.
\end{align*}

\smallskip
\noindent\emph{Step 3.1: control the cross-domain square terms.}
Let $u:=W_1(\hat\mu',\hat\nu')$.  Then
\[
|u^2-c^2| = |u-c|\,|u+c| \le |u-c|\,(|u-c|+2c)=2c|u-c|+|u-c|^2.
\]
Multiplying by $\tfrac12$ gives
\[
\tfrac12|u^2-c^2| \le c|u-c|+\tfrac12|u-c|^2.
\]
Applying this with $u=W_1(\hat\mu',\hat\nu')$ and again with
$u=W_1(\hat\mu'',\hat\nu'')$ yields
\begin{align*}
\tfrac12\Bigl|W_1(\hat\mu',\hat\nu')^{2}-c^2\Bigr|
&\le
c\,\bigl|W_1(\hat\mu',\hat\nu')-c\bigr|
+\tfrac12\,\bigl|W_1(\hat\mu',\hat\nu')-c\bigr|^2,
\\
\tfrac12\Bigl|W_1(\hat\mu'',\hat\nu'')^{2}-c^2\Bigr|
&\le
c\,\bigl|W_1(\hat\mu'',\hat\nu'')-c\bigr|
+\tfrac12\,\bigl|W_1(\hat\mu'',\hat\nu'')-c\bigr|^2.
\end{align*}

\smallskip
\noindent\emph{Step 3.2: relate $|W_1(\hat\mu',\hat\nu')-c|$ to $W_1(\mu,\hat\mu')$ and $W_1(\nu,\hat\nu')$.}
Using the triangle inequality for $W_1$,
\[
W_1(\hat\mu',\hat\nu')
\le W_1(\hat\mu',\mu)+W_1(\mu,\nu)+W_1(\nu,\hat\nu')
= W_1(\hat\mu',\mu)+c+W_1(\nu,\hat\nu'),
\]
hence $W_1(\hat\mu',\hat\nu')-c \le W_1(\hat\mu',\mu)+W_1(\nu,\hat\nu')$.
Similarly,
\[
c=W_1(\mu,\nu)\le W_1(\mu,\hat\mu')+W_1(\hat\mu',\hat\nu')+W_1(\hat\nu',\nu),
\]
so $c-W_1(\hat\mu',\hat\nu')\le W_1(\mu,\hat\mu')+W_1(\nu,\hat\nu')$.
Combining the two inequalities,
\[
\bigl|W_1(\hat\mu',\hat\nu')-c\bigr|
\le W_1(\mu,\hat\mu')+W_1(\nu,\hat\nu').
\]
The same argument gives
\[
\bigl|W_1(\hat\mu'',\hat\nu'')-c\bigr|
\le W_1(\mu,\hat\mu'')+W_1(\nu,\hat\nu'').
\]

\smallskip
\noindent\emph{Step 3.3: relate $W_1(\hat\mu',\hat\mu'')$ and $W_1(\hat\nu',\hat\nu'')$ to the same marginal errors.}
Again by the triangle inequality,
\[
W_1(\hat\mu',\hat\mu'')\le W_1(\hat\mu',\mu)+W_1(\mu,\hat\mu''),
\qquad
W_1(\hat\nu',\hat\nu'')\le W_1(\hat\nu',\nu)+W_1(\nu,\hat\nu'').
\]

\smallskip
\noindent\emph{Step 3.4: assemble the bound.}
Introduce the shorthand nonnegative random variables
\[
u':=W_1(\mu,\hat\mu'),\quad u'':=W_1(\mu,\hat\mu''),\quad
v':=W_1(\nu,\hat\nu'),\quad v'':=W_1(\nu,\hat\nu'').
\]
Then Steps 3.1--3.3 imply
\begin{align*}
T
&\le
c(u'+v')+\tfrac12(u'+v')^2
+c(u''+v'')+\tfrac12(u''+v'')^2
+\tfrac12(u'+u'')^2+\tfrac12(v'+v'')^2.
\end{align*}
Now use $(a+b)^2\le 2a^2+2b^2$ repeatedly to simplify:
\begin{align*}
\tfrac12(u'+v')^2 \le u'^2+v'^2,\qquad
\tfrac12(u''+v'')^2 \le u''^2+v''^2,\qquad
\tfrac12(u'+u'')^2 \le u'^2+u''^2,\qquad
\tfrac12(v'+v'')^2 \le v'^2+v''^2.
\end{align*}
Substituting yields the clean bound
\[
T \le c(u'+u''+v'+v'') + 2\bigl(u'^2+u''^2+v'^2+v''^2\bigr).
\]

In particular, on the event $\{u'\le \varepsilon_0,\ u''\le \varepsilon_0,\ v'\le \varepsilon_0,\ v''\le \varepsilon_0\}$,
we have
\[
T \le 4c\varepsilon_0 + 8\varepsilon_0^2.
\]
Hence,
\[
\{T>4c\varepsilon_0+8\varepsilon_0^2\}
\subseteq
\{u'>\varepsilon_0\}\cup\{u''>\varepsilon_0\}\cup\{v'>\varepsilon_0\}\cup\{v''>\varepsilon_0\},
\]
and by the union bound,
\[
\mathbb P\bigl(T>4c\varepsilon_0+8\varepsilon_0^2\bigr)
\le
\mathbb P(u'>\varepsilon_0)+\mathbb P(u''>\varepsilon_0)+\mathbb P(v'>\varepsilon_0)+\mathbb P(v''>\varepsilon_0).
\]

\paragraph{Step 4: apply traditional concentration inequality.}
By assumption, $\mu$ and $\nu$ have finite squared-exponential moments on $\mathbb R^d$.
Lemma~\ref{lem:sqexp-implies-T1} and ~\ref{lem:bolley} implies that there exist constants
$\lambda_S,\lambda_T>0$, depending only on the squared-exponential moments of $\mu$ and
$\nu$ respectively, such that for any $\varepsilon_0>0$ there exists $N$ with the following property:
whenever $N_S/2\ge N$ and $N_T/2\ge N$,
\[
\mathbb P\bigl(W_1(\mu,\hat\mu')>\varepsilon_0\bigr)\le \exp\!\Bigl(-\tfrac{\lambda_S N_S\varepsilon_0^2}{4}\Bigr),
\qquad
\mathbb P\bigl(W_1(\mu,\hat\mu'')>\varepsilon_0\bigr)\le \exp\!\Bigl(-\tfrac{\lambda_S N_S\varepsilon_0^2}{4}\Bigr),
\]
and similarly
\[
\mathbb P\bigl(W_1(\nu,\hat\nu')>\varepsilon_0\bigr)\le \exp\!\Bigl(-\tfrac{\lambda_T N_T\varepsilon_0^2}{4}\Bigr),
\qquad
\mathbb P\bigl(W_1(\nu,\hat\nu'')>\varepsilon_0\bigr)\le \exp\!\Bigl(-\tfrac{\lambda_T N_T\varepsilon_0^2}{4}\Bigr).
\]
Therefore, for $N_S,N_T$ large enough,
\[
\mathbb P\bigl(T>4c\varepsilon_0+8\varepsilon_0^2\bigr)
\le
2\exp\!\Bigl(-\tfrac{\lambda_S N_S\varepsilon_0^2}{4}\Bigr)
+
2\exp\!\Bigl(-\tfrac{\lambda_T N_T\varepsilon_0^2}{4}\Bigr).
\]

\paragraph{Step 5: choose $\varepsilon_0$ to match the deviation level $\varepsilon$.}
Recall from Step 2 that
\[
\mathbb P\bigl(|W_1^{deb}-c|>\varepsilon\bigr)\le \mathbb P\bigl(T>t(c,\varepsilon)\bigr),
\qquad
t(c,\varepsilon)=
\begin{cases}
\varepsilon^2, & c<\varepsilon,\\
2c\varepsilon-\varepsilon^2, & c\ge \varepsilon.
\end{cases}
\]
We now select $\varepsilon_0$ so that
\[
4c\varepsilon_0+8\varepsilon_0^2=t(c,\varepsilon).
\]
Let $a:=c/\varepsilon\ge 0$. Solving this quadratic in the two regimes yields the
closed-form expression
\[
\varepsilon_0^2 = \frac{V(a)\,\varepsilon^2}{8},
\]
where
\[
V(a)=
\begin{cases}
a^2+1-a\sqrt{a^2+2}, & a<1,\\
a^2+2a-1-a\sqrt{a^2+4a-2}, & a\ge 1.
\end{cases}
\]
Define $V_\varepsilon := V\!\bigl(c/\varepsilon\bigr)=V\!\bigl(W_1(\mu,\nu)/\varepsilon\bigr)$.
As shown in the ``Range of $V(a)$'' derivation, $V_\varepsilon\in[2-\sqrt3,2)$ and depends
only on $c/\varepsilon$.

With this choice, $\{T>t(c,\varepsilon)\}=\{T>4c\varepsilon_0+8\varepsilon_0^2\}$, hence
\begin{align*}
\mathbb P\bigl(|W_1^{deb}-c|>\varepsilon\bigr)
&\le
\mathbb P\bigl(T>4c\varepsilon_0+8\varepsilon_0^2\bigr)
\\
&\le
2\exp\!\Bigl(-\tfrac{\lambda_S N_S\varepsilon_0^2}{4}\Bigr)
+
2\exp\!\Bigl(-\tfrac{\lambda_T N_T\varepsilon_0^2}{4}\Bigr)
\\
&=
2\exp\!\Bigl(-\tfrac{\lambda_S N_S V_\varepsilon \varepsilon^2}{32}\Bigr)
+
2\exp\!\Bigl(-\tfrac{\lambda_T N_T V_\varepsilon \varepsilon^2}{32}\Bigr).
\end{align*}

Finally, note $c=W_1(\mu,\nu)=W_{\beta}(\mathcal D_X^S,\mathcal D_X^T)$ when $\beta=0$, and
$W_1^{deb}=W_{\beta}^{deb}$ when $\beta=0$, so the above inequality is exactly
\eqref{eq:deb-bd}. This completes the proof.

\paragraph{Range of $V(a)$.}
Recall
\[
V(a)=
\begin{cases}
V_1(a):=a^2+1-a\sqrt{a^2+2}, & 0\le a<1,\\[3pt]
V_2(a):=a^2+2a-1-a\sqrt{a^2+4a-2}, & a\ge 1.
\end{cases}
\]
First, the two branches agree at $a=1$:
\[
V_1(1)=2-\sqrt3,\qquad
V_2(1)=2-\sqrt3.
\]

\smallskip
\noindent\textbf{(i) Monotonicity on $[0,1]$.}
Differentiate $V_1$ on $(0,1)$:
\[
V_1'(a)=2a-\sqrt{a^2+2}-\frac{a^2}{\sqrt{a^2+2}}
=2a-\frac{2(a^2+1)}{\sqrt{a^2+2}}.
\]
For $a\ge 0$, we have $a\sqrt{a^2+2}\le a^2+1$ since
\[
(a^2+1)^2-a^2(a^2+2)=1>0.
\]
Thus $\frac{a^2+1}{\sqrt{a^2+2}}\ge a$, which implies $V_1'(a)\le 0$ on $(0,1)$.
Hence $V_1$ is decreasing on $[0,1]$, so
\[
V_1(a)\in\bigl[V_1(1),\,V_1(0)\bigr]=[\,2-\sqrt3,\;1\,].
\]

\smallskip
\noindent\textbf{(ii) Monotonicity on $[1,\infty)$.}
Let $s(a):=\sqrt{a^2+4a-2}$. Then for $a>1$,
\[
V_2'(a)=2a+2-s(a)-\frac{a(a+2)}{s(a)}.
\]
Since $s(a)>0$, the inequality $V_2'(a)\ge 0$ is equivalent to
\[
(a+1)s(a)\ \ge\ a^2+3a-1.
\]
Squaring both sides (both sides are nonnegative for $a\ge 1$) gives
\[
(a+1)^2(a^2+4a-2)\ \ge\ (a^2+3a-1)^2,
\]
and the difference factors as
\[
(a+1)^2(a^2+4a-2)-(a^2+3a-1)^2 = 3(2a-1)\ \ge\ 0 \qquad (a\ge 1).
\]
Therefore $V_2'(a)\ge 0$ for all $a\ge 1$, i.e., $V_2$ is increasing on $[1,\infty)$.
In particular,
\[
V_2(a)\ge V_2(1)=2-\sqrt3 \qquad (a\ge 1).
\]

\smallskip
\noindent\textbf{(iii) Upper bound $V(a)<2$ and $\lim_{a\to\infty}V(a)=2$.}
For $a\ge 1$, write $s(a)=\sqrt{(a+2)^2-6}$. Then
\[
(a+2-s(a))(a+2+s(a))=(a+2)^2-s(a)^2=6,
\quad\text{so}\quad
a+2-s(a)=\frac{6}{a+2+s(a)}.
\]
Using $V_2(a)=a(a+2)-1-a\,s(a)$, we get
\[
2-V_2(a)=3-a\,(a+2-s(a))
=3-\frac{6a}{a+2+s(a)}.
\]
Since $s(a)>a-2$ for $a\ge 1$ (indeed $s(a)=\sqrt{a^2+4a-2}>a$), we have
$a+2+s(a)>2a$, hence $\frac{6a}{a+2+s(a)}<3$, which implies $2-V_2(a)>0$ and thus
$V_2(a)<2$ for every finite $a$.
Moreover, as $a\to\infty$, one has $a+2+s(a)\sim 2a$, hence
$\frac{6a}{a+2+s(a)}\to 3$ and therefore $V_2(a)\to 2$.

\smallskip
\noindent\textbf{Conclusion.}
Combining (i)--(iii), the global minimum is attained at $a=1$ with
\[
\min_{a\ge 0}V(a)=2-\sqrt3,
\]
and the supremum equals $2$ but is not attained:
\[
V(a)\in[\,2-\sqrt3,\;2\,),\qquad a\ge 0.
\]

\end{proof}

\subsection{Proof of Lemma~\ref{lem:stability-of-entropic-optimal-transport-coupling}}
\begin{proof}
We follow \cite{eckstein2022quantitative}. 
Recall that the entropically regularized OT problem is
\[
S^\varepsilon_{\mathrm{ent}}(\mu,\nu,c)
:=\inf_{\pi\in\Pi(\mu,\nu)} \int c\,d\pi+\varepsilon\,\mathrm{KL}(\pi\,\|\,\mu\otimes\nu),
\]
and for fixed $\varepsilon>0$ one may assume $\varepsilon=1$ by dividing the
objective by $\varepsilon$ and using the cost $c/\varepsilon$.

\smallskip
\noindent\textbf{Step 1 (Reduction to $\varepsilon=1$ and identification of $L=1/\beta$).}
Let $c(x_1,x_2):=\rho_{\mathcal X}(x_1,x_2)$.
The $\beta$-entropic OT objective
\[
\int c\,d\pi+\beta\,\mathrm{KL}(\pi\,\|\,\mu\otimes\nu)
\]
has the \emph{same optimizer} as
\[
\int \frac{1}{\beta}c\,d\pi+\mathrm{KL}(\pi\,\|\,\mu\otimes\nu),
\]
since the two objectives differ only by a multiplicative factor $\beta$.
Hence the optimal coupling for $W_\beta(\mu,\nu)$ equals the optimizer of
$S^1_{\mathrm{ent}}(\mu,\nu,c_\beta)$ with $c_\beta:=c/\beta$.

Consider the product space $\mathcal X\times\mathcal X$ with metric
\[
\rho\big((x_1,x_2),(x_1',x_2')\big):=\rho_{\mathcal X}(x_1,x_1')+\rho_{\mathcal X}(x_2,x_2').
\]
By the triangle inequality,
\[
\big|c(x_1,x_2)-c(x_1',x_2')\big|
=\big|\rho_{\mathcal X}(x_1,x_2)-\rho_{\mathcal X}(x_1',x_2')\big|
\le \rho_{\mathcal X}(x_1,x_1')+\rho_{\mathcal X}(x_2,x_2')
=\rho((x_1,x_2),(x_1',x_2')),
\]
so $c$ is $1$-Lipschitz w.r.t.\ $\rho$, and therefore $c_\beta=c/\beta$ satisfies
the Lipschitz condition (AL) of \citet{eckstein2022quantitative} with constant
\[
L=\mathrm{Lip}(c_\beta)=\frac{1}{\beta}.
\]

\smallskip
\noindent\textbf{Step 2 (Apply Theorem 3.11).}
Let $\mu:=\mathcal D_X^S$, $\nu:=\mathcal D_X^T$, and similarly
$\hat\mu:=\widehat{\mathcal D_X^S}$, $\hat\nu:=\widehat{\mathcal D_X^T}$.
Let $\gamma^*$ and $\hat\gamma^*$ be the optimizers of
$S^1_{\mathrm{ent}}(\mu,\nu,c_\beta)$ and $S^1_{\mathrm{ent}}(\hat\mu,\hat\nu,c_\beta)$,
respectively.

Assume $\mu$ and $\nu$ have finite squared-exponential moments. Then by
Lemma~3.10(ii) in \citet{eckstein2022quantitative}, the pair of marginals
satisfies the inequality (I$_1$) with some finite constant $C_1>0
$ depending only on these squared-exponential moments.

Now apply Theorem~3.11 of \citet{eckstein2022quantitative} with $N=2$ and $p=q=1$.
Since $N^{1/q-1/p}=1$, we obtain
\[
W_1(\gamma^*,\hat\gamma^*)
\le \Delta + C_1\,(2L\Delta)^{1/2},
\qquad
\Delta:=W_1\big((\mu,\nu);(\hat\mu,\hat\nu)\big),
\].

\smallskip
\noindent\textbf{Step 3 (Upper bound $\Delta$ by $\varLambda$).}
Let $\pi_S$ be an optimal coupling between $\mu$ and $\hat\mu$, and $\pi_T$ an
optimal coupling between $\nu$ and $\hat\nu$. Then $\pi_S\otimes\pi_T$ is a coupling
between $(\mu,\nu)$ and $(\hat\mu,\hat\nu)$ on $\mathcal X\times\mathcal X$, and its
expected $\rho$-cost equals
\[
\int \rho_{\mathcal X}(x,x')\,d\pi_S(x,x')+\int \rho_{\mathcal X}(y,y')\,d\pi_T(y,y')
= W_1(\mu,\hat\mu)+W_1(\nu,\hat\nu)=:\varLambda.
\]
Taking the infimum over all couplings yields $\Delta\le \varLambda$. Therefore,
\[
W_1(\gamma^*,\hat\gamma^*)
\le \varLambda + C_1\sqrt{2L\varLambda}
= \varLambda + C_1\sqrt{\frac{2}{\beta}\,\varLambda}.
\]

\smallskip
\noindent\textbf{Step 4 (Definition of $\lambda_{\gamma^*}$ and the factor $2\sqrt{2}$).}
Define
\[
\lambda_{\gamma^*}:=\frac{4}{C_1^2}\quad\Longleftrightarrow\quad C_1=\frac{2}{\sqrt{\lambda_{\gamma^*}}}.
\]
Then
\[
C_1\sqrt{\frac{2}{\beta}\,\varLambda}
=\frac{2}{\sqrt{\lambda_{\gamma^*}}}\sqrt{\frac{2}{\beta}}\sqrt{\varLambda}
=2\sqrt{\frac{2}{\beta\lambda_{\gamma^*}}}\,\sqrt{\varLambda},
\]
which yields the claimed bound. Finally, $\lambda_{\gamma^*}>0$ depends only on the
squared-exponential moments of $\mu$ and $\nu$ because $C_1$ does
via Lemma~3.10(ii) \citep{eckstein2022quantitative}.
\end{proof}

\subsection{Proof of Theorem~\ref{thm:concentration-total-pair}}
\begin{proof} Let
\[
\mu:=\mathcal D_X^S,\qquad \nu:=\mathcal D_X^T,\qquad
\hat\mu:=\widehat{\mathcal D_X^S}=\frac1{N_S}\sum_{i=1}^{N_S}\delta_{X_i^{(S)}},\qquad
\hat\nu:=\widehat{\mathcal D_X^T}=\frac1{N_T}\sum_{j=1}^{N_T}\delta_{X_j^{(T)}}.
\]
Let $\gamma^*$ be the population entropic OT coupling between $\mu$ and $\nu$ (with $\beta>0$),
and let $\hat\gamma^*=(\hat\gamma^*_{ij})\in\mathbb R_+^{N_S\times N_T}$ be the discrete optimal
coupling solving $W_\beta(\hat\mu,\hat\nu)$.
It satisfies the marginal constraints
\[
\sum_{j=1}^{N_T}\hat\gamma^*_{ij}=\frac1{N_S}\quad(\forall i),\qquad
\sum_{i=1}^{N_S}\hat\gamma^*_{ij}=\frac1{N_T}\quad(\forall j).
\]
Associate to $\hat\gamma^*$ a probability measure on $\mathbb R^d\times\mathbb R^d$:
\[
\bar\gamma^* \;:=\; \sum_{i=1}^{N_S}\sum_{j=1}^{N_T}\hat\gamma^*_{ij}\,
\delta_{(X_i^{(S)},\,X_j^{(T)})}
\;\in\;\Gamma(\hat\mu,\hat\nu).
\]

Throughout, Wasserstein--$1$ distances on $\mathbb R^d$ use the Euclidean norm $\|\cdot\|$,
and on $\mathbb R^d\times\mathbb R^d$ we use the metric
\[
\rho\bigl((x_S,x_T),(x_S',x_T')\bigr):=\|x_S-x_S'\|+\|x_T-x_T'\|.
\]

\paragraph{Step 1: define the population quantity we will concentrate around and the bias $\varDelta$.}
Define the function $f:\mathbb R^d\times\mathbb R^d\to\mathbb R_+$ by
\[
f(x_S,x_T)
\;:=\;
\mathbb E_{\,y_S\sim \mathcal D_{Y|X=x_S}^S,\;y_T\sim \mathcal D_{Y|X=x_T}^T}
\bigl[\ \|y_S-y_T\|\ \bigr].
\]
Since $\mathcal Y\subset\mathbb R^{d'}$ is bounded with
$M=\sup_{y,y'\in\mathcal Y}\|y-y'\|$, we have $0\le f(x_S,x_T)\le M$ for all $(x_S,x_T)$.

Recall that
\[
S_{pair}(x_S,x_T)=W_1\bigl(\mathcal D_{Y|X=x_S}^S,\ \mathcal D_{Y|X=x_T}^T\bigr),
\qquad
S_{Cpt}^{\gamma^*}=\mathbb E_{(x_S,x_T)\sim\gamma^*}\bigl[S_{pair}(x_S,x_T)\bigr].
\]
Because $f(x_S,x_T)$ is the transport cost under the \emph{independent} coupling
$\mathcal D_{Y|x_S}^S\otimes \mathcal D_{Y|x_T}^T$, while $S_{pair}(x_S,x_T)$ is the infimum
over all couplings, we have
\[
f(x_S,x_T)-S_{pair}(x_S,x_T)\ge 0.
\]
Define the (constant) bias term
\[
\varDelta \;:=\; \mathbb E_{(x_S,x_T)\sim\gamma^*}\Bigl[f(x_S,x_T)-S_{pair}(x_S,x_T)\Bigr].
\]
Then
\[
S_{Cpt}^{\gamma^*}+\varDelta
=
\mathbb E_{\gamma^*}\bigl[f(x_S,x_T)\bigr].
\]
Hence it suffices to prove concentration of $\hat S_{Cpt}$ around $\mathbb E_{\gamma^*}[f]$.

\paragraph{Step 2: decompose the error into a $Y$-noise term and a coupling-stability term.}
By Definition~\ref{def:estimator-total-pair-yx},
\[
\hat S_{Cpt}
=
\sum_{i=1}^{N_S}\sum_{j=1}^{N_T}\|Y_i^{(S)}-Y_j^{(T)}\|\ \hat\gamma^*_{ij}.
\]
Add and subtract $\mathbb E_{\bar\gamma^*}[f]$:
\[
|\hat S_{Cpt}-\mathbb E_{\gamma^*}[f]|
\le
\underbrace{\bigl|\hat S_{Cpt}-\mathbb E_{\bar\gamma^*}[f]\bigr|}_{=:A}
+
\underbrace{\bigl|\mathbb E_{\bar\gamma^*}[f]-\mathbb E_{\gamma^*}[f]\bigr|}_{=:B}.
\]
We will bound $A$ by McDiarmid's inequality and $B$ by Lipschitzness of $f$ and stability of entropic OT.

\paragraph{Step 3: concentration of $A$ via McDiarmid inequality~\ref{lem:mcdiarmid}.}
Condition on all covariates $\{X_i^{(S)}\}_{i=1}^{N_S}$ and $\{X_j^{(T)}\}_{j=1}^{N_T}$.
Under this conditioning, $\hat\gamma^*$ is fixed (it depends only on the covariates),
and the labels $\{Y_i^{(S)}\}$ are independent with $Y_i^{(S)}\sim \mathcal D_{Y|X=X_i^{(S)}}^S$,
and similarly $\{Y_j^{(T)}\}$ are independent with $Y_j^{(T)}\sim \mathcal D_{Y|X=X_j^{(T)}}^T$.

Define the function of all labels
\[
F\bigl(\{Y_i^{(S)}\}_{i=1}^{N_S},\{Y_j^{(T)}\}_{j=1}^{N_T}\bigr)
:=
\sum_{i=1}^{N_S}\sum_{j=1}^{N_T}\|Y_i^{(S)}-Y_j^{(T)}\|\ \hat\gamma^*_{ij}.
\]
Then $A=\bigl|F-\mathbb E[F\mid \{X\}]\bigr|$.

\smallskip
\noindent\emph{Bounded differences for changing one source label.}
Fix an index $i$ and replace $Y_i^{(S)}$ by another value ${Y_i^{(S)}}'$ in $\mathcal Y$,
keeping all other labels the same. Then, by the triangle inequality,
\begin{align*}
&\Bigl|F(\ldots,Y_i^{(S)},\ldots)-F(\ldots,{Y_i^{(S)}}',\ldots)\Bigr|
\\
&=
\Bigl|\sum_{j=1}^{N_T}\Bigl(\|Y_i^{(S)}-Y_j^{(T)}\|-\|{Y_i^{(S)}}'-Y_j^{(T)}\|\Bigr)\hat\gamma^*_{ij}\Bigr|
\\
&\le
\sum_{j=1}^{N_T}\Bigl|\|Y_i^{(S)}-Y_j^{(T)}\|-\|{Y_i^{(S)}}'-Y_j^{(T)}\|\Bigr|\hat\gamma^*_{ij}
\\
&\le
\sum_{j=1}^{N_T}\|Y_i^{(S)}-{Y_i^{(S)}}'\|\ \hat\gamma^*_{ij}
\;\le\;
M\sum_{j=1}^{N_T}\hat\gamma^*_{ij}
\;=\;\frac{M}{N_S}.
\end{align*}

\smallskip
\noindent\emph{Bounded differences for changing one target label.}
Similarly, replacing one $Y_j^{(T)}$ by ${Y_j^{(T)}}'$ changes $F$ by at most
\[
\frac{M}{N_T}.
\]

\smallskip
\noindent\emph{Apply McDiarmid.}
By McDiarmid nequality~\ref{lem:mcdiarmid}, for any $\varepsilon_1>0$,
\begin{align*}
\mathbb P\Bigl(\,|F-\mathbb E[F\mid \{X\}]|>\varepsilon_1\ \Bigm|\ \{X\}\Bigr)
&\le
2\exp\Bigl(
-\frac{2\varepsilon_1^2}{\sum_{i=1}^{N_S}(M/N_S)^2+\sum_{j=1}^{N_T}(M/N_T)^2}
\Bigr)
\\
&=
2\exp\Bigl(
-\frac{2\varepsilon_1^2}{M^2\bigl(\frac1{N_S}+\frac1{N_T}\bigr)}
\Bigr)
\\
&=
2\exp\Bigl(
-\frac{2N_S N_T\,\varepsilon_1^2}{(N_S+N_T)\,M^2}
\Bigr).
\end{align*}
Taking expectation over $\{X\}$ gives the unconditional bound
\[
\mathbb P(A>\varepsilon_1)
\le
2\exp\Bigl(
-\frac{2N_S N_T\,\varepsilon_1^2}{(N_S+N_T)\,M^2}
\Bigr).
\]

\smallskip
\noindent\emph{Compute the conditional mean $\mathbb E[F\mid \{X\}]$ and identify $\mathbb E_{\bar\gamma^*}[f]$.}
Since $Y_i^{(S)}$ and $Y_j^{(T)}$ are independent given $\{X\}$, we have
\[
\mathbb E\bigl[\|Y_i^{(S)}-Y_j^{(T)}\|\mid \{X\}\bigr]
=
f\bigl(X_i^{(S)},X_j^{(T)}\bigr),
\]
hence
\[
\mathbb E[F\mid \{X\}]
=
\sum_{i=1}^{N_S}\sum_{j=1}^{N_T} f\bigl(X_i^{(S)},X_j^{(T)}\bigr)\ \hat\gamma^*_{ij}
=
\mathbb E_{(x_S,x_T)\sim \bar\gamma^*}\bigl[f(x_S,x_T)\bigr]
=\mathbb E_{\bar\gamma^*}[f].
\]
Therefore $A=|F-\mathbb E[F\mid\{X\}]|$ is exactly the deviation term we bounded.

\paragraph{Step 4: bound $B$ using Lipschitzness of $f$ and stability of entropic OT.}

\subparagraph{Step 4.1: show $f$ is $(M L_{Y|X})$-Lipschitz on $\mathbb R^d\times\mathbb R^d$.}
We prove Lipschitzness in the first coordinate; the second is symmetric.
Fix $x_T$ and let $x_S,x_S'\in\mathbb R^d$. Write $\pi_{x_S}^S:=\mathcal D_{Y|X=x_S}^S$ and
$\pi_{x_T}^T:=\mathcal D_{Y|X=x_T}^T$. Then
\begin{align*}
&f(x_S,x_T)-f(x_S',x_T)
\\
&=
\int_{\mathcal Y}\int_{\mathcal Y}\|y_S-y_T\|\,
\bigl(\mathrm d\pi_{x_S}^S(y_S)-\mathrm d\pi_{x_S'}^S(y_S)\bigr)\,\mathrm d\pi_{x_T}^T(y_T).
\end{align*}
Taking absolute values and using Fubini,
\begin{align*}
\bigl|f(x_S,x_T)-f(x_S',x_T)\bigr|
&\le
\int_{\mathcal Y}\Bigl|
\int_{\mathcal Y}\|y_S-y_T\|\,
\bigl(\mathrm d\pi_{x_S}^S(y_S)-\mathrm d\pi_{x_S'}^S(y_S)\bigr)
\Bigr|\,\mathrm d\pi_{x_T}^T(y_T).
\end{align*}
For each fixed $y_T$, define
\[
g_{y_T}(y_S):=\frac{2}{M}\|y_S-y_T\|-1.
\]
Since $\|y_S-y_T\|\in[0,M]$ for $y_S,y_T\in\mathcal Y$, we have $g_{y_T}\in[-1,1]$ and thus
$\|g_{y_T}\|_\infty\le 1$. Also note that
\[
\|y_S-y_T\|=\frac{M}{2}\bigl(g_{y_T}(y_S)+1\bigr),
\]
and $\int 1\,(\mathrm d\pi_{x_S}^S-\mathrm d\pi_{x_S'}^S)=0$, hence
\begin{align*}
\Bigl|
\int_{\mathcal Y}\|y_S-y_T\|\,
\bigl(\mathrm d\pi_{x_S}^S-\mathrm d\pi_{x_S'}^S\bigr)
\Bigr|
&=
\frac{M}{2}\Bigl|
\int_{\mathcal Y}g_{y_T}(y_S)\,
\bigl(\mathrm d\pi_{x_S}^S-\mathrm d\pi_{x_S'}^S\bigr)
\Bigr|
\\
&\le
M\ d_{\mathrm{TV}}(\pi_{x_S}^S,\pi_{x_S'}^S),
\end{align*}
where we used the dual representation
\[
d_{\mathrm{TV}}(P,Q)=\frac12\sup_{\|g\|_\infty\le 1}\Bigl|\int g\,\mathrm d(P-Q)\Bigr|.
\]
Plugging this into the previous bound and using that $\pi_{x_T}^T$ has total mass $1$ gives
\[
\bigl|f(x_S,x_T)-f(x_S',x_T)\bigr|
\le
M\ d_{\mathrm{TV}}\bigl(\mathcal D_{Y|X=x_S}^S,\mathcal D_{Y|X=x_S'}^S\bigr)
\le
M L_{Y|X}\|x_S-x_S'\|.
\]
Similarly,
\[
\bigl|f(x_S,x_T)-f(x_S,x_T')\bigr|
\le
M L_{Y|X}\|x_T-x_T'\|.
\]
Combining both, for all $(x_S,x_T),(x_S',x_T')$,
\[
\bigl|f(x_S,x_T)-f(x_S',x_T')\bigr|
\le
M L_{Y|X}\bigl(\|x_S-x_S'\|+\|x_T-x_T'\|\bigr)
=
M L_{Y|X}\,\rho\bigl((x_S,x_T),(x_S',x_T')\bigr).
\]
Thus $f$ is $(M L_{Y|X})$-Lipschitz on $(\mathbb R^d\times\mathbb R^d,\rho)$.

\subparagraph{Step 4.2: convert $B$ to a Wasserstein distance on couplings.}
By the Kantorovich--Rubinstein duality for $W_1$ on $(\mathbb R^d\times\mathbb R^d,\rho)$,
for any $K$-Lipschitz function $\varphi$,
\[
\Bigl|\mathbb E_{\pi}[\varphi]-\mathbb E_{\pi'}[\varphi]\Bigr|\le K\,W_1(\pi,\pi').
\]
Apply this with $\varphi=f$ and $K=M L_{Y|X}$, $\pi=\bar\gamma^*$, $\pi'=\gamma^*$:
\[
B=\bigl|\mathbb E_{\bar\gamma^*}[f]-\mathbb E_{\gamma^*}[f]\bigr|
\le
M L_{Y|X}\,W_1(\gamma^*,\bar\gamma^*).
\]

\subparagraph{Step 4.3: apply entropic OT stability and then concentration of marginals.}
By Lemma~\ref{lem:stability-of-entropic-optimal-transport-coupling},
\[
W_1(\gamma^*,\bar\gamma^*)
\le
\varLambda +2\sqrt{\frac{2}{\beta \lambda_{\gamma^*}}}\sqrt{\varLambda},
\qquad
\varLambda:=W_1(\mu,\hat\mu)+W_1(\nu,\hat\nu).
\]
Therefore,
\[
B
\le
M L_{Y|X}\Bigl(\varLambda +2\sqrt{\frac{2}{\beta \lambda_{\gamma^*}}}\sqrt{\varLambda}\Bigr).
\]

Fix $\varepsilon_2>0$. If $\varLambda\le 2\varepsilon_2$, then
\[
\sqrt{\varLambda}\le \sqrt{2\varepsilon_2},
\qquad
\varLambda +2\sqrt{\frac{2}{\beta \lambda_{\gamma^*}}}\sqrt{\varLambda}
\le
2\varepsilon_2 + 4\sqrt{\frac{\varepsilon_2}{\beta \lambda_{\gamma^*}}},
\]
and hence
\[
B \le 2M L_{Y|X}\varepsilon_2 + 4M L_{Y|X}\sqrt{\frac{\varepsilon_2}{\beta \lambda_{\gamma^*}}}.
\]
Consequently,
\[
\Bigl\{B>2M L_{Y|X}\varepsilon_2 + 4M L_{Y|X}\sqrt{\frac{\varepsilon_2}{\beta \lambda_{\gamma^*}}}\Bigr\}
\subseteq
\{\varLambda>2\varepsilon_2\}.
\]
By the union bound,
\[
\mathbb P(\varLambda>2\varepsilon_2)
\le
\mathbb P\bigl(W_1(\mu,\hat\mu)>\varepsilon_2\bigr)
+\mathbb P\bigl(W_1(\nu,\hat\nu)>\varepsilon_2\bigr).
\]

Since $\mu,\nu$ have finite squared-exponential moments on $\mathbb R^d$, by Lemma~\ref{lem:sqexp-implies-T1} and ~\ref{lem:bolley}, there exist constants
$\lambda_S,\lambda_T>0$ depending only on these moments such that for any $\varepsilon_2>0$
there exists $N$ with, whenever $N_S,N_T>N$,
\[
\mathbb P\bigl(W_1(\mu,\hat\mu)>\varepsilon_2\bigr)\le \exp\Bigl(-\frac{\lambda_S N_S\varepsilon_2^2}{2}\Bigr),
\qquad
\mathbb P\bigl(W_1(\nu,\hat\nu)>\varepsilon_2\bigr)\le \exp\Bigl(-\frac{\lambda_T N_T\varepsilon_2^2}{2}\Bigr).
\]
Thus we conclude
\[
\mathbb P\Bigl(B>2M L_{Y|X}\varepsilon_2 + 4M L_{Y|X}\sqrt{\frac{\varepsilon_2}{\beta \lambda_{\gamma^*}}}\Bigr)
\le
\exp\Bigl(-\frac{\lambda_S N_S\varepsilon_2^2}{2}\Bigr)
+\exp\Bigl(-\frac{\lambda_T N_T\varepsilon_2^2}{2}\Bigr).
\]

\paragraph{Step 5: combine the two parts.}
For any $\varepsilon_1,\varepsilon_2>0$, by the union bound,
\begin{align*}
\mathbb P\Bigl(|\hat S_{Cpt}-\mathbb E_{\gamma^*}[f]|>\varepsilon_1
+2M L_{Y|X}\varepsilon_2 + 4M L_{Y|X}\sqrt{\tfrac{\varepsilon_2}{\beta \lambda_{\gamma^*}}}\Bigr)
&\le
\mathbb P(A>\varepsilon_1)
\\
&\quad+
\mathbb P\Bigl(B>2M L_{Y|X}\varepsilon_2 + 4M L_{Y|X}\sqrt{\tfrac{\varepsilon_2}{\beta \lambda_{\gamma^*}}}\Bigr)
\\
&\le
2\exp\Bigl(-\frac{2N_S N_T\,\varepsilon_1^2}{(N_S+N_T)\,M^2}\Bigr)
\\
&\quad+
\exp\Bigl(-\frac{\lambda_S N_S\varepsilon_2^2}{2}\Bigr)
+\exp\Bigl(-\frac{\lambda_T N_T\varepsilon_2^2}{2}\Bigr).
\end{align*}
Recalling $\mathbb E_{\gamma^*}[f]=S_{Cpt}^{\gamma^*}+\varDelta$, the left-hand side is
exactly $\mathbb P(|\hat S_{Cpt}-S_{Cpt}^{\gamma^*}-\varDelta|>\cdots)$.

\paragraph{Step 6: parameter tying and the explicit $\varPhi$.}
Introduce a single auxiliary parameter $\varepsilon_0>0$ and set
\[
\varepsilon_1:=M\varepsilon_0,
\qquad
\varepsilon_2:=\lambda_S^{-1/4}\lambda_T^{-1/4}\,\varepsilon_0.
\]
Then the three exponential terms become
\begin{align*}
2\exp\Bigl(-\frac{2N_S N_T\,\varepsilon_1^2}{(N_S+N_T)\,M^2}\Bigr)
&=
2\exp\Bigl(-\frac{2N_S N_T\,\varepsilon_0^2}{N_S+N_T}\Bigr),
\\
\exp\Bigl(-\frac{\lambda_S N_S\varepsilon_2^2}{2}\Bigr)
&=
\exp\Bigl(-\frac{\lambda_S^{1/2}N_S\,\varepsilon_0^2}{2\lambda_T^{1/2}}\Bigr),
\\
\exp\Bigl(-\frac{\lambda_T N_T\varepsilon_2^2}{2}\Bigr)
&=
\exp\Bigl(-\frac{\lambda_T^{1/2}N_T\,\varepsilon_0^2}{2\lambda_S^{1/2}}\Bigr).
\end{align*}
The deviation level on the left-hand side becomes
\begin{align*}
\varepsilon_1 +2M L_{Y|X}\varepsilon_2 + 4M L_{Y|X}\sqrt{\tfrac{\varepsilon_2}{\beta \lambda_{\gamma^*}}}
&=
M\varepsilon_0
+2M L_{Y|X}\lambda_S^{-1/4}\lambda_T^{-1/4}\varepsilon_0
+4M L_{Y|X}\beta^{-1/2}\lambda_{\gamma^*}^{-1/2}\lambda_S^{-1/8}\lambda_T^{-1/8}\varepsilon_0^{1/2}.
\end{align*}
Define the constants
\[
c_2:=1+2L_{Y|X}\lambda_S^{-1/4}\lambda_T^{-1/4},
\qquad
c_1:=16L_{Y|X}^2\beta^{-1}\lambda_{\gamma^*}^{-1}\lambda_S^{-1/4}\lambda_T^{-1/4}.
\]
Then the deviation level can be written compactly as
\[
M\Bigl(c_2\varepsilon_0+\sqrt{c_1\varepsilon_0}\Bigr).
\]
Now choose $\varepsilon_0$ so that
\[
M\Bigl(c_2\varepsilon_0+\sqrt{c_1\varepsilon_0}\Bigr)=\varepsilon.
\]
Let $u:=\sqrt{\varepsilon_0}$. The equation becomes $c_2u^2+\sqrt{c_1}\,u=\varepsilon/M$,
whose positive solution is
\[
u=\frac{\sqrt{c_1+4c_2(\varepsilon/M)}-\sqrt{c_1}}{2c_2},
\qquad
\varepsilon_0=u^2.
\]
It is convenient to express $\varepsilon_0^2$ in the form $\varepsilon_0^2=\frac{\varPhi\varepsilon^2}{2M^2}$.
Let
\[
a:=\frac{c_1M}{c_2\varepsilon}\quad(>0),
\qquad
\varLambda(a):=a^2+4a+2-(a+2)\sqrt{a^2+4a},
\qquad
\varPhi:=\frac{\varLambda(a)}{c_2^2}.
\]
A direct algebraic simplification (expanding $(a+2-\sqrt{a^2+4a})^2$) shows
\[
\varepsilon_0^2=\frac{\varPhi\,\varepsilon^2}{2M^2}.
\]

Substituting this $\varepsilon_0^2$ into the exponential terms yields
\begin{align*}
2\exp\Bigl(-\frac{2N_SN_T\varepsilon_0^2}{N_S+N_T}\Bigr)
&=
2\exp\Bigl(-\frac{N_S N_T\,\varPhi\,\varepsilon^2}{(N_S+N_T)\,M^2}\Bigr),
\\
\exp\Bigl(-\frac{\lambda_S^{1/2}N_S\varepsilon_0^2}{2\lambda_T^{1/2}}\Bigr)
&=
\exp\Bigl(-\frac{\lambda_S^{1/2}N_S\,\varPhi\,\varepsilon^2}{4\lambda_T^{1/2}M^2}\Bigr),
\\
\exp\Bigl(-\frac{\lambda_T^{1/2}N_T\varepsilon_0^2}{2\lambda_S^{1/2}}\Bigr)
&=
\exp\Bigl(-\frac{\lambda_T^{1/2}N_T\,\varPhi\,\varepsilon^2}{4\lambda_S^{1/2}M^2}\Bigr).
\end{align*}
This is exactly the claimed bound, completing the proof.
\end{proof}

\subsection{Proof of Proposition~\ref{pro:varDelta-in-deterministic-labeling}}
\begin{proof}
Recall the notation used in the proof of Theorem~\ref{thm:concentration-total-pair}:
for $(x_S,x_T)\in\mathcal X\times\mathcal X$ define
\[
f(x_S,x_T)
:=\mathbb E_{y_S\sim\mathcal D_{Y|X=x_S}^{S},\,y_T\sim\mathcal D_{Y|X=x_T}^{T}}
\bigl[\rho_{\mathcal Y}(y_S,y_T)\bigr],
\qquad
S_{pair}(x_S,x_T)
:=W_1\!\bigl(\mathcal D_{Y|X=x_S}^{S},\mathcal D_{Y|X=x_T}^{T}\bigr),
\]
and the bias is
\[
\varDelta
=\mathbb E_{(x_S,x_T)\sim\gamma^*}\!\bigl[f(x_S,x_T)-S_{pair}(x_S,x_T)\bigr].
\]

Under deterministic labeling,
$\mathcal D_{Y|X=x_S}^{S}=\delta_{f_S(x_S)}$ and
$\mathcal D_{Y|X=x_T}^{T}=\delta_{f_T(x_T)}$.
Hence the random variables $y_S,y_T$ are almost surely equal to $f_S(x_S),f_T(x_T)$, and thus
\[
f(x_S,x_T)
=\mathbb E\bigl[\rho_{\mathcal Y}(y_S,y_T)\bigr]
=\rho_{\mathcal Y}\bigl(f_S(x_S),f_T(x_T)\bigr).
\]
On the other hand, the Wasserstein-$1$ distance between two Dirac measures is exactly the ground metric:
\[
S_{pair}(x_S,x_T)
=W_1\!\bigl(\delta_{f_S(x_S)},\delta_{f_T(x_T)}\bigr)
=\rho_{\mathcal Y}\bigl(f_S(x_S),f_T(x_T)\bigr).
\]
Therefore $f(x_S,x_T)-S_{pair}(x_S,x_T)=0$ for all $(x_S,x_T)$, and so
\[
\varDelta
=\mathbb E_{\gamma^*}\!\bigl[f(x_S,x_T)-S_{pair}(x_S,x_T)\bigr]
=0.
\]
This concludes the proof.
\end{proof}

\subsection{Proof of Proposition~\ref{pro:varDelta-in-stochastic-labeling}}
\begin{proof}
We follow the notation used in the proof of Theorem~\ref{thm:concentration-total-pair}.
For $(x_S,x_T)\in\mathcal X\times\mathcal X$, define
\[
f(x_S,x_T)
:=
\mathbb E_{y_S\sim\mathcal D_{Y|X=x_S}^{S},\;y_T\sim\mathcal D_{Y|X=x_T}^{T}}
\bigl[\|y_S-y_T\|\bigr],
\qquad
S_{pair}(x_S,x_T)
:=
W_1\!\bigl(\mathcal D_{Y|X=x_S}^{S},\mathcal D_{Y|X=x_T}^{T}\bigr),
\]
and recall that
\[
\varDelta
=
\mathbb E_{(x_S,x_T)\sim\gamma^*}\bigl[f(x_S,x_T)-S_{pair}(x_S,x_T)\bigr].
\]

\paragraph{Lower bound $\varDelta\ge 0$.}
For each fixed $(x_S,x_T)$, the product measure
$\mathcal D_{Y|X=x_S}^{S}\otimes \mathcal D_{Y|X=x_T}^{T}$
belongs to
$\Gamma\!\bigl(\mathcal D_{Y|X=x_S}^{S},\mathcal D_{Y|X=x_T}^{T}\bigr)$.
Since $S_{pair}(x_S,x_T)$ is the infimum of $\mathbb E[\|y_S-y_T\|]$ over all such couplings,
\[
S_{pair}(x_S,x_T)
\le
\mathbb E_{\mathcal D_{Y|X=x_S}^{S}\otimes \mathcal D_{Y|X=x_T}^{T}}
\bigl[\|y_S-y_T\|\bigr]
=
f(x_S,x_T).
\]
Thus $f(x_S,x_T)-S_{pair}(x_S,x_T)\ge 0$ pointwise, hence $\varDelta\ge 0$.

\paragraph{An upper bound for each pair $(x_S,x_T)$.}
Fix $(x_S,x_T)$ and write
$P:=\mathcal D_{Y|X=x_S}^{S}$ and $Q:=\mathcal D_{Y|X=x_T}^{T}$.
Let their means be
\[
m_S(x_S):=\mathbb E_{y\sim P}[y],
\qquad
m_T(x_T):=\mathbb E_{y\sim Q}[y],
\]
For independent draws $y_S\sim P$ and $y_T\sim Q$, the triangle inequality yields
\[
\|y_S-y_T\|
\le
\|y_S-m_S(x_S)\|
+\|m_S(x_S)-m_T(x_T)\|
+\|m_T(x_T)-y_T\|.
\]
Taking expectation gives
\[
f(x_S,x_T)
\le
\mathbb E_{y\sim P}\bigl[\|y-m_S(x_S)\|\bigr]
+\|m_S(x_S)-m_T(x_T)\|
+\mathbb E_{y\sim Q}\bigl[\|y-m_T(x_T)\|\bigr].
\]

Next we lower bound $S_{pair}(x_S,x_T)=W_1(P,Q)$ by the distance between the means.
Let $\pi\in\Gamma(P,Q)$ be any coupling, and denote
$m_P:=\mathbb E_{y\sim P}[y]$, $m_Q:=\mathbb E_{y\sim Q}[y]$.
Since $\|\cdot\|$ is convex, Jensen's inequality yields
\[
\mathbb E_{(y,y')\sim\pi}\bigl[\|y-y'\|\bigr]
\;\ge\;
\Bigl\|\mathbb E_{(y,y')\sim\pi}[\,y-y'\,]\Bigr\|.
\]
By the marginal constraints of $\pi$, we have
$\mathbb E_{(y,y')\sim\pi}[y]=m_P$ and $\mathbb E_{(y,y')\sim\pi}[y']=m_Q$, hence
\[
\Bigl\|\mathbb E_{(y,y')\sim\pi}[\,y-y'\,]\Bigr\|
=
\|m_P-m_Q\|.
\]
Therefore, for every $\pi\in\Gamma(P,Q)$,
\[
\mathbb E_{(y,y')\sim\pi}\bigl[\|y-y'\|\bigr]\;\ge\;\|m_P-m_Q\|.
\]
Taking the infimum over $\pi\in\Gamma(P,Q)$ gives
\[
W_1(P,Q)
=
\inf_{\pi\in\Gamma(P,Q)}\mathbb E_{(y,y')\sim\pi}\bigl[\|y-y'\|\bigr]
\;\ge\;
\|m_P-m_Q\|.
\]
Combining the two displays and canceling the middle term yields
\[
f(x_S,x_T)-S_{pair}(x_S,x_T)
\le
\mathbb E_{y\sim P}\bigl[\|y-m_S(x_S)\|\bigr]
+
\mathbb E_{y\sim Q}\bigl[\|y-m_T(x_T)\|\bigr].
\]

\paragraph{Integrate over $\gamma^*$ and relate to irreducible errors.}
Taking expectation over $(x_S,x_T)\sim\gamma^*$ and using that the marginals of $\gamma^*$
are $\mathcal D_X^{S}$ and $\mathcal D_X^{T}$,
\[
\varDelta
\le
\mathbb E_{x\sim\mathcal D_X^{S}}
\Bigl[\mathbb E\bigl[\|Y^{(S)}-m_S(x)\|\mid X^{(S)}=x\bigr]\Bigr]
+
\mathbb E_{x\sim\mathcal D_X^{T}}
\Bigl[\mathbb E\bigl[\|Y^{(T)}-m_T(x)\|\mid X^{(T)}=x\bigr]\Bigr].
\]
For any random variable $Z\ge 0$, $\mathbb E[Z]\le \sqrt{\mathbb E[Z^2]}$, hence
\[
\mathbb E\bigl[\|Y^{(S)}-m_S(x)\|\mid X^{(S)}=x\bigr]
\le
\sqrt{\mathbb E\bigl[\|Y^{(S)}-m_S(x)\|^2\mid X^{(S)}=x\bigr]},
\]
and similarly for the target term. Thus
\[
\varDelta
\le
\mathbb E_{x\sim\mathcal D_X^{S}}
\Bigl[\sqrt{\mathbb E\bigl[\|Y^{(S)}-m_S(x)\|^2\mid X^{(S)}=x\bigr]}\Bigr]
+
\mathbb E_{x\sim\mathcal D_X^{T}}
\Bigl[\sqrt{\mathbb E\bigl[\|Y^{(T)}-m_T(x)\|^2\mid X^{(T)}=x\bigr]}\Bigr].
\]
Since $\sqrt{\cdot}$ is concave on $\mathbb R_+$, Jensen's inequality gives
\begin{align*}
\mathbb E_{x\sim\mathcal D_X^{S}}
\Bigl[\sqrt{\mathbb E\bigl[\|Y^{(S)}-m_S(x)\|^2\mid X^{(S)}=x\bigr]}\Bigr]
&\le
\sqrt{
\mathbb E_{(x,y)\sim\mathcal D_{XY}^{S}}
\bigl[\|y-m_S(x)\|^2\bigr]
},\\
\mathbb E_{x\sim\mathcal D_X^{T}}
\Bigl[\sqrt{\mathbb E\bigl[\|Y^{(T)}-m_T(x)\|^2\mid X^{(T)}=x\bigr]}\Bigr]
&\le
\sqrt{
\mathbb E_{(x,y)\sim\mathcal D_{XY}^{T}}
\bigl[\|y-m_T(x)\|^2\bigr]
}.
\end{align*}

Finally, under squared loss on $\mathbb R^{d'}$, the minimizer of
$\mathbb E[\|Y-g(x)\|^2\!\mid\!X=x]$ is $g(x)=\mathbb E[Y\!\mid\!X=x]$.
Therefore the functions $m_S(x)=\mathbb E[Y^{(S)}\!\mid\!X^{(S)}=x]$ and
$m_T(x)=\mathbb E[Y^{(T)}\!\mid\!X^{(T)}=x]$ achieve the infima of
$\mathrm I(\mathcal D_{XY}^{S})$ and $\mathrm I(\mathcal D_{XY}^{T})$, i.e.,
\[
\mathbb E_{(x,y)\sim\mathcal D_{XY}^{S}}\bigl[\|y-m_S(x)\|^2\bigr]
=\mathrm I(\mathcal D_{XY}^{S}),
\qquad
\mathbb E_{(x,y)\sim\mathcal D_{XY}^{T}}\bigl[\|y-m_T(x)\|^2\bigr]
=\mathrm I(\mathcal D_{XY}^{T}).
\]
Putting the above bounds together yields
\[
\varDelta
\le
\sqrt{\mathrm I(\mathcal D_{XY}^{S})}
+
\sqrt{\mathrm I(\mathcal D_{XY}^{T})},
\]
which completes the proof.
\end{proof}

\section{Analysis of Lipschitz Constants: Results and Proofs}
\label{app:lipschitz}

\subsection{Lipschitz Constant of the Sigmoid Function}
\label{app:lipschitz-sigmoid-function}
\begin{proposition}[Lipschitz constant of the sigmoid]
\label{prop:sigmoid-lipschitz}
Let $\sigma:\mathbb{R}\to\mathbb{R}$ be the sigmoid function
\[
\sigma(x)\;=\;\frac{1}{1+e^{-x}}.
\]
Then $\sigma$ is globally $\frac{1}{4}$-Lipschitz on $\mathbb{R}$.
\end{proposition}

\begin{proof}
We first compute the derivative:
\[
\sigma'(x)
= \frac{d}{dx}\Bigl(1+e^{-x}\Bigr)^{-1}
= \frac{e^{-x}}{\bigl(1+e^{-x}\bigr)^2}.
\]
Using $\sigma(x)=\frac{1}{1+e^{-x}}$ and $1-\sigma(x)=\frac{e^{-x}}{1+e^{-x}}$, we can rewrite
\[
\sigma'(x)=\sigma(x)\bigl(1-\sigma(x)\bigr).
\]
Let $u=\sigma(x)\in(0,1)$. Then $\sigma'(x)=u(1-u)$, and for all $u\in[0,1]$,
\[
u(1-u)=u-u^2 \;\le\; \max_{t\in[0,1]}(t-t^2)=\frac{1}{4},
\]
where the maximum is attained at $t=\frac{1}{2}$. Hence $0<\sigma'(x)\le \frac{1}{4}$ for all $x$.

Now fix any $x,y\in\mathbb{R}$. Since $\sigma$ is differentiable on $\mathbb{R}$ (hence continuous), by the mean value theorem, there exists $c$ between $x$ and $y$ such that
\[
\sigma(x)-\sigma(y)=\sigma'(c)\,(x-y).
\]
Taking absolute values and using $\sigma'(c)\le \frac{1}{4}$ yields
\[
|\sigma(x)-\sigma(y)|
=|\sigma'(c)|\,|x-y|
\le \frac{1}{4}|x-y|,
\]
so $\sigma$ is $\frac{1}{4}$-Lipschitz.
\end{proof}

\subsection{Lipschitz Constant of the Logistic Regression}
\label{app:lipschitz-logistic-regression}
\begin{proposition}[Lipschitz constant of logistic regression under $\ell_p$/$\ell_q$]
\label{prop:logreg-lipschitz-pq}
Fix $w\in\mathbb{R}^d$ and $b\in\mathbb{R}$. Define
\[
h_{w,b}(x)\;=\;\sigma\!\bigl(w^\top x+b\bigr),
\qquad
\sigma(t)=\frac{1}{1+e^{-t}}.
\]
Let $p\in[1,\infty]$ and let $q\in[1,\infty]$ be its H\"older conjugate, i.e., $\frac{1}{p}+\frac{1}{q}=1$.
Then $h_{w,b}$ is globally $\frac{\|w\|_q}{4}$-Lipschitz with respect to $\|\cdot\|_p$, namely for all $x,x'\in\mathbb{R}^d$,
\[
\bigl|h_{w,b}(x)-h_{w,b}(x')\bigr|
\;\le\;
\frac{\|w\|_q}{4}\,\|x-x'\|_p.
\]
In particular, under $\|\cdot\|_2$, the Lipschitz constant equals $\|w\|_2/4$.
\end{proposition}

\begin{proof}
For any $x,x'\in\mathbb{R}^d$, apply the mean value theorem to $\sigma$ to obtain some $\xi$ between
$w^\top x+b$ and $w^\top x'+b$ such that
\[
h_{w,b}(x)-h_{w,b}(x')
=\sigma'(\xi)\,w^\top(x-x').
\]
Thus
\[
\bigl|h_{w,b}(x)-h_{w,b}(x')\bigr|
\le
|\sigma'(\xi)|\,\bigl|w^\top(x-x')\bigr|.
\]
By Proposition~\ref{prop:sigmoid-lipschitz}, we have $|\sigma'(\xi)|\le \frac14$.
By H\"older's inequality with conjugate exponents $(p,q)$,
\[
\bigl|w^\top(x-x')\bigr|
\le
\|w\|_q\,\|x-x'\|_p.
\]
Combining the two inequalities yields
\[
\bigl|h_{w,b}(x)-h_{w,b}(x')\bigr|
\le
\frac14\,\|w\|_q\,\|x-x'\|_p,
\]
which proves the claim.
\end{proof}

\subsection{Lipschitz Constant of the Linear Multi-class Classifier}
\label{app:lipschitz-linear-classifier}
\begin{subproposition}[Lipschitz constant of the linear classifier under $\|\cdot\|_{2}$]
\label{prop:softmax-linear-l2}
Let $W\in\mathbb{R}^{N\times d}$ and $b\in\mathbb{R}^N$. Define
\[
f(x)=\mathrm{softmax}(Wx+b)\in\Delta^{N-1},\qquad x\in\mathbb{R}^d.
\]
Here $\Delta^{N-1}:=\{p\in\mathbb{R}^N:\ p_i\ge 0,\ \forall i,\ \sum_{i=1}^N p_i=1\}$ denotes the probability simplex. Let $w_i^\top$ denote the $i$-th row of $W$ and define
\[
D(W):=\max_{i,j\in[N]}\|w_i-w_j\|_2.
\]
Then $f$ is Lipschitz from $(\mathbb{R}^d,\|\cdot\|_2)$ to $(\mathbb{R}^N,\|\cdot\|_2)$ and its optimal Lipschitz constant equals
\[
\sup_{x\in\mathbb{R}^d}\|\nabla f(x)\|_{2}
=\frac{1}{2\sqrt{2}}\,D(W),
\]
where $\|\cdot\|_{2}$ denotes the matrix operator norm induced by the vector $\|\cdot\|_{2}$ norm.
\end{subproposition}

\begin{proof}
For $z\in\mathbb{R}^N$, write $p=\mathrm{softmax}(z)\in\Delta^{N-1}$. The Jacobian of softmax is
\[
J(p)=\nabla_z\,\mathrm{softmax}(z)=\mathrm{Diag}(p)-pp^\top,
\]
where $\mathrm{Diag}(p)$ is the diagonal matrix with diagonal entries $p_1,\dots,p_N$.
By the chain rule, for $p=\mathrm{softmax}(Wx+b)$,
\[
\nabla f(x)=J(p)\,W.
\]
Therefore,
\begin{equation}
\label{eq:lip-sup-p}
\sup_{x\in\mathbb{R}^d}\|\nabla f(x)\|_2
=\sup_{p\in\Delta^{N-1}}\|J(p)W\|_2.
\end{equation}

Fix $p\in\Delta^{N-1}$ and a unit vector $u\in\mathbb{R}^d$ with $\|u\|_2=1$. Let
\[
v:=Wu\in\mathbb{R}^N,\qquad \mu:=p^\top v.
\]
and $\mu$ represents the mean of $v$ under the probability distribution $p$.
We have
\[
J(p)v=\bigl(\mathrm{Diag}(p)-pp^\top\bigr)v
=\mathrm{Diag}(p)\bigl(v-\mu\mathbf{1}\bigr),
\]
hence
\begin{equation}
\label{eq:norm-square}
\|J(p)Wu\|_2^2=\|J(p)v\|_2^2=\sum_{i=1}^N p_i^2\,(v_i-\mu)^2.
\end{equation}

Now fix $v\in\mathbb{R}^N$ and consider the quantity
\[
\Phi(v):=\sup_{p\in\Delta^{N-1}} \sum_{i=1}^N p_i^2\,(v_i-p^\top v)^2.
\]
Indeed, for any $p\in\Delta^{N-1}$,
\[
\sum_{i=1}^N p_i^2 (v_i-\mu)^2
\le \Big(\max_{i}p_i\Big)\sum_{i=1}^N p_i (v_i-\mu)^2
=\Big(\max_{i}p_i\Big)\mathrm{Var}_p(v),
\]
where $\mathrm{Var}_p(v):=\sum_i p_i (v_i-\mu)^2$, which represents the variance of $v$ under the probability distribution $p$. If $p$ is supported on two values $m:=\min_i v_i$ and $M:=\max_i v_i$ with masses $\alpha$ and $1-\alpha$, let $r:=M-m$, then
\[
\sum_{i=1}^N p_i^2 (v_i-\mu)^2 = 2\alpha^2(1-\alpha)^2 r^2 \le \frac{r^2}{8},
\]
with equality at $\alpha=\tfrac12$. Since any mass placed at intermediate values of $v$ cannot increase the range-based extremum, we obtain $\Phi(v)=r^2/8$, and the optimal $p$ is supported on indices attaining $M$ and $m$.

Combining \eqref{eq:norm-square}, for any unit $u$,
\begin{equation}
\label{eq:sup-p-fixed-u}
\sup_{p\in\Delta^{N-1}}\|J(p)Wu\|_2
=\frac{1}{2\sqrt{2}}\Big(\max_i (Wu)_i-\min_j (Wu)_j\Big).
\end{equation}

Finally, maximize the range over $u$. Writing $(Wu)_i=w_i^\top u$, we have
\[
\max_i (Wu)_i-\min_j (Wu)_j
=\max_{i,j}(w_i-w_j)^\top u.
\]
Thus,
\[
\sup_{\|u\|_2=1}\Big(\max_i (Wu)_i-\min_j (Wu)_j\Big)
=\max_{i,j}\sup_{\|u\|_2=1}(w_i-w_j)^\top u
=\max_{i,j}\|w_i-w_j\|_2
=D(W).
\]
Plugging this into \eqref{eq:sup-p-fixed-u} and then into \eqref{eq:lip-sup-p} yields
\begin{align*}
&\sup_{x\in\mathbb{R}^d}\|\nabla f(x)\|_2 \\
=&\sup_{p\in\Delta^{N-1}}\|J(p)W\|_2 \\
=&\sup_{p\in\Delta^{N-1}}\sup_{\|u\|_2=1}\|J(p)Wu\|_2 \\
=&\sup_{\|u\|_2=1}\sup_{p\in\Delta^{N-1}}\|J(p)Wu\|_2 \\
=&\frac{1}{2\sqrt{2}}\,D(W)
\end{align*}
This completes the proof.
\end{proof}

\begin{subproposition}[Lipschitz constant of the linear classifier under $\|\cdot\|_{2}\to\|\cdot\|_{1}$]
\label{prop:softmax-linear-l21}
Let $W\in\mathbb{R}^{N\times d}$ and $b\in\mathbb{R}^N$. Define
\[
f(x)=\mathrm{softmax}(Wx+b)\in\Delta^{N-1},\qquad x\in\mathbb{R}^d.
\]
Here $\Delta^{N-1}:=\{p\in\mathbb{R}^N:\ p_i\ge 0,\ \forall i,\ \sum_{i=1}^N p_i=1\}$ denotes the probability simplex. Let $w_i^\top$ denote the $i$-th row of $W$ and define
\[
D(W):=\max_{i,j\in[N]}\|w_i-w_j\|_2.
\]
Then $f$ is Lipschitz from $(\mathbb{R}^d,\|\cdot\|_2)$ to $(\mathbb{R}^N,\|\cdot\|_1)$, and its optimal Lipschitz constant equals
\[
\sup_{x\in\mathbb{R}^d}\|\nabla f(x)\|_{2\to 1}
=\frac{1}{2}\,D(W),
\]
where $\|\cdot\|_{2\to 1}$ denotes the operator norm induced by $\|\cdot\|_{2}$ on the input and $\|\cdot\|_{1}$ on the output.
\end{subproposition}

\stepcounter{theorem}

\begin{proof}
For $z\in\mathbb{R}^N$, write $p=\mathrm{softmax}(z)\in\Delta^{N-1}$ and recall
\[
J(p)=\nabla_z\,\mathrm{softmax}(z)=\mathrm{Diag}(p)-pp^\top.
\]
By the chain rule, for $p=\mathrm{softmax}(Wx+b)$ we have
\[
\nabla f(x)=J(p)\,W.
\]
Hence
\begin{equation}
\label{eq:l21-sup-p}
\sup_{x\in\mathbb{R}^d}\|\nabla f(x)\|_{2\to 1}
=\sup_{p\in\Delta^{N-1}}\|J(p)W\|_{2\to 1}.
\end{equation}

Fix $p\in\Delta^{N-1}$ and $v\in\mathbb{R}^N$. Let $\mu:=p^\top v$. Using
\[
J(p)v=(\mathrm{Diag}(p)-pp^\top)v=\mathrm{Diag}(p)\bigl(v-\mu\mathbf 1\bigr),
\]
we obtain
\begin{equation}
\label{eq:l1-mad}
\|J(p)v\|_1=\sum_{i=1}^N p_i\,|v_i-\mu|.
\end{equation}
The right-hand side is the mean absolute deviation of $v$ under the probability distribution $p$.

Fix $v\in\mathbb{R}^N$ and define $m:=\min_i v_i$, $M:=\max_i v_i$, and $r:=M-m$.
We claim that
\begin{equation}
\label{eq:sup-p-mad}
\sup_{p\in\Delta^{N-1}} \sum_{i=1}^N p_i\,|v_i-p^\top v|
=\frac{r}{2}.
\end{equation}

\emph{Upper bound.}
Let $\mu=p^\top v$. By Cauchy--Schwarz,
\[
\sum_{i=1}^N p_i|v_i-\mu|
\le \Big(\sum_{i=1}^N p_i\Big)^{1/2}\Big(\sum_{i=1}^N p_i(v_i-\mu)^2\Big)^{1/2}
=\sqrt{\mathrm{Var}_p(v)}.
\]
By Popoviciu's inequality for bounded random variables, $\mathrm{Var}_p(v)\le r^2/4$, hence
\[
\sum_{i=1}^N p_i|v_i-\mu|\le \frac{r}{2}.
\]

\emph{Lower bound (attainability).}
Let $i_{\max}\in\arg\max_i v_i$ and $i_{\min}\in\arg\min_i v_i$, and take
\[
p=\frac12(e_{i_{\max}}+e_{i_{\min}}).
\]
Here $e_i\in\mathbb{R}^N$ denotes the $i$-th standard basis vector, i.e., $(e_i)_j=\mathbf{1}\{j=i\}$.
Then $\mu =(M+m)/2$, and
\[
\sum_{i=1}^N p_i|v_i-\mu|
=\frac12\Big|M-\frac{M+m}{2}\Big|+\frac12\Big|m-\frac{M+m}{2}\Big|
=\frac{r}{2}.
\]
This proves \eqref{eq:sup-p-mad}.

Combining \eqref{eq:l1-mad} and \eqref{eq:sup-p-mad}, for any $u\in\mathbb{R}^d$ with $\|u\|_2=1$,
\begin{equation}
\label{eq:sup-p-fixed-u-l1}
\sup_{p\in\Delta^{N-1}}\|J(p)Wu\|_1
=\frac12\Big(\max_i (Wu)_i-\min_j (Wu)_j\Big).
\end{equation}

Writing $(Wu)_i=w_i^\top u$, we have
\[
\max_i (Wu)_i-\min_j (Wu)_j
=\max_{i,j}(w_i-w_j)^\top u.
\]
Therefore,
\[
\sup_{\|u\|_2=1}\Big(\max_i (Wu)_i-\min_j (Wu)_j\Big)
=\max_{i,j}\sup_{\|u\|_2=1}(w_i-w_j)^\top u
=\max_{i,j}\|w_i-w_j\|_2
=D(W).
\]

From \eqref{eq:l21-sup-p} and \eqref{eq:sup-p-fixed-u-l1},
\begin{align*}
&\sup_{x\in\mathbb{R}^d}\|\nabla f(x)\|_{2\to 1} \\
=&\sup_{p\in\Delta^{N-1}}\sup_{\|u\|_2=1}\|J(p)Wu\|_1 \\
=&\sup_{\|u\|_2=1}\sup_{p\in\Delta^{N-1}}\|J(p)Wu\|_1 \\
=&\frac12\sup_{\|u\|_2=1}\Big(\max_i (Wu)_i-\min_j (Wu)_j\Big) \\
=&\frac12\,D(W).
\end{align*}
This completes the proof.
\end{proof}

\subsection{Lipschitz Constant of MLPs}
\label{app:lipschitz-mlp}
Based on the previous theory for the Lipschitz constant of MLPs by \citet{fazlyab2020safety}, we obtain the following direct result:

\begin{proposition}[SDP-certified Lipschitz constant for MLPs]
Consider the $\ell$-hidden-layer MLP
\[
x_0=x\in\mathbb R^{n_0},\qquad
x_{k}=\phi(W_{k-1}x_{k-1}+b_{k-1})\in\mathbb R^{n_k}\ (k=1,\dots,\ell),
\]
\[
f(x)=W_{\ell}x_{\ell}+b_{\ell}\in\mathbb R^{m},
\]
where $\phi(z)=[\varphi(z_1),\dots,\varphi(z_{n_k})]^\top$ acts componentwise and the scalar activation
$\varphi$ is slope-restricted on $[0,\beta]$ (i.e., $0\le \frac{\varphi(u)-\varphi(v)}{u-v}\le \beta$ for all $u\neq v$).

Let $N:=n_0+n_1+\cdots+n_{\ell}$ and $n:=n_1+\cdots+n_{\ell}$.
Define the matrices $A\in\mathbb R^{n\times N}$ and $B\in\mathbb R^{n\times N}$ by
\[
A=\left[
\begin{matrix}
W_0 & 0 & \cdots & 0 & 0\\
0 & W_1 & \cdots & 0 & 0\\
\vdots & \vdots & \ddots & \vdots & \vdots\\
0 & 0 & \cdots & W_{\ell-1} & 0
\end{matrix}\right],\qquad
B=\left[
\begin{matrix}
0 & I_{n_1} & 0 & \cdots & 0\\
0 & 0 & I_{n_2} & \cdots & 0\\
\vdots & \vdots & \vdots & \ddots & \vdots\\
0 & 0 & 0 & \cdots & I_{n_{\ell}}
\end{matrix}\right].
\]
Let the decision variable $T_n\in\mathbb S^{n}$ be block-diagonal across layers:
\[
T_n=\mathrm{blkdiag}(T_{n_1},\dots,T_{n_{\ell}}),\qquad
T_{n_k}=\mathrm{diag}(t_{k,1},\dots,t_{k,n_k}),\quad t_{k,i}\ge 0.
\]

Consider the SDP (with variable $H\ge 0$):
\[
\begin{aligned}
\min_{\{t_{k,i}\},\,H}\quad & H\\
\mathrm{s.t.}\quad & M \succeq 0,\\
& M= -\Big(M_L(H)+M_A(T_n)\Big),\\[2mm]
& M_L(H)=
\left[
\begin{matrix}
-H I_{n_0} & 0 & \cdots & 0\\
0 & 0 & \cdots & 0\\
\vdots & \vdots & \ddots & \vdots\\
0 & 0 & \cdots & W_{\ell}^{\top}W_{\ell}
\end{matrix}\right]\in\mathbb S^{N},\\[2mm]
& M_A(T_n)=
\left[\begin{matrix}A\\ B\end{matrix}\right]^{\!\top}
\left[\begin{matrix}
0 & \beta T_n\\
\beta T_n & -2T_n
\end{matrix}\right]
\left[\begin{matrix}A\\ B\end{matrix}\right]\in\mathbb S^{N}.
\end{aligned}
\]
Let $H^\star$ be the optimal value and define $L_{\mathrm{SDP}}:=\sqrt{H^\star}$.
Then $L_{\mathrm{SDP}}$ is a global $\|\cdot\|_2$-Lipschitz constant of $f$, i.e.,
\[
\|f(x)-f(y)\|_2 \le L_{\mathrm{SDP}}\|x-y\|_2,\qquad \forall x,y\in\mathbb R^{n_0}.
\]
\end{proposition}

\section{Additional Experiments}
\label{app:additional-experiments}

\subsection{Experiments on Estimator Sensitivity}
\label{app:sensitivity}
In this section, we conduct a sensitivity analysis of the proposed estimators $\hat S_{Cov}$ and $\hat S_{Cpt}$ with respect to the parameter $\beta$. We take the ''-90\%'' domain of ColoredMNIST as the target domain and regard the mixture of the other two domains as the source domain. We train the model on the source domain using ERM with the best hyperparameters in DomainBed. At the last checkpoint, we store all representation-label pairs from both the source and target domains. For each $\beta$, we randomly sample 10,000 pairs from each domain, run the DataShifts algorithm to obtain the estimated results, repeat this procedure 50 times, and report the average. The results are shown below:

\begin{table}[h]
\centering
\caption{Sensitivity analysis with respect to $\beta$.}
\label{tab}
\small
\setlength{\tabcolsep}{6pt}
\begin{tabular}{c c c c c c c c c}
\toprule
$\beta$ & 0.001 & 0.002 & 0.005 & 0.01 & 0.02 & 0.05 & 0.1 & 0.2 \\
\midrule
$\hat S_{Cov}$ & 0.5081 & 0.5080 & 0.5078 & 0.5074 & 0.5066 & 0.5041 & 0.4995 & 0.4916 \\
$\hat S_{Cpt}$ & 1.4456&  1.4458& 1.4467& 1.4478& 1.4501& 1.4567& 1.4674& 1.4895 \\
Time (s) & 15.12 &14.09 &12.62 &11.60 &10.54  &9.09  &8.03  &6.83 \\
\bottomrule
\end{tabular}
\end{table}

Across magnitude changes in $\beta$, the coefficients of variation of $\hat S_{Cov}$ and $\hat S_{Cpt}$ are 1.076\% and 0.990\%, so both estimators are stable for small $\beta$. Since entropic OT gets faster as $\beta$ grows, we use $\beta=0.2$ in the main experiments to balance speed and accuracy.

\subsection{Experiments on the Bias under Stochastic Labeling}
\label{app:overest}
We test the bias of our $\gamma^*$-Y$\mid$X shift estimator $\hat S_{Cpt}$ on the following stochastic-labeling synthetic data. In both source and target domains, covariates are drawn from a 10-dim standard normal distributions. At each $x$, scalar labels are sampled from normal distribution $\mathcal{N}(\|x\|_2,\sigma^2)$ in the source and $\mathcal{N}(\|x\|_2+1.0,\sigma^2)$ in the target. Thus, the true Y$\mid$X shift $S_{Cpt}^{\gamma^{*}}=1.0$, and standard deviation $\sigma$ controls the noise. By Proposition~\ref{pro:varDelta-in-stochastic-labeling}, the irreducible errors of both domains are $\sigma^2$, and the bias $\Delta\geq0$ of our estimator should be bounded by $2\sigma$. $\hat \Delta=\hat S_{Cpt}-S_{Cpt}^{\gamma^{*}}$ is an estimate of the bias $\Delta$. For each $\sigma$, we randomly sample 10,000 points from each domain, and run DataShifts algorithm with $\beta=0.2$, repeat this procedure 50 times, and report the average $\hat \Delta$. The results are shown below:

\begin{table}[h]
\centering
\caption{Bias of $\hat S_{Cpt}$ under different noise levels $\sigma$.}
\label{tab}
\small
\setlength{\tabcolsep}{6pt}
\begin{tabular}{c c c c c c c c c}
\toprule
$\sigma$ & 0.01 & 0.1 & 0.3 & 0.5 & 1.0 & 2.0 & 5.0 & 10.0 \\
\midrule
Irreducible error bound ($2\sigma$)
& 0.02 & 0.2 & 0.6 & 1.0 & 2.0 & 4.0 & 10.0 & 20.0 \\
Bias ($\hat \Delta$)
& 0.0036&  0.0389&   0.0702&     0.1440&  0.4865&  1.4524&  4.7314& 10.3467 \\
\bottomrule
\end{tabular}
\end{table}

When the noise $\sigma$ is small, the bias $\hat \Delta$ is also very small, which means our estimator $\hat S_{Cpt}$ stays close to the true Y$\mid$X shift $S_{Cpt}^{\gamma^{*}}=1.0$. Even when the noise in both domains reaches half of the true Y$\mid$X shift ($\sigma=0.5$), the bias $\hat{\Delta}$ remains small, and the $\hat S_{Cpt}$ still does not notably overestimate. As $\sigma$ grows far beyond 1.0, Y$\mid$X shift becomes less identifiable. The bias increases and overestimation appears, but it stays below the bound set by irreducible error, as proved in Proposition~\ref{pro:varDelta-in-stochastic-labeling}.


\end{document}